\documentclass{article}
\pdfoutput=1
\usepackage[final]{neurips_2021} 

\usepackage[utf8]{inputenc} 
\usepackage[T1]{fontenc}    
\usepackage{hyperref}       
\usepackage{url}            
\usepackage{booktabs}       
\usepackage{amsfonts}       
\usepackage{nicefrac}       
\usepackage{microtype}      
\usepackage{lipsum}		
\usepackage{graphicx}
\usepackage{natbib}
\usepackage{doi}

\usepackage[american]{babel}
\usepackage{filecontents}
\usepackage{microtype}
\usepackage{subfigure}
\usepackage{booktabs} 
\usepackage[normalem]{ulem}
\usepackage{graphicx}
\usepackage{caption}

\usepackage{hyperref}

\usepackage{verbatim}
\usepackage{amsmath}
\usepackage{amsthm}
\usepackage{bbm}
\usepackage{microtype}
\usepackage{color}
\usepackage{eucal}
\usepackage{epsfig,amsmath,amssymb,amsfonts,amstext,amsthm,mathrsfs}
\usepackage{latexsym,graphics,epsf,epsfig,psfrag}
\usepackage{dsfont,color,epstopdf,fixmath}
\usepackage{enumitem}
\usepackage{algorithm,algorithmic,latexsym}

\newtheorem{assumption}{\textbf{Assumption}}

\newtheorem{corollary}{\textbf{Corollary}}
\newtheorem{Lemma}{\textbf{Lemma}}
\newtheorem{theorem}{\textbf{Theorem}}

\usepackage[acronym]{glossaries} 

\newacronym{T1}{T1}{\theta_1}
\usepackage{color}

\newcommand{\mcs}{\mathcal{S}}
\newcommand{\mca}{\mathcal{A}}
\newcommand{\nn}{\nonumber}
\newcommand{\mE}{\mathbb{E}}
\newcommand{\om}{\omega}

\newcommand{\mo}{\mathcal{O}}

\newcommand{\tb}{\tau_{\beta}}
\allowdisplaybreaks

\usepackage{microtype}
\usepackage{graphicx}
\usepackage{subfigure}
\usepackage{booktabs} 

\usepackage{hyperref}

\usepackage{natbib} 
\usepackage{mathtools} 
\usepackage{siunitx} 
\usepackage{booktabs} 
\usepackage{tikz} 
\usepackage{cleveref}

\usepackage{mathtools} 
\usepackage{siunitx} 
\usepackage{cleveref}
\usepackage[toc,page,header]{appendix}
\usepackage{minitoc}

\title{Non-Asymptotic Analysis for Two Time-scale TDC with General Smooth Function Approximation}
\date{}
\author{
  Yue Wang \\
    Department of Electrical Engineering\\
    University at Buffalo\\
    Buffalo, NY, USA\\
    \texttt{ywang294@buffalo.edu}
   \And
 Shaofeng Zou \\
    Department of Electrical Engineering\\
    University at Buffalo\\
    Buffalo, NY, USA\\
    \texttt{szou3@buffalo.edu}
  
  \And
  Yi Zhou\\
  Department of Electrical and Computer Engineering\\
     University of Utah\\
    Salt Lake City, Utah, USA\\
    \texttt{yi.zhou@utah.edu}
}

\begin{document}

\doparttoc 
\faketableofcontents
\maketitle
\begin{abstract}
    Temporal-difference learning with gradient correction (TDC) is a two time-scale algorithm for policy evaluation in reinforcement learning. This algorithm was initially proposed with linear function approximation, and was later extended to the one with general smooth function approximation. The asymptotic convergence for the on-policy setting with general smooth function approximation was established in \citep{bhatnagar2009convergent}, however, the non-asymptotic convergence analysis remains unsolved due to challenges in the non-linear and two-time-scale update structure, non-convex objective function and the projection onto a time-varying  tangent plane. In this paper, we develop novel techniques to address the above challenges and explicitly characterize the non-asymptotic error bound for the general off-policy setting with i.i.d.\ or Markovian samples, and show that it converges as fast as $\mathcal O(1/\sqrt T)$ (up to a factor of $\mathcal O(\log T)$). Our approach can be applied to a wide range of value-based reinforcement learning algorithms with general smooth function approximation.
    
\end{abstract}

\section{Introduction}

In reinforcement learning (RL), an agent interacts with a stochastic environment in order to maximize the total reward \citep{sutton2018reinforcement}. Towards this goal, it is often needed to evaluate how good a policy performs, and more specifically, to learn its value function. Temporal difference (TD) learning algorithm is one of the most popular policy evaluation approaches. However, when applied with function approximation approach and/or under the off-policy setting,  the TD learning algorithm may diverge \citep{baird1995residual,tsitsiklis1997analysis}. 
%
%
%
%
To address this issue, a family of gradient-based TD (GTD) algorithms, e.g., GTD, GTD2, temporal-difference learning with gradient correction (TDC) and Greedy-GQ, were developed for the case with linear function approximation   \citep{maei2011gradient, sutton2009fast,maei2010toward,sutton2009acov,sutton2009fast}. These algorithms were later extended to the case with general smooth function approximation in \citep{bhatnagar2009convergent}, where asymptotic convergence guarantee was established for the on-policy setting with i.i.d. samples. 


Despite the success of the GTD methods in practice, previous theoretical studies only showed that these algorithms converge asymptotically, and did not suggest how fast these algorithms converge and how the accuracy of the solution depends on various parameters of the algorithms. Not until recently have the non-asymptotic error bounds for these algorithms been investigated, e.g., \citep{dalal2020tale,karmakar2018two,wang2020finite,xu2019two,kaledin2020finite,dalal2018finite,wang2017finite}, which mainly focus on the case with linear function approximation. These results thus cannot be directly applied to more practical applications with general smooth function approximation, e.g., neural networks, which have greater representation power, do not need to construct feature mapping, and are widely used in practice.

In this paper, we develop a non-asymptotic analysis for the TDC algorithm with general smooth function approximation (which we refer to as non-linear TDC) for both i.i.d. and Markovian samples. Technically, the analysis in this paper is not a straightforward extension of previous studies on those GTD algorithms with linear function approximation. First of all, 
different from existing studies with linear function approximation whose objective functions are \textit{convex} and the updates are \textit{linear}, the objective function of the non-linear TDC algorithm is \textit{non-convex}, and the two time-scale updates are \textit{non-linear} functions of the parameters. 
Second, the objective function of the non-linear TDC algorithm, the mean-square projected Bellman error (MSPBE),  involves a projection onto a \textit{time-varying} tangent plane which depends on the sample trajectory, whereas for GTD algorithms with linear function approximation, this projection is time-invariant.
Third, due to the two time-scale structure of the algorithm and the Markovian noise, novel techniques to deal with the stochastic bias and the tracking error need to be developed.

\subsection{Challenges and Contributions}

In this section, we summarize the technical challenges and our contributions. 

\textbf{Analysis for two time-scale non-linear updates and non-convex objective.}
Unlike many existing results on two time-scale stochastic approximation, e.g.,  \citep{konda2004convergence, gupta2019finite,kaledin2020finite} and the studies of linear GTD algorithms in \citep{xu2019two,ma2020variance,wang2017finite,dalal2020tale}, the objective function of the  non-linear TDC is non-convex, and its two time-scale updates are non-linear. Therefore, existing studies on linear two time-scale algorithms cannot be directly applied. Moreover, the convergence to global optimum cannot be guaranteed for the non-linear TDC algorithm, and therefore, we study the convergence to stationary points. In this paper, we develop a novel non-asymptotic analysis of the non-linear TDC algorithm, which solves RL problems from a non-convex optimization perspective. We note that our analysis is not a straightforward extension of analyses of non-convex optimization, as the update rule here is two time-scale and the noise is Markovian. The framework we develop in this paper can be applied to analyze a wide range of value-based RL algorithms with general smooth function approximation.

\textbf{Time-varying projection.} For the MSPBE, a projection of the Bellman error onto the parameterized function class is involved. However, unlike linear function approximation,  the projection onto a general smooth class of functions usually does not have a closed-form solution. Thus, a projection onto the tangent plane at the current parameter is used instead, which incurs a time-varying projection that depends on the current parameter and thus the sample trajectory. This brings in additional challenges in the bias and variance analysis due to such dependency. We develop a novel approach to decouple such a dependency and characterize the bias by exploiting the uniform ergodicity of the underlying MDP and the smoothness of the parameterized function. The new challenges posed by the time-varying projection and the dependence between the projection and the sample trajectory are not special to the non-linear TDC investigated in this paper, and they exist in a wide range of value-based algorithms with general smooth function approximation, where our techniques can be applied.


\textbf{A tight tracking error analysis.} Due to the two time-scale structure of the update rule, the tracking error, which measures how fast the fast time-scale tracks its own limit, needs to be explicitly bounded. Unlike the studies on two time-scale linear stochastic approximation \citep{dalal2020tale,kaledin2020finite,konda2004convergence}, where a linear transformation can asymptotically decouple the dependence between the fast and slow time-scale updates, it is non-trivial to construct such a transformation for non-linear updates. To develop a tight bound on the tracking  error, 
we develop a novel technique that bounds the tracking error as a function of the gradient of the MSPBE. This leads to a tighter bound on the tracking error compared to many existing works on two time-scale analysis, e.g., \citep{wu2020finite, hong2020two}. Although we do not decouple the fast and slow time-scale updates, we still obtain a desired convergence rate of $\mathcal O(1/\sqrt{T})$ (up to a factor of $\log T$), which matches with the complexity of stochastic gradient descent for non-convex problems \citep{ghadimi2013stochastic}.


\subsection{Related Work}

\textbf{TD, Q-learning and SARSA.}
The asymptotic convergence of TD with linear function approximation was shown in  \citep{tsitsiklis1997analysis}, and the non-asymptotic analysis of TD was developed in  \citep{srikant2019finite, lakshminarayanan2018linear, bhandari2018finite,dalal2018finite,sun2020finite}. Moreover, \citep{cai2019neural} further studied the non-asymptotic error bound of TD learning with neural function approximation. Q-learning and SARSA are usually used for solving the optimal control problem and were shown to converge asymptotically under some conditions in  \citep{melo2008analysis,perkins2003convergent}. Their non-asymptotic error bounds were also studied in  \citep{zou2019finite}. The non-asymptotic analysis of Q-learning under the neural function approximation 
was developed in  \citep{cai2019neural, xu2020finite}. Note that all these algorithms are one time-scale, while the TDC algorithm we study is a two time-scale algorithm.

\textbf{GTD methods with linear function approximation.}
A class of GTD algorithms were proposed to address the divergence issue for off-policy training \citep{baird1995residual} and arbitrary smooth function approximation \citep{tsitsiklis1997analysis}, e.g., GTD, GTD2 and TDC \citep{maei2011gradient, sutton2009fast,maei2010toward,sutton2009acov,sutton2009fast}. Recent studies established their non-asymptotic convergence rate, e.g., \citep{dalal2018finite,wang2017finite,liu2015finite,gupta2019finite,xu2019two,dalal2020tale,kaledin2020finite,ma2020variance,wang2020finite,ma2021greedygq} under i.i.d.\ and Markovian settings. These studies focus on the case with linear function approximation, and thus the objective functions are convex, and the updates are linear. 
In this paper, we focus on the non-linear TDC algorithm with general smooth function approximation, where the two time-scale update rule is non-linear, the objective is non-convex, and the projection is time-varying, and thus new techniques are required to develop the non-asymptotic analysis.


\textbf{Non-linear two time-scale stochastic approximation.} There are also studies on asymptotic convergence rate and non-asymptotic analysis for non-linear two time-scale stochastic approximation, e.g., \citep{mokkadem2006convergence, doan2021nonlinear}. Although the non-linear update rule is investigated, it is assumed that the algorithm converges to the global optimum. In this paper, we do not make such an assumption on the global convergence, which may not necessarily hold for the non-linear TDC algorithm, and instead, we study the convergence to stationary points, which is a widely used convergence criterion for non-convex optimization problems.  We also note that there is a resent work studying the batch-based non-linear TDC in  \citep{xu2021sample}, where at each update, a batch of samples is used. To achieve a sample complexity of $\mo(\epsilon^{-2})$, a batch size of $\mo(\epsilon^{-1})$ is required in  \citep{xu2021sample} to control the bias and variance. 
We note that by setting the batch size being one in \citep{xu2021sample}, the desired sample complexity cannot be obtained, and their error bound will be a constant. 
In this paper, we focus on the non-linear TDC algorithm without using the batch method, where the parameters update in an online and incremental fashion and at each update only one sample is used. Our error analysis is novel and more refined as it does not require a large batch size of $\mo(\epsilon^{-1})$ while still achieving the same sample complexity. 


\section{Preliminaries}
\subsection{Markov Decision Process}
A Markov decision process (MDP) is a tuple $(\mcs,\mca,\mathsf P,r,\gamma)$, where $\mcs$ and $\mca$ are the state and action spaces, $\mathsf P=\mathsf P(s'|s,a)$ is the transition kernel, $r:\mcs\times\mca\times\mcs\to \mathbb{R}^+$ is the reward function bounded by $r_{\max}$, and $\gamma\in [0,1]$ is the discount factor. A stationary policy $\pi$ maps a state $s\in\mcs$ to a probability distribution $\pi(\cdot|s)$ over the action space $\mca$. At each time-step $t$, suppose the process is at some state $s_t\in\mcs$, and an action $a_t\in\mca$ is taken. Then the system transits to the next state $s_{t+1}$ following the transition kernel $\mathsf P(\cdot|s_t,a_t)$, and the agent receives a reward $r(s_t,a_t,s_{t+1})$. 

For a given policy $\pi$ and any initial state $s\in\mcs$, we define its value function as $
V^\pi\left(s\right)=\mE\left[\sum_{t=0}^{\infty}\gamma^t   r(S_t,A_t,S_{t+1})|S_0=s, \pi\right]$. The goal of policy evaluation is to use the samples generated from the MDP to estimate the value function. The value function satisfies the Bellman equation: $V^\pi(s)=T^\pi V^\pi(s)$ for any $s\in\mcs$,  where the Bellman operator $T^\pi$ is defined as
\begin{align}
        T^\pi V(s)&=\sum_{s'\in\mcs,a\in\mca} \mathsf P(s'|s,a)\pi(a|s) r(s,a,s') +\gamma\sum_{s'\in\mcs,a\in\mca} \mathsf P(s'|s,a)\pi(a|s)V(s').
\end{align} 
Hence the value function $V^{\pi}$ is the fixed point of the Bellman operator $T^\pi$ \citep{bertsekas2011dynamic}. 

\subsection{Function Approximation}
In practice, the state space $\mcs$ usually contains a large number of states or is even continuous, which will induce a heavy computational overhead. A popular approach is to approximate the value function using a parameterized class of  functions. Consider a parameterized family of functions $\left\{V_{\theta}: \mcs \to \mathbb{R}|\theta\in\mathbb{R}^N \right\}$, e.g., neural networks. The goal is to find a $V_{\theta}$ with a compact representation in $\theta$ to approximate the value function $V^{\pi}$. In this paper, we focus on a general family of smooth functions, which may not be linear in $\theta$.

\section{TDC with Non-Linear Function Approximation}
In this section, we introduce the TDC algorithm with general smooth function approximation in \citep{bhatnagar2009convergent} for the off-policy setting with both i.i.d.\ samples and Markovian samples, and further characterize the non-asymptotic error bounds. 

Consider the the following mean-square projected Bellman error (MSPBE):
\begin{align}\label{eq:obj}
    J(\theta)=\mE_{\mu^\pi}\left[\left\|V_\theta(s)-\mathbf \Pi_\theta T^\pi V_\theta(s)\right\|^2\right],
\end{align}
where $\mu^\pi$ is the stationary distribution induced by the policy $\pi$, and $\mathbf \Pi_{\theta}$ is the orthogonal projection onto the tangent plane of $V_{\theta}$ at $\theta$: $\left\{ \hat{V}_{\zeta}(s)|\zeta\in\mathbb{R}^N \text{ and }\hat{V}_{\zeta}(s)=\phi_{\theta}(s)^\top\zeta\right\}$ and $\phi_{\theta}(s)=\nabla V_\theta(s)$. Note that the projection is onto the tangent plane instead of $\left\{V_{\theta}:\theta\in\mathbb{R}^N \right\}$ since the projection onto the latter one may not be computationally tractable if $V_{\theta}$ is non-linear. 

In \citep{bhatnagar2009convergent}, the authors proposed a two time-scale TDC algorithm to minimize the MSPBE $J(\theta)$. Specifically, a stochastic gradient descent approach is used with the weight doubling trick (for the double sampling problem) \citep{sutton2009acov}, which yield a two time-scale update rule. We note that the algorithm developed in \citep{bhatnagar2009convergent} was for the on-policy setting with i.i.d.\ samples from the stationary distribution, and the asymptotic convergence of the algorithm to stationary points was established. 

In the off-policy setting, the goal is to estimate the value function $V^{\pi}$ of the target policy $\pi$ using the samples from a different behavior policy $\pi_b$. In this case, the MSPBE can be written as 
\begin{align} 
    J(\theta)=\mE_{\mu^{\pi_b}}[\left\|V_\theta(s)-\mathbf \Pi_\theta T^\pi V_\theta(s)\right\|^2],
\end{align}
and we use the approach of importance sampling.
Following steps similar to those in \citep{maei2011gradient}, $J(\theta)$ can be further written as
\begin{align}
    J(\theta)&=\mE_{\mu^{\pi_b}}[\rho(S,A)\delta_{S,A,S'}(\theta)\phi_{\theta}(S)]^\top A_{\theta}^{-1} \mE_{\mu^{\pi_b}}[\rho(S,A)\delta_{S,A,S'}(\theta)\phi_{\theta}(S)],
\end{align}
where $\delta_{s,a,s'}(\theta)=r(s,a,s')+\gamma V_{\theta}(s')-V_{\theta}(s)$ is the TD error, $\phi_{\theta}(s)=\nabla V_{\theta}(s)$ is the character vector, $\rho(s,a)=\frac{\pi(a|s)}{\pi_b(a|s)}$ is the importance sampling ratio for a given sample $O=(s,a,r,s')$ and $A_{\theta}=\mE_{\mu^{\pi_b}}[\phi_{\theta}(S)\phi_{\theta}(S)^\top]$. 

To compute $\nabla J(\theta)$, we consider its $i$-th entry, i.e., the partial derivative w.r.t. the $i$-th entry of $\theta$:
\begin{align}\label{eq:J1}
    &-\frac{1}{2}\frac{\partial J(\theta)}{ \partial \theta^i}\nn\\
    &=\underbrace{-\mE_{\mu^{\pi_b}}\left[\frac{\partial}{ \partial \theta^i}(\rho\delta\phi)\right]^\top A_{\theta}^{-1}\mE_{\mu^{\pi_b}}\left[\rho\delta\phi\right]}_{(a)}+\frac{1}{2}\underbrace{( A_{\theta}^{-1}\mE_{\mu^{\pi_b}}\left[\rho\delta\phi\right])^\top\mE_{\mu^{\pi_b}}\left[\frac{\partial}{ \partial \theta^i}(\phi\phi^\top)\right]( A_{\theta}^{-1}\mE_{\mu^{\pi_b}}\left[\rho\delta\phi\right])}_{(b)},
\end{align}
where to simplify notations, we omit the dependence on $\theta, S, A$ and $S'$. 
To get an unbiased estimate of the terms in \eqref{eq:J1}, several independent samples are needed, but this is not applicable when there is only one sample trajectory. Hence we employ the weight doubling trick \citep{sutton2009acov}.
Define 
$
    \omega(\theta)= A_{\theta}^{-1}\mathbb{E}_{\mu^{\pi_b}}\left[\rho(S,A)\delta_{S,A,S'}(\theta)\phi_{\theta}(S)\right],
$
then term $(a)$ can be written as follows:
\begin{align}
    &-\mE_{\mu^{\pi_b}}\left[\frac{\partial}{ \partial \theta^i}(\rho\delta\phi)\right]^\top A_{\theta}^{-1}\mE_{\mu^{\pi_b}}\left[\rho\delta\phi\right]\nn\\
    &=-\mE_{\mu^{\pi_b}}\left[\rho(\gamma (\phi_{\theta}(S'))_i-(\phi_{\theta}(S))_i)\phi_{\theta}(S)\right]^\top\om(\theta) -\mE_{\mu^{\pi_b}}\left[\rho\delta(\nabla^2 V)_i\right]^\top\om(\theta);
\end{align}
and term $(b)$ can be written as follows:
\begin{align}
    &( A_{\theta}^{-1}\mE_{\mu^{\pi_b}}\left[\rho\delta\phi\right])^\top\mE_{\mu^{\pi_b}}\left[\frac{\partial}{ \partial \theta^i}(\phi\phi^\top)\right]( A_{\theta}^{-1}\mE_{\mu^{\pi_b}}\left[\rho\delta\phi\right]) =2\mE_{\mu^{\pi_b}}\left[\phi^\top\om(\theta)(\frac{\partial}{ \partial \theta^i} \phi^\top)\om(\theta)\right].
\end{align}
 Hence the gradient can be re-written as 
\begin{align}\label{eq:gradientJ}
    -\frac{\nabla J(\theta)}{2}&=\mE_{\mu^{\pi_b}}\left[ \rho(S,A)\delta_{S,A,S'}(\theta)\phi_{\theta}(S)\right]-h(\theta,\omega(\theta)) -\gamma\mE_{\mu^{\pi_b}}\left[ \rho(S,A)\phi_{\theta}(S')\phi_{\theta}(S)^\top\right]\omega(\theta),
\end{align}
 where
$
    h(\theta,\omega)=\mathbb{E}_{\mu^{\pi_b}} [\left(\rho(S,A)\delta_{S,A,S'}(\theta)-\phi_{\theta}(S)^\top\omega\right) \nabla^2 V_{\theta}(S)\omega ].
$
Thus with this weight doubling trick \citep{sutton2009acov}, a two time-scale stochastic gradient descent algorithm can be constructed. In Algorithm \ref{alg:markov}, we present the algorithm for the Markovian setting. The algorithm under the i.i.d.\ setting is slightly different, hence we refer the readers to \Cref{alg:iid} in \Cref{section:iid}.
\begin{algorithm}
\caption{Non-Linear Off-Policy TDC under the Markovian Setting}
\label{alg:markov}
\textbf{Input}: $T$, $\alpha$, $\beta$, $\pi$, $\pi_b$, $\left\{V_{\theta}|\theta\in\mathbb{R}^N\right\}$\\
\textbf{Initialization}: $\theta_0$,$w_0$
\begin{algorithmic}[1] 
\STATE {Choose $W\sim \text{Uniform}(0,1,...,T-1)$}
\FOR {$t=0,1,...,W-1$}
\STATE {Sample $O_t=(s_t,a_t,r_t,s_{t+1})$ following $\pi_b$}
\STATE {$\delta_t(\theta_t)=r(s_t,a_t,s_{t+1})+\gamma V_{\theta_t}(s_{t+1})-V_{\theta_t}(s_t)$}
\STATE { $\rho_t=\frac{\pi(a_t|s_t)}{\pi_b(a_t|s_t)}$}
\STATE {$h_t(\theta_t,\omega_t) =\left(\rho_t\delta_t(\theta_t)-\phi_{\theta_t}(s_t)^\top\omega_t\right)\nabla^2 V_{\theta_t}(s_t)\omega_t$}
\STATE {$\omega_{t+1} =\mathbf \Pi_{R_{\omega}}\big(\omega_t +\beta\big(-\phi_{\theta_t}(s_t)\phi_{\theta_t}(s_t)^\top\omega_t+\rho_t\delta_t(\theta_t)\phi_{\theta_t}(s_t)\big)\big)$}
\STATE {$\theta_{t+1}=\theta_t +\alpha\big( \rho_t\delta_t(\theta_t)\phi_{\theta_t}(s_t) -\gamma\rho_t\phi_{\theta_t}(s_{t+1} )\phi_{\theta_t}(s_t)^\top\omega_t-h_t(\theta_t,\omega_t)\big)$}
\ENDFOR
\end{algorithmic}
\textbf{Output}: $\theta_W$
\end{algorithm}

In Algorithm \ref{alg:markov},  $\mathbf\Pi_{R_{\omega}}(v)=\arg\min_{\|w\|\leq R_{\om}}\|v-w\|$ denotes the projection operator, where $R_{\omega}=\frac{\rho_{\max}C_{\phi}}{\lambda_v}(r_{\max}+(1+\gamma) C_v)$ (the constants are defined in \Cref{section:result}). As we will show in \eqref{eq:omegabound} in the appendix that for any $\theta\in\mathbb{R}^N$, $\om(\theta)$ is always upper bounded by $R_{\om}$, i.e.,  $\|\omega(\theta)\|\leq R_{\omega}$. The projection step in the algorithm is introduced mainly for the convenience of the analysis.
Motivated by the randomized stochastic gradient method in \citep{ghadimi2013stochastic}, which is designed to analyze non-convex optimization problems, in this paper, we also consider a randomized version of the non-linear TDC algorithm. Specifically, let $W$  be an independent random variable with a uniform distribution over $\left\{ 0,1,...,T-1\right\}$. We then run the non-linear TDC algorithm for $W$ steps and output $\theta_W$. 
\subsection{Non-asymptotic Error Bounds}\label{section:result}
In this section, we present our main results of the non-asymptotic error bounds on the convergence of the off-policy non-linear TDC algorithm. 
Our results will be based on the following assumptions.
\begin{assumption}[Boundedness and Smoothness] \label{ass:bound}
For any $s\in \mathcal{S}$ and any $\theta,\theta'\in\mathbb{R}^N$, 
\begin{align*}
       |V_{\theta}(s)|&\leq C_v,&
    \|\phi_{\theta}(s)\|&\leq C_{\phi},\\
    \| \nabla^2 V_{\theta}(s)\|&\leq D_v,&
    \|\nabla^2 V_{\theta}(s)-\nabla^2 V_{\theta'}(s)\|&\leq L_V \|\theta-\theta'\|, 
\end{align*}
where $C_{\phi}, C_v$,  $D_v$ and $L_V$  are some positive constants. 

\end{assumption}

From Assumption \ref{ass:bound}, it follows that for any $\theta,\theta'\in\mathbb{R}^N$, 
$
    |V_{\theta}(s)-V_{\theta'}(s)|\leq C_{\phi}\|\theta-\theta'\|, $ and $
    \|\phi_{\theta}(s)-\phi_{\theta'}(s)\|\leq D_v\|\theta-\theta'\|.
$
We note that these assumptions are equivalent to the assumptions adopted in the original non-linear TDC asymptotic convergence analysis in  \citep{bhatnagar2009convergent}, and can be easily satisfied by appropriately choosing the function class $\left\{V_{\theta}:\theta\in\mathbb{R}^N\right\}$. For example, in neural networks, these assumptions can be satisfied if the activation function is Lipschitz and smooth \citep{du2019gradient,neyshabur2017implicit,miyato2018spectral}.

\begin{assumption}[Non-singularity] \label{ass:solvable}
For any $\theta\in\mathbb{R}^N$, $
    \lambda_L\left(A_{\theta}\right)\geq \lambda_v>0,
$
where $\lambda_L(A)$ denotes the minimal eigenvalue of the matrix $A$ and $\lambda_v$ is a positive constant. 
\end{assumption}

\begin{assumption}[Bounded Importance Sampling Ratio]
For any $(s,a)\in\mcs\times\mca$,  
$
    \rho(s,a)=\frac{\pi(a|s)}{\pi_b(a|s)}\leq \rho_{\max},
$
for some positive constant $\rho_{\max}$. 
\end{assumption}

The following assumption is only needed for the analysis under the Markovian setting, and is widely used for analyzing the Markovian noise, e.g., \citep{wang2020finite, kaledin2020finite,xu2021sample,zou2019finite,srikant2019finite,bhandari2018finite}. 
\begin{assumption}[Geometric uniform ergodicity]\label{ass:mixing}
There exist some constants $m>0$ and $\kappa\in(0,1)$ such that 
$
    \sup_{s\in\mcs} d_{TV}(\mathbb{P}(s_t=\cdot|s_0=s, \pi),\mu^{\pi})\leq m\kappa^t,
$
for any $t>0$, where $d_{TV}$ denotes the total-variation distance between the probability measures.
\end{assumption}

We then present the bounds on the convergence of the TDC algorithm with general smooth function approximation in the following theorem.
\begin{theorem}\label{thm:main}
Consider the following step-sizes: $\alpha=\mathcal{O}\left(\frac{1}{T^a}\right)$, and $\beta=\mathcal{O}\left(\frac{1}{T^b}\right)$, where $\frac{1}{2}\leq a\leq 1$ and $0<b\leq a$. Then, 
(1) under the i.i.d.\ setting, 
$
    \|\nabla J(\theta_W)\|^2 = \mathcal{O}\left( \frac{1}{T^{1-a}}+\frac{1}{T^b}+\frac{1}{T^{1-b}}\right);
$
and (2)  under the Markovian setting, 
$
    \|\nabla J(\theta_W)\|^2 =\mathcal{O}\left(\frac{\log T}{T^{1-a}}+\frac{1}{T^{1-b}}+\frac{\log T}{T^b}\right).
$
\end{theorem}
Here we only assume the order of the step-sizes in terms of $T$ for simplicity, their exact assumptions on them can be found in Section \ref{sec:step1} and Section \ref{sec:step2}. Similarly, we only provide the order of the bounds  here, and the explicit bounds can be found in \eqref{eq:iidresult} and \eqref{eq:markovresult} in the appendix. It can be seen that the rate under the Markovian setting is slower than the one under the i.i.d.\ setting by a factor of $\log T$, which is essentially the mixing time introduced by the dependence of samples.

Theorem \ref{thm:main} characterizes the dependence between convergence rate and the step-sizes $\alpha$ and $\beta$. We also optimize over the step-sizes in the following corollary. 
\begin{corollary}
Let $a=b=\frac{1}{2}$, i.e., $\alpha, \beta =\mathcal{O}({1}/{\sqrt{T}})$, then $(1)$ under the i.i.d.\ setting, 
$
    \|\nabla J(\theta_W)\|^2 = \mathcal{O}( {1}/{\sqrt{T}});
$
and (2)  under the Markovian setting, 
$
    \|\nabla J(\theta_W)\|^2 = \mathcal{O}( {\log T}/{\sqrt{T}}). 
$
\end{corollary}

\textbf{Remark 1.}
Our result matches with the sample complexity for the batch-based algorithm in  \citep{xu2021sample}. But their work requires a large batch size of $\mathcal{O}(\epsilon^{-1})$ to control the bias and variance, while ours only needs one sample in each step to update $\theta$ and $\om$ and can still obtain the same convergence rate. We note that by setting the batch size being one in \citep{xu2021sample}, their desired sample complexity cannot be obtained, and their error bound will be a \textit{constant}. To obtain our non-asymptotic bound and sample complexity for the  non-linear TDC algorithm, we develop a novel and more refined analysis on the tracking error, which will be discussed in the next section.  Moreover, our result matches with the convergence rate of solving general non-convex optimization problems using stochastic gradient descent in  \citep{ghadimi2013stochastic}. Compared to their work, our analysis is more challenging due to the two time-scale structure and the gradient bias from the Markovian noise and the  tracking error.

\textbf{Remark 2.}
Some analyses  on two time-scale stochastic approximation bound the tracking error in terms of $\frac{\alpha}{\beta}$, and 
require $\frac{\alpha}{\beta}\to 0$ in order to drive the tracking error to zero resulting in a convergence rate of $\mathcal{O}\left(\beta+\frac{\alpha}{\beta}\right)$ \citep{borkar2009stochastic}.   In this paper, we develop a much tighter bound on the tracking error in terms of the slow time-scale parameter $\nabla J(\theta)$. Therefore, the tracking error in our analysis is driven to zero by $\nabla J(\theta)\rightarrow 0$ not $\frac{\alpha}{\beta}\rightarrow 0$. Similar results that do not need  $\frac{\alpha}{\beta}\to 0$ can also be found, e.g., in  \citep{konda2004convergence,kaledin2020finite}. We would like to point out that the techniques in \citep{konda2004convergence,kaledin2020finite} cannot be applied in our analysis due to the non-linear two time-scale updates in this paper.

\section{Proof Sketch}
In this section, we provide an outline of the proof of Theorem \ref{thm:main} under the Markovian setting, and highlight our major technical contributions. For the complete proof of Theorem \ref{thm:main}, we refer the readers to \Cref{section:iidmain,section:markovmain}.

Let $O_t=(s_t,a_t,r_t,s_{t+1})$ be the sample observed at time $t$. Denote the tracking error by $z_t=\om_t-\om(\theta_t)$, which characterizes the error between the fast time-scale update and its limit if the slow time-scale update $\theta_t$ is kept fixed and only the fast time-scale is being updated. Denote by $G_{t+1}(\theta,\om)\triangleq \rho_t\delta_t(\theta)\phi_{\theta}(s_t) -\gamma\rho_t\phi_{\theta}(s_{t+1} )\phi_{\theta}(s_t)^\top\omega_t-h_t(\theta,\omega)$. Denote by $\tb$ the mixing time of the MDP, i.e., $\tb\triangleq\min\left\{ t:m\kappa^t\leq \beta\right\}$. 

\textbf{Step 1.} In this step, we decompose the error of gradient norm into two parts: the stochastic bias and the tracking error.
We first show in \Cref{sec:app_a} that $J(\theta)$ is $L_J$-smooth: for any $\theta_1, \theta_2\in\mathbb{R}^N$, 
\begin{align}\label{eq:lsmooth}
    \|\nabla J(\theta_1)-\nabla J(\theta_2)\|\leq L_J\|\theta_1-\theta_2\|.
\end{align}
We note that the smoothness of $J(\theta)$ is also used in \citep{xu2021sample}, which, however, is assumed instead of being proved as in this paper. 
It then follows that 
\begin{align}\label{eq:1}
    \frac{\alpha}{2}\|\nabla J(\theta_t)\|^2
    & {\leq} J(\theta_t)-J(\theta_{t+1}) +\underbrace{\alpha\langle \nabla J(\theta_t) ,-G_{t+1}(\theta_t,\omega(\theta_t))+G_{t+1}(\theta_t,\omega_t) \rangle}_{(a)}  \nn\\
    &\quad+\underbrace{\alpha \left\langle \nabla J(\theta_t), \frac{\nabla J(\theta_t)}{2}+G_{t+1}(\theta_t, \omega(\theta_t)) \right\rangle}_{(b)}+\frac{L_J}{2} \alpha^2\|G_{t+1}(\theta_t,\omega_t)\|^2.
\end{align}
This implies that the error bound on the gradient norm is controlled by the tracking error $(a)$ which is introduced by the two time-scale update rule, and the stochastic bias $(b)$ which is due to the time-varying projection and the Markovian sampling. 

\textbf{Step 2.}
We first bound the tracking error. Re-write the update of $\om_t$ in terms of $z_t$: 
$
    z_{t+1}=z_t+\beta\left(-A_{\theta_t}(s_t)z_t+b_t(\theta_t)\right)+\omega(\theta_t)-\omega(\theta_{t+1}),
$
where $A_{\theta_t}(s_t)=\phi_{\theta_t}(s_t)\phi_{\theta_t}(s_t)^\top$ and $b_t(\theta_t)=-A_{\theta_t}(s_t)\omega(\theta_t)+\rho_t\delta_t(\theta_t)\phi_{\theta_t}(s_t)$. From the Lipschitz continuity of $\om(\theta)$, it follows that 
\begin{align} \label{eq:updaz}
    \|z_{t+1}\| &\leq(1+\beta C_{\phi}^2)\|z_t\|+\beta(b_{\max}+L_{\omega}C_g),\nn\\
    \|z_{t+1}-z_t\|&\leq  \beta C_{\phi}^2 \|z_t\|+\beta(b_{\max}+L_{\omega}C_g),
\end{align}which further implies 
\begin{align}\label{eq:0}
    &\mE\left[\|z_{t+1}\|^2-\|z_t\|^2 \right]\nn\\
    &\leq \underbrace{\mE[2z_t^\top(z_{t+1}-z_t+\beta A_{\theta_t}z_t)]}_{(c)}+\mathcal{O}\left(\beta^2 \mE[\left\| z_t\right\|^2]+\beta^2\right) +\beta\mE\left[2z_t^\top(-A_{\theta_t})z_t \right],
\end{align} 
where the last term in \eqref{eq:0} can be further upper bounded by $-2\beta\lambda_v\mE[\|z_t\|^2]$.

One challenging part in our analysis is to bound term $(c)$. Equivalently, we decompose the following term into three parts:
\begin{align}\label{eq:2}
    &\mathbb{E}\left[z_{t}^\top\left(-A_{\theta_{t}}z_{t}-\frac{1}{\beta}(z_{t+1}-z_{t})\right)\right]\nn\\
    &=\underbrace{\mathbb{E}[z_{t}^\top(-A_{\theta_{t}}+A_{\theta_{t}}(s_{t}))z_{t}]}_{(d)}-\underbrace{\mathbb{E}[z_{t}^\top b_t(\theta_t)]}_{(e)} \underbrace{-\mathbb{E}\left[ z_{t}^\top\frac{\omega(\theta_{t})-\omega(\theta_{t+1})}{\beta}\right]}_{(f)}.
\end{align}
Consider term $(d)$ in \eqref{eq:2}. Unlike the case with linear function approximation, where the character function $\nabla V_{\theta}(s)=\phi(s)$ is independent with $\theta$, here the character function $\phi_{\theta}(s)$ depends on $\theta$. We use the geometric uniform ergodicity property of the MDP and the Lipschitz continuity of $A_{\theta}$ and $A_{\theta}(s)$ to decouple the dependence. More specifically, for any fixed $\theta$, $\mE[A_{\theta}(s_t)]$ converges to $A_{\theta}$ as $t$ increases. Let $t=\tb$, then we have that
\begin{align} \label{eq:trackingterm1}
    &\mathbb{E}\left[z_{\tau_{\beta}}^\top(-A_{\theta_{\tau_{\beta}}}+A_{\theta_{\tau_{\beta}}}(s_{\tau_{\beta}}))z_{\tau_{\beta}}\right ]\nn\\
    &=\mathbb{E}\left[z_0^\top(-A_{\theta_{0}}+A_{\theta_{0}}(s_{\tau_{\beta}}))z_0\right ] + \mathbb{E}\left[z_0^\top(-A_{\theta_{\tau_{\beta}}}+A_{\theta_{\tau_{\beta}}}(s_{\tau_{\beta}})+A_{\theta_{0}}-A_{\theta_{0}}(s_{\tau_{\beta}})))z_0\right ]\nn\\
    &\quad+\mE\left[(z_{\tau_{\beta}}-z_0)^\top (-A_{\theta_{\tau_{\beta}}}+A_{\theta_{\tau_{\beta}}}(s_{\tau_{\beta}}))(z_{\tau_{\beta}}-z_0)\right ] +2\mE\left[(z_{\tau_{\beta}}-z_0)^\top (-A_{\theta_{\tau_{\beta}}}+A_{\theta_{\tau_{\beta}}}(s_{\tau_{\beta}}))z_0 \right ],
\end{align}
which can be further bounded using the mixing time $\tb$ and the Lipschitz property of $A_{\theta}$ and $A_{\theta}(s_{\tb})$. We note that from the update of $z_t$, we can bound $\|z_{\tb}-z_0\|$ and $\|z_0\|$ by $\|z_{\tb}\|$, hence the bound in \eqref{eq:trackingterm1} can be bounded in terms of $\|z_{\tb}\|$.

Similarly, note that $\mE[b_t(\theta)]$ converges to $0$ as $t\to\infty$,  then we can also bound term $(e)$ in \eqref{eq:2}:
\begin{align}\label{eq:trackingterm2}
     \mE[z_{\tau_{\beta}}^\top b_{\tau_{\beta}}(\theta_{\tb})] 
    &=\mE[(z_{\tau_{\beta}}-z_0)^\top b_{\tau_{\beta}}(\theta_{\tb})]+\mE[z_0^\top b_{\tau_{\beta}}(\theta_0)] +\mE[z_0^\top( b_{\tau_{\beta}}(\theta_{\tb})-b_{\tau_{\beta}}(\theta_0))],
\end{align}
which can  be similarly bounded in terms of $\|z_{\tb}\|$. 

The challenge of bounding the third term $(f)$ in \eqref{eq:2} lies in bounding the difference between $\om(\theta_t)$ and $\om(\theta_{t+1})$. One simple approach is to use the Lipschitz continuity of $\om(\theta)$ and bound $\|\theta_t-\theta_{t+1}\|$ by a constant of order $\mathcal{O}(\alpha)$, but this will lead to a loose bound because the update $G_{t+1}(\theta_t,\om_t)$ is actually an estimator of the gradient, which will also converge to zero. The key idea in our analysis is to bound term $(f)$ in terms of the gradient of the objective function $\nabla J(\theta)$. Specifically, we first rewrite term $\langle z_t,\om(\theta_t)-\om(\theta_{t+1})\rangle=-\langle z_t, \nabla \om(\hat{\theta}_t)(\theta_{t+1}-\theta_t)\rangle=-\alpha\langle\nabla \om(\hat{\theta}_t)G_{t+1}(\theta_t,\omega_t)\rangle$, where $\hat{\theta}_t=c\theta_t+(1-c)\theta_{t+1}$ for some $c\in [0,1]$.  It can be shown that 
\begin{align}\label{eq:3}
       &\mathbb{E}\left[ z_{\tau_{\beta}}^\top\frac{\omega(\theta_{\tau_{\beta}})-\omega(\theta_{\tau_{\beta}+1})}{\beta}\right] 
     =-\frac{\alpha}{\beta}\mE[z_{\tau_{\beta}}^\top \nabla \omega(\hat{\theta}_{\tau_{\beta}})(G_{\tau_{\beta}+1}(\theta_{\tau_{\beta}},\omega_{\tau_{\beta}})-G_{\tau_{\beta}+1}(\theta_{\tau_{\beta}},\omega(\theta_{\tau_{\beta}})))]\nn\\
    &\quad-\frac{\alpha}{\beta}\mE\left[z_{\tau_{\beta}}^\top \nabla \omega(\hat{\theta}_{\tau_{\beta}})\left(G_{\tau_{\beta}+1}(\theta_{\tau_{\beta}},\omega(\theta_{\tau_{\beta}}))+\frac{\nabla J(\theta_{\tau_{\beta}})}{2}\right)\right] +\frac{\alpha}{\beta}\mE\left[z_{\tau_{\beta}}^\top \nabla \omega(\hat{\theta}_{\tau_{\beta}})\left(\frac{\nabla J(\theta_{\tau_{\beta}})}{2}\right)\right].
\end{align}
The first term in \eqref{eq:3} can be bounded in terms of $\|z_{\tb}\|^2$ using the Lipschitz property of $G_{\tb+1}(\theta,\omega)$ in $\omega$. The second term can be bounded using the uniform ergodicity of the MDP and the Lipschitz property of $z_0^\top \nabla \omega(\theta)\left(G_{\tau_{\beta}+1}(\theta,\omega(\theta))+\frac{\nabla J(\theta)}{2}\right)$ in $\theta$. The third term can be bounded in terms of $\|z_{\tb}\|^2$ and $\|\nabla J(\theta_{\tb})\|^2$. Combining all bounds together, we have the bound
on term $(f)$ in \eqref{eq:2}:
\begin{align}\label{eq:trackingterm33}
    &\left|\mathbb{E}\left[ z_{\tau_{\beta}}^\top\frac{\omega(\theta_{\tau_{\beta}})-\omega(\theta_{\tau_{\beta}+1})}{\beta}\right]\right|\nn\\
    &\leq  \mathcal{O}\left( \frac{\alpha}{\beta}\right)\mE\left[\left\|z_{\tb}\right\|^2\right]+ \mathcal{O}\left( \alpha\tb\right)\mE\left[\left\|z_{\tb}\right\|\right]+ \mathcal{O}\left( \alpha\tb\right)+\mathcal{O}\left(\frac{\alpha}{8\beta}\right)\mE\left[\left\|\nabla J(\theta_{\tau_{\beta}})\right\|^2\right].
 \end{align}
%
 We combine all the bounds on terms $(d), (e)$ and $(f)$ and hence get the error bound on \eqref{eq:2}:
\begin{align}\label{eq:trackinga/b}
    &\mathbb{E}\left[z_{t}^\top\left(-A_{\theta_{t}}z_{t}-\frac{1}{\beta}(z_{t+1}-z_{t})\right)\right] \leq\mathcal{O}\left( \frac{\alpha}{\beta}\right)\mE[\|z_t\|^2]+\mathcal{O}(\beta\tb)+\mathcal{O}\left( \frac{\alpha}{\beta}\right)\mE[\|\nabla J(\theta_t)\|^2].
\end{align}
Plugging the above bound in \eqref{eq:0}, we have the following recursive bound on the tracking error: 
 \begin{align}\label{eq:33}
      \mE\left[\|z_{t+1}\|^2\right]&\leq \mathcal{O}(1-\beta)\mE\left[\|z_t\|^2\right] +\mathcal{O}(\alpha)\mE\left[\|\nabla J(\theta_{t})\|^2\right]+\mathcal{O}(\beta^2\tb).
 \end{align}
Then by recursively applying the inequality in \eqref{eq:33} and summing up w.r.t. $t$ from $0$ to $T-1$, we obtain the bound on the tracking error ${\sum^{T-1}_{t=0} \mE[\|z_t\|^2]}/{T}$:
\begin{align*} 
    &\frac{\sum^{T-1}_{t=0}\mE[\|z_t\|^2]}{T}\leq \mathcal{O}\left(\frac{1}{T\beta}+\frac{\alpha}{\beta}\frac{\sum^{T-1}_{t=0}\mE[\|\nabla J(\theta_t)\|^2]}{T}+\beta\tau_{\beta}\right).
\end{align*}

\textbf{Step 3.}
In this step we bound the stochastic bias term $\mE\left[\left\langle \nabla J(\theta_t), \frac{\nabla J(\theta_t)}{2}+G_{t+1}(\theta_t, \omega(\theta_{t})) \right\rangle\right]$. Similarly, we add and subtract $\nabla J(\theta_0)$ and $G_{\tau_{\beta}+1}(\theta_0, \omega(\theta_{0}))$, and obtain that
\begin{align}\label{eq:32}
  & \left\langle \nabla J(\theta_{\tau_{\beta}}), \frac{\nabla J(\theta_{\tau_{\beta}})}{2}+G_{\tau_{\beta}+1}(\theta_{\tau_{\beta}}, \omega(\theta_{\tau_{\beta}})) \right\rangle\nn\\
    &= \left\langle \nabla J(\theta_0), \frac{\nabla J(\theta_0)}{2}+G_{\tau_{\beta}+1}(\theta_0, \omega(\theta_{0})) \right\rangle+\Bigg( \left\langle \nabla J(\theta_{\tau_{\beta}}), \frac{\nabla J(\theta_{\tau_{\beta}})}{2}+G_{\tau_{\beta}+1}(\theta_{\tau_{\beta}}, \omega(\theta_{\tau_{\beta}})) \right\rangle\nn\\
    &\quad-\left\langle \nabla J(\theta_0), \frac{\nabla J(\theta_0)}{2}+G_{\tau_{\beta}+1}(\theta_0, \omega(\theta_{0})) \right\rangle\Bigg),
\end{align}
which again can be bounded using the geometry uniform ergodicity of the MDP and the Lipschitz continuity of $\left\langle \nabla J(\theta), \frac{\nabla J(\theta)}{2}+G_{\tau_{\beta}+1}(\theta, \omega(\theta)) \right\rangle$.  
 
\textbf{Step 4.}
Plugging in the bounds on the tracking error and the stochastic bias and rearranging the terms, then it follows that  
$
     \frac{\sum^{T-1}_{t=0}\mathbb{E}[\|\nabla J(\theta_t)\|^2]}{T} 
    \leq    U\sqrt{\frac{\sum^{T-1}_{t=0}\mathbb{E}[\|\nabla J(\theta_t)\|^2]}{T}} +V,
$
where $U$ and $V$ are some constants depending on the step sizes, and the explicit definitions can be found in \eqref{eq:UV}. By solving the inequality of $ \frac{\sum^{T-1}_{t=0}\mathbb{E}[\|\nabla J(\theta_t)\|^2]}{T} $, we obtain that 
\begin{align*} 
     \frac{\sum^{T-1}_{t=0}\mathbb{E}[\|\nabla J(\theta_t)\|^2]}{T}\leq \mo\left(\beta\tb+\frac{1}{T\beta}+ \alpha\tb+\frac{1}{T\alpha}\right).
\end{align*}

\section{Conclusion}
In this paper, we extend the on-policy non-linear TDC algorithm to the off-policy setting, and characterize its non-asymptotic error bounds under both the i.i.d. and the Markovian settings. We show that the non-linear TDC algorithm converges as fast as $\mathcal O(1/\sqrt{T})$ (up to a factor of $\log T$). The techniques and tools developed in this paper can be used to analyze a wide range of value-based RL algorithms with general smooth function approximation. 

{\bf Limitations:} It is not clear yet whether the stationary points that the TDC converges to are second-order stationary or potentially saddle points.

{\bf Negative social impacts:} This work is a theoretical investigation of some fundamental RL algorithms, and therefore, the authors do not foresee any negative societal impact.   
   
\section{Acknowledgment}
The work of Yue Wang and Shaofeng Zou was supported in part by the National Science Foundation under Grants CCF-2106560 and CCF-2007783. Yi Zhou's work was supported in part by U.S. National Science Foundation under the Grant CCF-2106216.
   
\newpage

\bibliography{TDC}

\begin{thebibliography}{41}
\providecommand{\natexlab}[1]{#1}
\providecommand{\url}[1]{\texttt{#1}}
\expandafter\ifx\csname urlstyle\endcsname\relax
  \providecommand{\doi}[1]{doi: #1}\else
  \providecommand{\doi}{doi: \begingroup \urlstyle{rm}\Url}\fi

\bibitem[Archibald et~al.(1995)Archibald, McKinnon, and
  Thomas]{archibald1995generation}
TW~Archibald, KIM McKinnon, and LC~Thomas.
\newblock On the generation of markov decision processes.
\newblock \emph{Journal of the Operational Research Society}, 46\penalty0
  (3):\penalty0 354--361, 1995.

\bibitem[Baird(1995)]{baird1995residual}
Leemon Baird.
\newblock Residual algorithms: Reinforcement learning with function
  approximation.
\newblock In \emph{Machine Learning Proceedings}, pages 30--37. Elsevier, 1995.

\bibitem[Bertsekas(2011)]{bertsekas2011dynamic}
Dimitri~P Bertsekas.
\newblock {Dynamic Programming and Optimal Control 3rd edition, volume II}.
\newblock \emph{Belmont, MA: Athena Scientific}, 2011.

\bibitem[Bhandari et~al.(2018)Bhandari, Russo, and Singal]{bhandari2018finite}
Jalaj Bhandari, Daniel Russo, and Raghav Singal.
\newblock A finite time analysis of temporal difference learning with linear
  function approximation.
\newblock In \emph{Proc. Annual Conference on Learning Theory (CoLT)}, pages
  1691--1692. PMLR, 2018.

\bibitem[Bhatnagar et~al.(2009)Bhatnagar, Precup, Silver, Sutton, Maei, and
  Szepesv{\'a}ri]{bhatnagar2009convergent}
Shalabh Bhatnagar, Doina Precup, David Silver, Richard~S Sutton, Hamid Maei,
  and Csaba Szepesv{\'a}ri.
\newblock Convergent temporal-difference learning with arbitrary smooth
  function approximation.
\newblock In \emph{Proc. Advances in Neural Information Processing Systems
  (NIPS)}, volume~22, pages 1204--1212, 2009.

\bibitem[Borkar(2009)]{borkar2009stochastic}
Vivek~S Borkar.
\newblock \emph{Stochastic approximation: a dynamical systems viewpoint},
  volume~48.
\newblock Springer, 2009.

\bibitem[Cai et~al.(2019)Cai, Yang, Lee, and Wang]{cai2019neural}
Qi~Cai, Zhuoran Yang, Jason~D Lee, and Zhaoran Wang.
\newblock Neural temporal-difference learning converges to global optima.
\newblock In \emph{Proc. Advances in Neural Information Processing Systems
  (NeurIPS)}, pages 11312--11322, 2019.

\bibitem[Dalal et~al.(2018)Dalal, Sz{\"o}r{\'e}nyi, Thoppe, and
  Mannor]{dalal2018finite}
Gal Dalal, Bal{\'a}zs Sz{\"o}r{\'e}nyi, Gugan Thoppe, and Shie Mannor.
\newblock Finite sample analysis of two-timescale stochastic approximation with
  applications to reinforcement learning.
\newblock \emph{Proceedings of Machine Learning Research}, 75:\penalty0 1--35,
  2018.

\bibitem[Dalal et~al.(2020)Dalal, Szorenyi, and Thoppe]{dalal2020tale}
Gal Dalal, Balazs Szorenyi, and Gugan Thoppe.
\newblock A tale of two-timescale reinforcement learning with the tightest
  finite-time bound.
\newblock In \emph{Proc. AAAI Conference on Artificial Intelligence (AAAI)},
  pages 3701--3708, 2020.

\bibitem[Doan(2021)]{doan2021nonlinear}
Thinh~T Doan.
\newblock Nonlinear two-time-scale stochastic approximation: Convergence and
  finite-time performance.
\newblock In \emph{Learning for Dynamics and Control}, pages 47--47. PMLR,
  2021.

\bibitem[Du et~al.(2019)Du, Lee, Li, Wang, and Zhai]{du2019gradient}
Simon Du, Jason Lee, Haochuan Li, Liwei Wang, and Xiyu Zhai.
\newblock Gradient descent finds global minima of deep neural networks.
\newblock In \emph{Proc. International Conference on Machine Learning (ICML)},
  pages 1675--1685. PMLR, 2019.

\bibitem[Ghadimi and Lan(2013)]{ghadimi2013stochastic}
Saeed Ghadimi and Guanghui Lan.
\newblock Stochastic first- and zeroth-order methods for nonconvex stochastic
  programming.
\newblock \emph{SIAM Journal on Optimization}, 23\penalty0 (4):\penalty0
  2341--2368, 2013.

\bibitem[Gupta et~al.(2019)Gupta, Srikant, and Ying]{gupta2019finite}
Harsh Gupta, R~Srikant, and Lei Ying.
\newblock Finite-time performance bounds and adaptive learning rate selection
  for two time-scale reinforcement learning.
\newblock In \emph{Proc. Advances in Neural Information Processing Systems
  (NeurIPS)}, pages 4706--4715, 2019.

\bibitem[Hong et~al.(2020)Hong, Wai, Wang, and Yang]{hong2020two}
Mingyi Hong, Hoi-To Wai, Zhaoran Wang, and Zhuoran Yang.
\newblock A two-timescale framework for bilevel optimization: Complexity
  analysis and application to actor-critic.
\newblock \emph{arXiv preprint arXiv:2007.05170}, 2020.

\bibitem[Kaledin et~al.(2020)Kaledin, Moulines, Naumov, Tadic, and
  Wai]{kaledin2020finite}
Maxim Kaledin, Eric Moulines, Alexey Naumov, Vladislav Tadic, and Hoi-To Wai.
\newblock Finite time analysis of linear two-timescale stochastic approximation
  with markovian noise.
\newblock In \emph{Proc. Annual Conference on Learning Theory (CoLT)}, pages
  2144--2203. PMLR, 2020.

\bibitem[Karmakar and Bhatnagar(2018)]{karmakar2018two}
Prasenjit Karmakar and Shalabh Bhatnagar.
\newblock Two time-scale stochastic approximation with controlled {Markov}
  noise and off-policy temporal-difference learning.
\newblock \emph{Mathematics of Operations Research}, 43\penalty0 (1):\penalty0
  130--151, 2018.

\bibitem[Konda et~al.(2004)Konda, Tsitsiklis, et~al.]{konda2004convergence}
Vijay~R Konda, John~N Tsitsiklis, et~al.
\newblock Convergence rate of linear two-time-scale stochastic approximation.
\newblock \emph{The Annals of Applied Probability}, 14\penalty0 (2):\penalty0
  796--819, 2004.

\bibitem[Lakshminarayanan and Szepesvari(2018)]{lakshminarayanan2018linear}
Chandrashekar Lakshminarayanan and Csaba Szepesvari.
\newblock Linear stochastic approximation: {H}ow far does constant step-size
  and iterate averaging go?
\newblock In \emph{Proc. International Conference on Artificial Intelligence
  and Statistics}, pages 1347--1355, 2018.

\bibitem[Liu et~al.(2015)Liu, Liu, Ghavamzadeh, Mahadevan, and
  Petrik]{liu2015finite}
Bo~Liu, Ji~Liu, Mohammad Ghavamzadeh, Sridhar Mahadevan, and Marek Petrik.
\newblock Finite-sample analysis of proximal gradient td algorithms.
\newblock In \emph{Proc. International Conference on Uncertainty in Artificial
  Intelligence (UAI)}, pages 504--513. Citeseer, 2015.

\bibitem[Ma et~al.(2020)Ma, Zhou, and Zou]{ma2020variance}
Shaocong Ma, Yi~Zhou, and Shaofeng Zou.
\newblock Variance-reduced off-policy {TDC} learning: Non-asymptotic
  convergence analysis.
\newblock In \emph{Proc. Advances in Neural Information Processing Systems
  (NeurIPS)}, volume~33, pages 14796--14806, 2020.

\bibitem[Ma et~al.(2021)Ma, Zhou, and Zou]{ma2021greedygq}
Shaocong Ma, Yi~Zhou, and Shaofeng Zou.
\newblock Greedy-{GQ} with variance reduction: Finite-time analysis and
  improved complexity.
\newblock In \emph{Proc. International Conference on Learning Representations
  (ICLR)}, 2021.

\bibitem[Maei(2011)]{maei2011gradient}
Hamid~Reza Maei.
\newblock Gradient temporal-difference learning algorithms.
\newblock \emph{Thesis, University of Alberta}, 2011.

\bibitem[Maei et~al.(2010)Maei, Szepesv{\'a}ri, Bhatnagar, and
  Sutton]{maei2010toward}
Hamid~Reza Maei, Csaba Szepesv{\'a}ri, Shalabh Bhatnagar, and Richard~S Sutton.
\newblock Toward off-policy learning control with function approximation.
\newblock In \emph{Proc. International Conference on Machine Learning (ICML)},
  pages 719--726, 2010.

\bibitem[Melo et~al.(2008)Melo, Meyn, and Ribeiro]{melo2008analysis}
Francisco~S Melo, Sean~P Meyn, and M~Isabel Ribeiro.
\newblock An analysis of reinforcement learning with function approximation.
\newblock In \emph{Proc. International Conference on Machine Learning (ICML)},
  pages 664--671. ACM, 2008.

\bibitem[Miyato et~al.(2018)Miyato, Kataoka, Koyama, and
  Yoshida]{miyato2018spectral}
Takeru Miyato, Toshiki Kataoka, Masanori Koyama, and Yuichi Yoshida.
\newblock Spectral normalization for generative adversarial networks.
\newblock In \emph{Proc. International Conference on Learning Representations
  (ICLR)}, 2018.

\bibitem[Mokkadem et~al.(2006)Mokkadem, Pelletier,
  et~al.]{mokkadem2006convergence}
Abdelkader Mokkadem, Mariane Pelletier, et~al.
\newblock Convergence rate and averaging of nonlinear two-time-scale stochastic
  approximation algorithms.
\newblock \emph{The Annals of Applied Probability}, 16\penalty0 (3):\penalty0
  1671--1702, 2006.

\bibitem[Neyshabur(2017)]{neyshabur2017implicit}
Behnam Neyshabur.
\newblock Implicit regularization in deep learning.
\newblock \emph{arXiv preprint arXiv:1709.01953}, 2017.

\bibitem[Perkins and Precup(2003)]{perkins2003convergent}
Theodore~J Perkins and Doina Precup.
\newblock A convergent form of approximate policy iteration.
\newblock In \emph{Proc. Advances in Neural Information Processing Systems
  (NIPS)}, pages 1627--1634, 2003.

\bibitem[Srikant and Ying(2019)]{srikant2019finite}
Rayadurgam Srikant and Lei Ying.
\newblock Finite-time error bounds for linear stochastic approximation andtd
  learning.
\newblock In \emph{Proc. Annual Conference on Learning Theory (CoLT)}, pages
  2803--2830. PMLR, 2019.

\bibitem[Sun et~al.(2020)Sun, Wang, Giannakis, Yang, and Yang]{sun2020finite}
Jun Sun, Gang Wang, Georgios~B Giannakis, Qinmin Yang, and Zaiyue Yang.
\newblock Finite-sample analysis of decentralized temporal-difference learning
  with linear function approximation.
\newblock In \emph{Proeedings of the International Workshop on Artificial
  Intelligence and Statistics}, 2020.

\bibitem[Sutton and Barto(2018)]{sutton2018reinforcement}
Richard~S. Sutton and Andrew~G. Barto.
\newblock \emph{Reinforcement Learning: An Introduction, Second Edition}.
\newblock The MIT Press, Cambridge, Massachusetts, 2018.

\bibitem[Sutton et~al.(2009{\natexlab{a}})Sutton, Maei, and
  Szepesv{\'a}ri]{sutton2009acov}
Richard~S Sutton, Hamid~R Maei, and Csaba Szepesv{\'a}ri.
\newblock A convergent $ {O} (n) $ temporal-difference algorithm for off-policy
  learning with linear function approximation.
\newblock In \emph{Proc. Advances in Neural Information Processing Systems
  (NIPS)}, pages 1609--1616, 2009{\natexlab{a}}.

\bibitem[Sutton et~al.(2009{\natexlab{b}})Sutton, Maei, Precup, Bhatnagar,
  Silver, Szepesv{\'a}ri, and Wiewiora]{sutton2009fast}
Richard~S Sutton, Hamid~Reza Maei, Doina Precup, Shalabh Bhatnagar, David
  Silver, Csaba Szepesv{\'a}ri, and Eric Wiewiora.
\newblock Fast gradient-descent methods for temporal-difference learning with
  linear function approximation.
\newblock In \emph{Proc. International Conference on Machine Learning (ICML)},
  pages 993--1000, 2009{\natexlab{b}}.

\bibitem[Tsitsiklis and Van~Roy(1997)]{tsitsiklis1997analysis}
John~N Tsitsiklis and Benjamin Van~Roy.
\newblock An analysis of temporal-difference learning with function
  approximation.
\newblock \emph{IEEE transactions on automatic control}, 42\penalty0
  (5):\penalty0 674--690, 1997.

\bibitem[Wang and Zou(2020)]{wang2020finite}
Yue Wang and Shaofeng Zou.
\newblock Finite-sample analysis of greedy-{GQ} with linear function
  approximation under markovian noise.
\newblock In \emph{Proc. International Conference on Uncertainty in Artificial
  Intelligence (UAI)}, pages 11--20. PMLR, 2020.

\bibitem[Wang et~al.(2017)Wang, Chen, Liu, Ma, and Liu]{wang2017finite}
Yue Wang, Wei Chen, Yuting Liu, Zhi-Ming Ma, and Tie-Yan Liu.
\newblock Finite sample analysis of the gtd policy evaluation algorithms in
  markov setting.
\newblock In \emph{Proc. Advances in Neural Information Processing Systems
  (NIPS)}, pages 5504--5513, 2017.

\bibitem[Wu et~al.(2020)Wu, Zhang, Xu, and Gu]{wu2020finite}
Yue Wu, Weitong Zhang, Pan Xu, and Quanquan Gu.
\newblock A finite time analysis of two time-scale actor critic methods.
\newblock In \emph{Proc. Advances in Neural Information Processing Systems
  (NeurIPS)}, 2020.

\bibitem[Xu and Gu(2020)]{xu2020finite}
Pan Xu and Quanquan Gu.
\newblock A finite-time analysis of q-learning with neural network function
  approximation.
\newblock In \emph{Proc. International Conference on Machine Learning (ICML)},
  pages 10555--10565. PMLR, 2020.

\bibitem[Xu and Liang(2021)]{xu2021sample}
Tengyu Xu and Yingbin Liang.
\newblock Sample complexity bounds for two timescale value-based reinforcement
  learning algorithms.
\newblock In \emph{Proc. International Conference on Artifical Intelligence and
  Statistics (AISTATS)}, pages 811--819. PMLR, 2021.

\bibitem[Xu et~al.(2019)Xu, Zou, and Liang]{xu2019two}
Tengyu Xu, Shaofeng Zou, and Yingbin Liang.
\newblock Two time-scale off-policy {TD} learning: Non-asymptotic analysis over
  {Markovian} samples.
\newblock In \emph{Proc. Advances in Neural Information Processing Systems
  (NeurIPS)}, pages 10633--10643, 2019.

\bibitem[Zou et~al.(2019)Zou, Xu, and Liang]{zou2019finite}
Shaofeng Zou, Tengyu Xu, and Yingbin Liang.
\newblock Finite-sample analysis for {SARSA} with linear function
  approximation.
\newblock In \emph{Proc. Advances in Neural Information Processing Systems
  (NeurIPS)}, pages 8665--8675, 2019.

\end{thebibliography}

\newpage
\onecolumn
\appendix

\addcontentsline{toc}{section}{Appendix} 
\part{Appendix} 
\parttoc

We first introduce some notations. In the following proofs, $\|a\|$ denotes the $\ell_2$ norm if $a$ is a vector; and $\|A\|$ denotes the operator norm if $A$ is a matrix.

In \Cref{sec:app_a}, we prove the Lipschitz continuity of some important functions, including $\omega(\theta)$, $\nabla \om(\theta)$ and the gradient $\nabla J(\theta)$ of objective function. In \Cref{section:iid}, we present the non-asymptotic analysis for the i.i.d. setting. In \Cref{section:maralg}, we present the non-asymptotic analysis for the Markovian setting. 
In \cref{sec:experiments}, we present some numerical experiments.
\section{Useful Lemmas}\label{sec:app_a}
\subsection{Lipschitz Continuity of $\omega(\theta)$}
In this section, we show that $\omega(\theta)$ is Lipschitz in $\theta$.
\begin{Lemma}\label{lemma:omegalip}
For any $\theta, \theta'\in\mathbb{R}^N$, we have that
\begin{align}\label{eq:omegalip}
    \|\omega(\theta)-\omega(\theta')\|\leq  L_{\omega}\|\theta-\theta'\|,
\end{align}
where $L_{\om}=\frac{1}{\lambda_v}\left((1+\gamma)C_{\phi}^2+(r_{\max}+(1+\gamma)C_v)D_v\right)+\frac{2C_{\phi}^2 D_v}{\lambda^2_v} (r_{\max}+(1+\gamma) C_v)$.
\end{Lemma}
\begin{proof}
Recall that 
\begin{align}
    \omega(\theta)&=\mathbb{E}_{\mu^{\pi_b}}[\phi_{\theta}(S)\phi_{\theta}(S)^\top]^{-1}\mathbb{E}_{\mu^{\pi_b}}[\rho(S,A)\delta_{S,A,S'}(\theta)\phi_{\theta}(S)]\nn\\
    &=A_{\theta}^{-1}\mathbb{E}_{\mu^{\pi_b}}[\rho(S,A)\delta_{S,A,S'}(\theta)\phi_{\theta}(S)],
\end{align}
hence we can show the conclusion by showing that $A_{\theta}^{-1}$ and $\mathbb{E}_{\mu^{\pi_b}}[\rho(S,A)\delta_{S,A,S'}(\theta)\phi_{\theta}(S)]$ are both Lipschitz and bounded. 

From Assumption \ref{ass:solvable}, we know that
\begin{align}\label{eq:a-1bpund}
    \|A_{\theta}^{-1}\|\leq \frac{1}{\lambda_v}.
\end{align}
We also show that 
\begin{align}\label{eq:a-1lip}
    &\|A_{\theta}^{-1}-A_{\theta'}^{-1}\|\nn\\
    &=\|A_{\theta}^{-1}A_{\theta'}A_{\theta'}^{-1}-A_{\theta}^{-1}A_{\theta}A_{\theta'}^{-1}\|\nn\\
    &= \| A_{\theta}^{-1}(A_{\theta'}-A_{\theta})A_{\theta'}^{-1}\|\nn\\
    &\leq \frac{2C_{\phi} D_v}{\lambda^2_v}\|\theta-\theta'\|,
\end{align}
which is from the fact that $\|A_{\theta}-A_{\theta'}\|=\left\|\mathbb{E}_{\mu^{\pi_b}}[\phi_{\theta}(S)\phi_{\theta}(S)^\top]-\mathbb{E}_{\mu^{\pi_b}}[\phi_{\theta'}(S)\phi_{\theta'}(S)^\top] \right\|\leq 2C_{\phi}D_v  \|\theta-\theta'\|$.

By Assumption \ref{ass:bound} and the boundedness of the reward function, it can be shown that for any $\theta\in \mathbb{R}^N$ and any  $(s,a,s')\in \mathcal{S}\times\mathcal{A}\times\mathcal{S}$, 
\begin{align}
    |\delta_{s,a,s'}(\theta)|=|r(s, a, s')+\gamma V_{\theta}(s')-V_{\theta}(s)|\leq r_{\max}+(1+\gamma) C_v.
\end{align}
We then show that $\delta_{s,a,s'}(\theta)$ is Lipschitz, i.e., for any $\theta, \theta' \in \mathbb{R}^N$ and any  $(s,a,s')\in \mathcal{S}\times\mathcal{A}\times\mathcal{S}$,
\begin{align}\label{eq:deltalip}
    &|\delta_{s,a,s'}(\theta)-\delta_{s,a,s'}(\theta')|\nn\\
    &=|\gamma V_{\theta}(s')-V_{\theta}(s)-\gamma V_{\theta'}(s')-V_{\theta'}(s)|\nn\\
    &\leq (\gamma+1)C_{\phi}\|\theta-\theta'\|.
\end{align}
Hence, the function $\|\mathbb{E}_{\mu^{\pi_b}}[\rho(S,A)\delta_{S,A,S'}(\theta)\phi_{\theta}(S)]\|$ is Lipschitz:
\begin{align}\label{eq:rhodeltaphilip}
    &\|\mathbb{E}_{\mu^{\pi_b}}[\rho(S,A)\delta_{S,A,S'}(\theta)\phi_{\theta}(S)]-\mathbb{E}_{\mu^{\pi_b}}[\rho(S,A)\delta_{S,A,S'}(\theta')\phi_{\theta'}(S)]\|\nn\\
    &=\|\mathbb{E}_{\mu^{\pi_b}}[\rho(S,A)\delta_{S,A,S'}(\theta)\phi_{\theta}(S)]-\mathbb{E}_{\mu^{\pi_b}}[\rho(S,A)\delta_{S,A,S'}(\theta')\phi_{\theta}(S)]\nn\\
    &\quad+\mathbb{E}_{\mu^{\pi_b}}[\rho(S,A)\delta_{S,A,S'}(\theta')\phi_{\theta}(S)]-\mathbb{E}_{\mu^{\pi_b}}[\rho(S,A)\delta_{S,A,S'}(\theta')\phi_{\theta'}(S)]\|\nn\\
    &\leq \|\mathbb{E}_{\mu^{\pi_b}}[\rho(S,A)\delta_{S,A,S'}(\theta)\phi_{\theta}(S)]-\mathbb{E}_{\mu^{\pi_b}}[\rho(S,A)\delta_{S,A,S'}(\theta')\phi_{\theta}(S)]\|\nn\\
    &\quad+\|\mathbb{E}_{\mu^{\pi_b}}[\rho(S,A)\delta_{S,A,S'}(\theta')\phi_{\theta}(S)]-\mathbb{E}_{\mu^{\pi_b}}[\rho(S,A)\delta_{S,A,S'}(\theta')\phi_{\theta'}(S)]\|\nn\\
    &\leq \mathbb{E}_{\mu^{\pi_b}}[\rho(S,A)|\delta_{S,A,S'}(\theta)-\delta_{S,A,S'}(\theta')|\|\phi_{\theta}(S) \|]\nn\\
    &\quad+\mathbb{E}_{\mu^{\pi_b}}[\rho(S,A)|\delta_{S,A,S'}(\theta')|\|\phi_{\theta}(S)-\phi_{\theta'}(S)\|]\nn\\
    &\overset{(a)}{\leq} (1+\gamma)C_{\phi}^2\|\theta-\theta'\|+(r_{\max}+(1+\gamma)C_v)D_v\|\theta-\theta'\|\nn\\
    &=\left((1+\gamma)C_{\phi}^2+(r_{\max}+(1+\gamma)C_v)D_v\right)\|\theta-\theta'\|,
\end{align}
where $(a)$ is from \eqref{eq:deltalip} and the fact that $\mathbb{E}_{\mu^{\pi_b}}[\rho(S,A)]=1$.
Also $\left\|\mathbb{E}_{\mu^{\pi_b}}[\rho(S,A)\delta_{S,A,S'}(\theta)\phi_{\theta}(S)]\right\|$ can be upper bounded as follows:
\begin{align}\label{eq:rhodelphibound}
    \|\mathbb{E}_{\mu^{\pi_b}}[\rho(S,A)\delta_{S,A,S'}(\theta)\phi_{\theta}(S)]\|\leq  C_{\phi}(r_{\max}+(1+\gamma) C_v). 
\end{align}
Combining \eqref{eq:a-1bpund}, \eqref{eq:a-1lip}, \eqref{eq:rhodelphibound} and \eqref{eq:rhodeltaphilip}, we show that $\omega(\cdot)$ is Lipschitz in $\theta$:
\begin{align}
     &\|\omega(\theta)-\omega(\theta')\|\nn\\
     &\leq \left(\frac{1}{\lambda_v}\left((1+\gamma)C_{\phi}^2+(r_{\max}+(1+\gamma)C_v)D_v\right)+\frac{2C_{\phi}^2 D_v}{\lambda^2_v} (r_{\max}+(1+\gamma) C_v)\right) \|\theta-\theta'\|\nn\\
     &\triangleq L_{\omega}\| \theta-\theta'\|,
\end{align}
where $L_{\om}=\frac{1}{\lambda_v}\left((1+\gamma)C_{\phi}^2+(r_{\max}+(1+\gamma)C_v)D_v\right)+\frac{2C_{\phi}^2 D_v}{\lambda^2_v} (r_{\max}+(1+\gamma) C_v)$.
\end{proof}

\subsection{Lipschitz Continuity of $\nabla \omega(\theta)$}
In this section, we show that $\nabla \omega(\theta)$ is Lipschitz. 
\begin{Lemma}\label{lemma:graomegalip}
For any $\theta, \theta' \in \mathbb{R}^N$, it follows that 
\begin{align}\label{eq:nablaomegalip}
    \|\nabla \omega(\theta)-\nabla \omega(\theta')\|\leq D_{\omega}\| \theta-\theta'\|,  
\end{align}
where 
\begin{align}
    D_{\om}&=\left( \frac{(C_{\phi}L_v+2D_v^2+D_vC_{\phi})}{\lambda_v^2}+\frac{8C_{\phi}^2D_v^2}{\lambda_v^3}\right)C_{\phi}(r_{\max}+C_v+\gamma C_v)\nn\\
    &\quad+\frac{4C_{\phi}D_v}{\lambda_v^2}\left( C_{\phi}^2(1+\gamma)+D_v(r_{\max}+(1+\gamma)C_v)\right)\nn\\
    &\quad+\frac{3C_{\phi}D_v(1+\gamma)+L_v(r_{\max}+(1+\gamma)C_v)}{\lambda_v}.
\end{align}
\end{Lemma}
\begin{proof}
Recall the definition of $\om(\theta)=A_{\theta}^{-1}\mathbb{E}_{\mu^{\pi_b}}[\rho(S,A)\delta_{S,A,S'}(\theta)\phi_{\theta}(S)]$, hence we have 
\begin{align}\label{eq:gradomegadef}
    \nabla \om(\theta)&=-A_{\theta}^{-1}(\nabla A_{\theta}) A_{\theta}^{-1}\mathbb{E}_{\mu^{\pi_b}}[\rho(S,A)\delta_{S,A,S'}(\theta)\phi_{\theta}(S)]\nn\\ &\quad+A_{\theta}^{-1}\mathbb{E}_{\mu^{\pi_b}}[\nabla \rho(S,A)\delta_{S,A,S'}(\theta)\phi_{\theta}(S)],
\end{align}
where the tensor $\nabla A_{\theta}$ can be equivalently viewed as an operator: $\mathbb{R}^N\to \mathbb{R}^{N\times N}$, i.e., $\nabla A_{\theta} (w)=\nabla (A_{\theta}w)$ for any $w\in\mathbb{R}^N$. 

We show that the operator norm of $\nabla A_{\theta}$ is bounded as follows:
\begin{align}
    \| \nabla A_{\theta}\|&=\sup_{\|w\|=1} \|\nabla A_{\theta}(w) \|\nn\\
    &=\sup_{\|w\|=1} \|\nabla (A_{\theta}w) \|\nn\\
    &=\sup_{\|w\|=1} \|\nabla \mE_{\mu^{\pi_b}}[\phi_{\theta}(S)\phi_{\theta}(S)^\top w] \|\nn\\
    &=\sup_{\|w\|=1} \|\mE_{\mu^{\pi_b}}[(\phi_{\theta}(S)^\top w) \nabla\phi_{\theta}(S)] +\mE_{\mu^{\pi_b}}[\phi_{\theta}(S) (\nabla \phi_{\theta}(S)^\top w)^\top]\|\nn\\
    &\leq \sup_{\|w\|=1} 2C_{\phi}D_v\|w\|\nn\\
    &= 2C_{\phi}D_v. 
\end{align}

The Lipschitz continuous of $\nabla A_{\theta}$ can be shown as follows:
\begin{align}
    &\|\nabla A_{\theta} -\nabla A_{\theta'}\|\nn\\
    &=\sup_{\|w\|=1} \|\nabla (A_{\theta}w)-\nabla (A_{\theta'}w) \|\nn\\
    &=\sup_{\|w\|=1} \| \mE_{\mu^{\pi_b}}[\nabla \phi_{\theta}(S) (\phi_{\theta}(S)^\top w)+ (\nabla \phi_{\theta}(S)^\top w)\phi_{\theta}(S)^\top-\nabla \phi_{\theta'}(S) (\phi_{\theta'}(S)^\top w)\nn\\
    &\quad- (\nabla \phi_{\theta'}(S)^\top w)\phi_{\theta'}(S)^\top] \|\nn\\
    &\leq \sup_{\|w\|=1} (C_{\phi}L_v+2D_v^2+D_vC_{\phi}) \|\theta-\theta'\|\|w\|\nn\\
    &=(C_{\phi}L_v+2D_v^2+D_vC_{\phi})\|\theta-\theta'\|.
\end{align}

Then we conclude that the operator norm of $-A_{\theta}^{-1}(\nabla A_{\theta})$ is upper bounded by $\frac{2C_{\phi}D_v}{\lambda_v}$, and is Lipschitz with constant $\frac{(C_{\phi}L_v+2D_v^2+D_vC_{\phi})}{\lambda_v}+\frac{4C_{\phi}^2D_v^2}{\lambda_v^2}$.
It can be further seen that $-A_{\theta}^{-1}(\nabla A_{\theta})A_{\theta}^{-1}$ is upper bounded by $\frac{2C_{\phi}D_v}{\lambda_v^2}$, and Lipschitz with constant $\frac{(C_{\phi}L_v+2D_v^2+D_vC_{\phi})}{\lambda_v^2}+\frac{8C_{\phi}^2D_v^2}{\lambda_v^3}$. 

Recall that we have shown in \eqref{eq:rhodelphibound} that 
\begin{align} 
    &\|\mathbb{E}_{\mu^{\pi_b}}[\rho(S,A)\delta_{S,A,S'}(\theta)\phi_{\theta}(S)]-\mathbb{E}_{\mu^{\pi_b}}[\rho(S,A)\delta_{S,A,S'}(\theta')\phi_{\theta'}(S)]\|\nn\\
    &\leq
     \left((1+\gamma)C_{\phi}^2+(r_{\max}+(1+\gamma)C_v)D_v\right)\|\theta-\theta'\|,
\end{align}
and it is upper bounded by $C_{\phi}(r_{\max}+(1+\gamma)C_V)$. Hence we have that $-A_{\theta}^{-1}(\nabla A_{\theta}) A_{\theta}^{-1}\mathbb{E}_{\mu^{\pi_b}}[\rho(S,A)\delta_{S,A,S'}(\theta)\phi_{\theta}(S)]$ can be upper bounded by $(r_{\max}+(1+\gamma)C_V)\frac{2C_{\phi}^2D_v}{\lambda_v^2}$, and it is Lipschitz with constant $\left( \frac{(C_{\phi}L_v+2D_v^2+D_vC_{\phi})}{\lambda_v^2}+\frac{8C_{\phi}^2D_v^2}{\lambda_v^3}\right)C_{\phi}(r_{\max}+C_v+\gamma C_v)+\frac{2C_{\phi}D_v}{\lambda_v^2} \left((1+\gamma)C_{\phi}^2+(r_{\max}+(1+\gamma)C_v)D_v\right)\triangleq L_A$.

For the second term of \eqref{eq:gradomegadef}, we also show it is Lipschitz as follows. 
First note that $\nabla \delta_{s,a,s'}(\theta)\phi_{\theta}(s)=\nabla \delta_{s,a,s'}(\theta)\phi_{\theta}(s)^\top+\delta_{s,a,s'}(\theta)\nabla \phi_{\theta}(s)$, hence we know $\mathbb{E}_{\mu^{\pi_b}}[\nabla \rho(S,A)\delta_{S,A,S'}(\theta)\phi_{\theta}(S)]$ can be upper bounded by
$C_{\phi}^2(1+\gamma)+D_v(r_{\max}+(1+\gamma)C_v)$, and is Lipschitz with constant $3C_{\phi}D_v(1+\gamma)+L_v(r_{\max}+(1+\gamma)C_v)$. Finally we conclude that the second term in \eqref{eq:gradomegadef} $A_{\theta}^{-1}\mathbb{E}_{\mu^{\pi_b}}[\nabla \rho(S,A)\delta_{S,A,S'}(\theta)\phi_{\theta}(S)]$ is Lipschitz with constant $\frac{2C_{\phi}D_v}{\lambda_v^2}\left( C_{\phi}^2(1+\gamma)+D_v(r_{\max}+(1+\gamma)C_v)\right)+\frac{3C_{\phi}D_v(1+\gamma)+L_v(r_{\max}+(1+\gamma)C_v)}{\lambda_v}\triangleq L_A'$. 

Hence $\nabla \om(\theta)$ is Lipschitz with constant $L_A+L_A'\triangleq D_{\om}$, where
\begin{align}
    D_{\om}&=\left( \frac{(C_{\phi}L_v+2D_v^2+D_vC_{\phi})}{\lambda_v^2}+\frac{8C_{\phi}^2D_v^2}{\lambda_v^3}\right)C_{\phi}(r_{\max}+C_v+\gamma C_v)\nn\\
    &\quad+\frac{4C_{\phi}D_v}{\lambda_v^2}\left( C_{\phi}^2(1+\gamma)+D_v(r_{\max}+(1+\gamma)C_v)\right)\nn\\
    &\quad+\frac{3C_{\phi}D_v(1+\gamma)+L_v(r_{\max}+(1+\gamma)C_v)}{\lambda_v}.
\end{align}
\end{proof}

\subsection{Smoothness of $J(\theta)$}
In the following lemma, we show that the objective function $J(\theta)$ is $L_J$-smooth. We note that the smoothness of $J(\theta)$ is assumed in \citep{xu2021sample} instead of being proved as in this paper.
\begin{Lemma}\label{lemma:Lsmooth}
$ J(\theta)$ is $L_J$-smooth, i.e., for any $\theta, \theta' \in \mathbb{R}^N$, 
\begin{align}
    \|\nabla J(\theta)-\nabla J(\theta')\|\leq L_J \|\theta-\theta'\|,
\end{align}
where 
\begin{align}
        L_J=&2 \left((1+\gamma)C_{\phi}^2+(r_{\max}+(1+\gamma)C_v)D_v\right) +2\gamma\left(C_{\phi}^2L_{\omega} +2D_v\frac{C^2_{\phi}}{\lambda_v}(r_{\max}+(1+\gamma) C_v)\right)\nn\\
    &\quad+2\bigg(\left( D_vR_{\omega}+C_{\phi}L_{\omega}+(1+\gamma)C_{\phi}\right)D_vR_{\omega}\nn\\
    &\quad+\left( R_{\omega}L_V+D_vL_{\omega}\right)((r_{\max}+(1+\gamma)C_v)+C_{\phi}R_{\om})\bigg).
\end{align}
\end{Lemma}
\begin{proof}
Before we prove the main statement, we first drive some boundedness and Lipschitz properties. 
Recall that 
\begin{align}
    -\frac{\nabla J(\theta)}{2}&=\mathbb{E}_{\mu^{\pi_b}}\bigg[\big(\rho(S,A)\delta_{S,A,S'}(\theta)\phi_{\theta}(S)-\gamma\rho(S,A) \phi_{\theta}(S') \phi_{\theta}(S)^\top \om(\theta)\nn\\
    &\quad-h_{S,A,S'}(\theta,\om(\theta))\big)\bigg],\\
    \omega(\theta)&=\mathbb{E}_{\mu^{\pi_b}}[\phi_{\theta}(S)\phi_{\theta}(S)^\top]^{-1}\mathbb{E}_{\mu^{\pi_b}}[\rho(S,A)\delta_{S,A,S'}(\theta)\phi_{\theta}(S)],\\
    h_{s,a,s'}(\theta,\omega(\theta))&=(\rho(s,a)\delta_{s,a,s'}(\theta)- \phi_{\theta}(s)^\top\omega(\theta))\nabla^2V_{\theta}(s)\omega(\theta).
\end{align}
We have shown in Lemma \ref{lemma:omegalip} that for any $\theta\in \mathbb{R}^N$ and any  $(s,a,s')\in \mathcal{S}\times\mathcal{A}\times\mathcal{S}$, 
\begin{align}\label{eq:deltabounds}
    |\delta_{s,a,s'}(\theta)|=|r(s, a, s')+\gamma V_{\theta}(s')-V_{\theta}(s)|&\leq r_{\max}+(1+\gamma) C_v;
\end{align}
and that 
\begin{align}\label{eq:deltaphilip}
    &\|\mathbb{E}_{\mu^{\pi_b}}[\rho(S,A)\delta_{S,A,S'}(\theta)\phi_{\theta}(S)]-\mathbb{E}_{\mu^{\pi_b}}[\rho(S,A)\delta_{S,A,S'}(\theta')\phi_{\theta'}(S)]\|\nn\\
    &\leq\left((1+\gamma)C_{\phi}^2+(r_{\max}+(1+\gamma)C_v)D_v\right)\|\theta-\theta'\|.
\end{align}
Also it is easy to see from the definition that   
\begin{align}\label{eq:omegabound}
     \|\omega(\theta)\|&\leq \frac{C_{\phi}}{\lambda_v}(r_{\max}+(1+\gamma) C_v)\triangleq R_{\om}.         
\end{align}
Hence  the Lipschitz continuity of $\mathbb{E}_{\mu^{\pi_b}}[\rho(S,A)\phi_{\theta}(S')\phi_{\theta}(S)^\top]\omega(\theta)$ can be shown as follows
\begin{align}\label{eq:phiphiomegalip}
    &\left\|\mathbb{E}_{\mu^{\pi_b}}[\rho(S,A)\phi_{\theta}(S')\phi_{\theta}(S)^\top]\omega(\theta)-\mathbb{E}_{\mu^{\pi_b}}[\rho(S,A)\phi_{\theta'}(S')\phi_{\theta'}(S)^\top]\omega(\theta')\right\|\nn\\
    &\leq \left\|\mathbb{E}_{\mu^{\pi_b}}[\rho(S,A)\phi_{\theta}(S')\phi_{\theta}(S)^\top]\omega(\theta)-\mathbb{E}_{\mu^{\pi_b}}[\rho(S,A)\phi_{\theta}(S')\phi_{\theta}(S)^\top]\omega(\theta')\right\|\nn\\
    &\quad+\left\|\mathbb{E}_{\mu^{\pi_b}}[\rho(S,A)\phi_{\theta}(S')\phi_{\theta}(S)^\top]\omega(\theta')-\mathbb{E}_{\mu^{\pi_b}}[\rho(S,A)\phi_{\theta'}(S')\phi_{\theta'}(S)^\top]\omega(\theta')\right\|\nn\\
    &\overset{(a)}{\leq} C_{\phi}^2L_{\omega}\|\theta-\theta'\|+2C_{\phi}D_vR_{\omega}\|\theta-\theta'\|\nn\\
    &=\left(C_{\phi}^2L_{\omega} +2D_v\frac{C^2_{\phi}}{\lambda_v}(r_{\max}+(1+\gamma) C_v)\right)\|\theta-\theta'\|,
\end{align}
where $(a)$ is due to the fact that $\omega(\theta)$ is Lipschitz in \eqref{eq:omegalip} and the fact that  
\begin{align}
    \left\|\mathbb{E}_{\mu^{\pi_b}}[\rho(S,A)\phi_{\theta}(S')\phi_{\theta}(S)^\top]-\mathbb{E}_{\mu^{\pi_b}}[\rho(S,A)\phi_{\theta'}(S')\phi_{\theta'}(S)^\top]\right\|\leq 2C_{\phi}D_v\|\theta-\theta'\|.
\end{align}

We then show that the function $h_{s,a,s'}(\theta,\omega(\theta))$ is Lipschitz in $\theta$ as follows. We first note that for any $s\in \mathcal{S}$ and $\theta,\theta'\in \mathbb{R}^N$, 
\begin{align}
    &\|\phi_{\theta}(s)^\top\omega(\theta)-\phi_{\theta'}(s)^\top\omega(\theta')\|\nn\\
    &\leq \|\phi_{\theta}(s)^\top\omega(\theta)-\phi_{\theta'}(s)^\top\omega(\theta)\|+\|\phi_{\theta'}(s)^\top\omega(\theta)-\phi_{\theta'}(s)^\top\omega(\theta')\|\nn\\
    &\leq \left(D_vR_{\omega}+C_{\phi}L_{\omega}\right)\|\theta-\theta'\|.
\end{align}
This implies that for any  $(s,a,s')\in \mathcal{S}\times\mathcal{A}\times\mathcal{S}$ and $\theta,\theta'\in \mathbb{R}^N$,
\begin{align}\label{eq:h1lip}
    &\|\rho(s,a)\delta_{s,a,s'}(\theta)- \phi_{\theta}(s)^\top\omega(\theta)-\rho(s,a)\delta_{s,a,s'}(\theta')+\phi_{\theta'}(s)^\top\omega(\theta') \|\nn\\
    &\leq \left( D_vR_{\omega}+C_{\phi}L_{\omega}+(1+\gamma)C_{\phi}\rho(s,a) \right)\|\theta-\theta'\|.
\end{align}
We also show the following function is Lipschitz:
\begin{align}\label{eq:h2lip}
    &\|\nabla^2 V_{\theta}(s)\omega(\theta)-\nabla^2 V_{\theta'}(s)\omega(\theta') \|\nn\\
    &\leq \|\nabla^2 V_{\theta}(s)\omega(\theta)-\nabla^2 V_{\theta'}(s)\omega(\theta)\|+\|\nabla^2 V_{\theta'}(s)\omega(\theta)-\nabla^2 V_{\theta'}(s)\omega(\theta')\|\nn\\
    &\leq R_{\omega}L_V\|\theta-\theta'\|+D_vL_{\omega}\|\theta-\theta'\|\nn\\
    &=\left( R_{\omega}L_V+D_vL_{\omega}\right)\|\theta-\theta'\|.
\end{align}
Combining \eqref{eq:h1lip} and \eqref{eq:h2lip}, it can be shown that $h_{s,a,s'}(\theta,\omega(\theta))$ is Lipschitz in $\theta$ as follows
\begin{align}\label{eq:hlip}
    &\|h_{s,a,s'}(\theta,\omega(\theta))-h_{s,a,s'}(\theta',\omega(\theta'))\|\nn\\
    &=\| \left(\rho(s,a)\delta_{s,a,s'}(\theta)-\phi_{\theta}(s)^\top\omega(\theta)\right)\nabla^2 V_{\theta}(s)\omega(\theta) \nn\\
    &\quad- \left(\rho(s,a)\delta_{s,a,s'}(\theta')-\phi_{\theta'}(s)^\top\omega(\theta')\right)\nabla^2 V_{\theta'}(s)\omega(\theta')\|\nn\\
    &\leq   \left(\left( D_vR_{\omega}+C_{\phi}L_{\omega}+(1+\gamma)C_{\phi}\rho(s,a) \right)D_vR_{\omega}\right) \|\theta-\theta'\|\nn\\
    &\quad+\left( R_{\omega}L_V+D_vL_{\omega}\right)(\rho(s,a)(r_{\max}+(1+\gamma)C_v)+C_{\phi}R_{\om})\|\theta-\theta'\|.
\end{align}
From the results in \eqref{eq:deltaphilip}, \eqref{eq:phiphiomegalip} and \eqref{eq:hlip}, it follows that  
\begin{align}
    &\|\nabla J(\theta)-\nabla J(\theta')\|\nn\\
    &\leq 2\left\|\mathbb{E}_{\mu^{\pi_b}}\left[\rho(S,A)\delta_{S,A,S'}(\theta)\phi_{\theta}(S)-\rho(S,A)\delta_{S,A,S'}(\theta')\phi_{\theta'}(S)\right]\right\|\nn\\
    &\quad+2\gamma\left\|\mathbb{E}_{\mu^{\pi_b}}\left[\rho(S,A) \phi_{\theta}(S') \phi_{\theta}(S)^\top \om(\theta)-\rho(S,A) \phi_{\theta'}(S') \phi_{\theta'}(S)^\top \om(\theta')\right] \right\|\nn\\
    &\quad+2\left\| \mathbb{E}_{\mu^{\pi_b}}\left[h_{S,A,S'}(\theta,\om(\theta))-h_{S,A,S'}(\theta',\om(\theta'))\right] \right\|\nn\\
    &{\leq} 2 \left((1+\gamma)C_{\phi}^2+(r_{\max}+(1+\gamma)C_v)D_v\right)\|\theta-\theta'\|\nn\\
    &\quad+2\gamma\left(C_{\phi}^2L_{\omega} +2D_v\frac{C^2_{\phi}}{\lambda_v}(r_{\max}+(1+\gamma) C_v)\right)\|\theta-\theta'\|\nn\\
    &\quad+2\mathbb{E}_{\mu^{\pi_b}}[\left(\left( D_vR_{\omega}+C_{\phi}L_{\omega}+(1+\gamma)C_{\phi}\rho(S,A) \right)D_vR_{\omega}\right)] \|\theta-\theta'\|\nn\\
    &\quad+2\mathbb{E}_{\mu^{\pi_b}}[\left( R_{\omega}L_V+D_vL_{\omega}\right)(\rho(S,A)(r_{\max}+(1+\gamma)C_v)+C_{\phi}R_{\om})]\|\theta-\theta'\|\nn\\
    &\overset{(a)}{\leq}2 \left((1+\gamma)C_{\phi}^2+(r_{\max}+(1+\gamma)C_v)D_v\right)\|\theta-\theta'\|\nn\\
    &\quad+2\gamma\left(C_{\phi}^2L_{\omega} +2D_v\frac{C^2_{\phi}}{\lambda_v}(r_{\max}+(1+\gamma) C_v)\right)\|\theta-\theta'\|\nn\\
    &\quad+2\big(\left( D_vR_{\omega}+C_{\phi}L_{\omega}+(1+\gamma)C_{\phi}\right)D_vR_{\omega}\nn\\
    &\quad+\left( R_{\omega}L_V+D_vL_{\omega}\right)((r_{\max}+(1+\gamma)C_v)+C_{\phi}R_{\om})\big) \|\theta-\theta'\|\nn\\
    &\triangleq L_J\|\theta-\theta'\|,
\end{align}
where $(a)$ is due to the fact that $\mathbb{E}_{\mu^{\pi_b}}[\rho(S,A)]=1$, and 
\begin{align}
    L_J=&2 \left((1+\gamma)C_{\phi}^2+(r_{\max}+(1+\gamma)C_v)D_v\right) +2\gamma\left(C_{\phi}^2L_{\omega} +2D_v\frac{C^2_{\phi}}{\lambda_v}(r_{\max}+(1+\gamma) C_v)\right)\nn\\
    &\quad+2\big(\left( D_vR_{\omega}+C_{\phi}L_{\omega}+(1+\gamma)C_{\phi}\right)D_vR_{\omega}\nn\\
    &\quad+\left( R_{\omega}L_V+D_vL_{\omega}\right)((r_{\max}+(1+\gamma)C_v)+C_{\phi}R_{\om})\big).
\end{align}
This completes the proof. 
\end{proof}

\section{Non-asymptotic Analysis under the i.i.d.\ Setting}\label{section:iid}
First we introduce the off-policy TDC learning with non-linear function approximation algorithm under the i.i.d.\ setting in Algorithm \ref{alg:iid}. We then bound the tracking error in \Cref{section:iidtracking}, and prove the Theorem \ref{thm:main} under the i.i.d.\ setting in \Cref{section:iidmain}. 
\begin{algorithm}
\caption{Non-Linear Off-Policy TDC under the i.i.d.\ Setting}
\label{alg:iid}
\textbf{Input}: $T$, $\alpha$, $\beta$, $\pi$, $\pi_b$, $\left\{V_{\theta}|\theta\in\mathbb{R}^N\right\}$\\
\textbf{Initialization}: $\theta_0$,$\omega_0$ 
\begin{algorithmic}[1] 
\STATE {Choose $W\sim \text{Uniform}(0,1,...,T-1)$}
\FOR {$t=0,1,...,W-1$}
\STATE {Sample $O_t=(s_t,a_t,r_t,s'_t)$ according to $\mu^{\pi_b}$}
\STATE {$\rho_t=\frac{\pi(a_t|s_t)}{\pi_b(a_t|s_t)}$}
\STATE {$\delta_t(\theta_t)=r(s_t,a_t,s'_t)+\gamma V_{\theta_t}(s_t')-V_{\theta_t}(s_t)$}
\STATE {$h_t(\theta_t,\omega_t) =\left(\rho_t\delta_t(\theta_t)-\phi_{\theta_t}(s_t)^\top\omega_t\right)\nabla^2 V_{\theta_t}(s_t)\omega_t$}
\STATE {$\omega_{t+1}=\mathbf \Pi_{R_{\om}}\left(\omega_t +\beta\left(-\phi_{\theta_t}(s_t)\phi_{\theta_t}(s_t)^\top\omega_t+\rho_t\delta_t(\theta_t)\phi_{\theta_t}(s_t)\right)\right)$}
\STATE {$\theta_{t+1}=\theta_t +\alpha\big(\rho_t\delta_t(\theta_t)\phi_{\theta_t}(s_t)-\gamma\rho_t\phi_{\theta_t}(s'_t)\phi_{\theta_t}(s_t)^\top\omega_t -h_t(\theta_t,\omega_t) \big)$}
\ENDFOR
\end{algorithmic}
\textbf{Output}: $\theta_W$
\end{algorithm}

We note that under the i.i.d.\ setting, it is assumed that at each time step $t$, a sample $O_t=(s_t,a_t,r_t,s'_t)$ is available, where $s_t\sim \mu^{\pi_b}(\cdot)$, $a_t\sim \pi_b(\cdot|s_t)$ and $s'_t\sim \mathsf P(\cdot|s_t,a_t)$.

\subsection{Tracking Error Analysis under the i.i.d.\ Setting}\label{section:iidtracking}
Denote the tracking error by $z_t=\omega_t-\omega(\theta_t)$. Then by the update of $\omega_t$, the update of $z_t$ can be written as
\begin{align}\label{eq:zupdate}
    z_{t+1}
    &=\omega_{t+1}-\omega(\theta_{t+1})\nn\\
    &=\omega_t+\beta\left(-\phi_{\theta_t}(s_t)\phi_{\theta_t}(s_t)^\top\omega_t+\rho_t\delta_t(\theta_t)\phi_{\theta_t}(s_t)\right)-\omega(\theta_{t+1})\nn\\
    &=z_t+\omega(\theta_t)-\omega(\theta_{t+1})+\beta\left(-\phi_{\theta_t}(s_t)\phi_{\theta_t}(s_t)^\top(z_t+\omega(\theta_t))+\rho_t\delta_t(\theta_t)\phi_{\theta_t}(s_t)\right)\nn\\
    &= z_t+\omega(\theta_t)-\omega(\theta_{t+1})+\beta \left(-A_{\theta_t}(s_t)z_t-A_{\theta_t}(s_t)\omega(\theta_t)+\rho_t\delta_t(\theta_t)\phi_{\theta_t}(s_t)\right),
\end{align}
where $A_{\theta_t}(s_t)=\phi_{\theta_t}(s_t)\phi_{\theta_t}(s_t)^\top$. It then follows that 
\begin{align}\label{eq:znormupdate}
    &\|z_{t+1}\|^2\nn\\
    &=\left\| z_t+\omega(\theta_t)-\omega(\theta_{t+1})+\beta \left(-A_{\theta_t}(s_t)z_t-A_{\theta_t}(s_t)\omega(\theta_t)+\rho_t\delta_t(\theta_t)\phi_{\theta_t}(s_t)\right)\right\|^2\nn\\
    &= \|z_t\|^2+\|\omega(\theta_t)-\omega(\theta_{t+1})+\beta \left(-A_{\theta_t}(s_t)z_t-A_{\theta_t}(s_t)\omega(\theta_t)+\rho_t\delta_t(\theta_t)\phi_{\theta_t}(s_t)\right) \|^2\nn\\
    &\quad+2\langle z_t, \omega(\theta_t)-\omega(\theta_{t+1})\rangle -2\beta\langle z_t,A_{\theta_t}(s_t)z_t \rangle +2\beta \langle z_t, -A_{\theta_t}(s_t)\omega(\theta_t)+\rho_t\delta_t(\theta_t)\phi_{\theta_t}(s_t)\rangle\nn\\
    &\leq \|z_t\|^2+\underbrace{2\beta^2\|\left(-A_{\theta_t}(s_t)z_t-A_{\theta_t}(s_t)\omega(\theta_t)+\rho_t\delta_t(\theta_t)\phi_{\theta_t}(s_t)\right) \|^2}_{(a)}\nn\\
    &\quad+\underbrace{2\|\omega(\theta_t)-\omega(\theta_{t+1})\|^2}_{(b)}+\underbrace{2\langle z_t, \omega(\theta_t)-\omega(\theta_{t+1})\rangle}_{(c)} \underbrace{-2\beta\langle z_t,A_{\theta_t}(s_t)z_t \rangle}_{(d)}\nn\\
    &\quad+{2\beta \langle z_t, -A_{\theta_t}(s_t)\omega(\theta_t)+\rho_t\delta_t(\theta_t)\phi_{\theta_t}(s_t)\rangle}.
\end{align}
We then provide the bounds of the terms in \eqref{eq:znormupdate} one by one. Their proofs can be found in \Cref{section:a,section:b,section:c,section:d}.

\textbf{Term $(a)$ can be bounded as follows:}
\begin{align}\label{eq:bound4}
    2\beta^2\|\left(-A_{\theta_t}(s_t)z_t-A_{\theta_t}(s_t)\omega(\theta_t)+\rho_t\delta_t(\theta_t)\phi_{\theta_t}(s_t)\right) \|^2\leq 4\beta^2C_{\phi}^2\|z_t\|^2+4\beta^2C_{g1},
\end{align}
where $C_{g1}=\left( \frac{C_{\phi}^3}{\lambda_v}(r_{\max}+(1+\gamma) C_v)+\rho_{\max}C_{\phi}(r_{\max}+(1+\gamma)C_v)\right)^2$.

\textbf{Term $(b)$ can be bounded as follows:}
\begin{align}\label{eq:bound3}
    2\|\omega(\theta_t)-\omega(\theta_{t+1}) \|^2\leq 4\alpha^2 L_{\omega}^2 L_g^2\|z_t \|^2+4\alpha^2 C^2_gL^2_{\omega},
\end{align}
where $C_g=\rho_{\max}C_{\phi}(r_{\max}+(1+\gamma)C_v)+\gamma  \rho_{\max}R_{\om}C_{\phi}^2+D_vR_{\om}(R_{\om}C_{\phi}+\rho_{\max}(r_{\max}+C_v+\gamma C_v))$.

\textbf{Term $(c)$ can be bounded as follows:}
\begin{align}\label{eq:bound2}
    &2\langle z_t,  \omega(\theta_t)-\omega(\theta_{t+1})\rangle\nn\\
    &\leq  2(\alpha L_{\omega}L_g+\frac{1}{2}\alpha L_{\omega}+4\alpha^2C_gL_gD_{\omega})\|z_t\|^2+\frac{\alpha L_{\omega}}{4}\| \nabla J(\theta_t)\|^2+\frac{\alpha^2C_g^3D_{\omega}}{L_g}+2\alpha \eta_G(\theta_t,z_t,O_t),
\end{align}
where $\eta_G(\theta_t,z_t,O_t)=-\left\langle z_t,\nabla \omega({\theta_t}) \left(G_{t+1}(\theta_t,\omega(\theta_t))+\frac{\nabla J(\theta_t)}{2}\right)\right\rangle$.

\textbf{Term $(d)$ can be bounded as follows:}
\begin{align}\label{eq:bound1}
    -2\beta\langle z_t,A_{\theta_t}(s_t)z_t \rangle \leq -2\beta\lambda_v\|z_t\|^2+2\beta\langle z_t,(A_{\theta_t}-A_{\theta_t}(s_t))z_t \rangle,
\end{align}
where $A_{\theta}=\mE_{\mu^{\pi_b}}\left[ \phi_{\theta}(S)\phi_{\theta}(S)^\top\right]$ is the expectation of $A_{\theta}(S)$.

By plugging all the bounds from \eqref{eq:bound4}, \eqref{eq:bound3}, \eqref{eq:bound2} and \eqref{eq:bound1} in \eqref{eq:znormupdate}, it follows that  
\begin{align}\label{eq:tracking1}
    &\|z_{t+1}\|^2\nn\\
    &\leq (1+4\beta^2C_{\phi}^2+4\alpha^2L_{\omega}^2L_g^2+2\alpha L_wL_g+\alpha L_w+8\alpha^2C_gL_gD_{\omega}-2\beta\lambda_v)\|z_t\|^2\nn\\
    &\quad+\frac{1}{4}{\alpha L_{\omega}}\|\nabla J(\theta_t)\|^2+4\beta^2C_{g1}+4\alpha^2C_g^2L_{\omega}^2+\frac{\alpha^2C_g^3D_{\omega}}{L_g}+2\alpha \eta_G(\theta_t,z_t,O_t)\nn\\
    &\quad+2\beta\langle z_t,(A_{\theta_t}-A_{\theta_t}(s_t))z_t \rangle+ 2\beta \langle z_t, -A_{\theta_t}(s_t)\omega(\theta_t)+\rho_t\delta_t(\theta_t)\phi_{\theta_t}(s_t)\rangle\nn\\
    &\triangleq (1-q)\|z_t\|^2 +\frac{\alpha L_{\omega}}{4}\|\nabla J(\theta_t)\|^2+4\beta^2C_{g1}+4\alpha^2C_g^2L_{\omega}^2+\frac{\alpha^2C_g^3D_{\omega}}{L_g}+2\alpha \eta_G(\theta_t,z_t,O_t)\nn\\
    &\quad+2\beta\langle z_t,(A_{\theta_t}-A_{\theta_t}(s_t))z_t \rangle+ 2\beta \langle z_t, -A_{\theta_t}(s_t)\omega(\theta_t)+\rho_t\delta_t(\theta_t)\phi_{\theta_t}(s_t)\rangle,
\end{align}
where $q=2\beta\lambda_v-4\beta^2C_{\phi}^2-4\alpha^2L_{\omega}^2L_g^2-2\alpha L_wL_g-\alpha L_w-8\alpha^2C_gL_gD_{\omega}$. Note that $q=\mathcal{O}(\beta-\beta^2-\alpha-\alpha^2)=\mathcal{O}(\beta)$, hence we can choose $\alpha$ and $\beta$ such that $q>0$.

Note that under the i.i.d.\ setting, 
\begin{align}
    \mathbb{E}\left[\eta_G(\theta_t,z_t,O_t)\right]&=\mathbb{E}\left[\mathbb{E}\left[\eta_G(\theta_t,z_t,O_t)|\mathcal{F}_t\right]\right]\nn\\
    &=\mathbb{E}\left[-\left\langle z_t,\nabla \omega({\theta_t}) \mathbb{E}\left[\left(G_{t+1}(\theta_t,\omega(\theta_t))+\frac{\nabla J(\theta_t)}{2}\right)\Bigg|\mathcal{F}_t\right]\right\rangle\right]\nn\\
    &=0,
\end{align}
which is due to the fact that $\mathbb{E}_{\mu^{\pi_b}}[G_{t+1}(\theta ,\omega(\theta ))]=-\frac{\nabla J(\theta)}{2}$ when $\theta$ is fixed, and $\mathcal{F}_t$ is the $\sigma$-field generated by the randomness until $\theta_t$ and $\omega_t$. 
Similarly, it can also be shown that 
\begin{align}
    \mathbb{E}[\langle z_t,(A_{\theta_t}-A_{\theta_t}(s_t))z_t \rangle]&=0\\
    \mathbb{E}[\langle z_t, -A_{\theta_t}(s_t)\omega(\theta_t)+\rho_t\delta_t(\theta_t)\phi_{\theta_t}(s_t)\rangle]&=0.
\end{align}
Hence the tracking error in \eqref{eq:tracking1} can be further bounded as 
\begin{align}\label{eq:iidrecur}
    &\mathbb{E}[\|z_{t+1}\|^2]\leq (1-q)\mathbb{E}\left[\|z_t\|^2\right] +\frac{\alpha L_{\omega}}{4}\mathbb{E}\left[\|\nabla J(\theta_t)\|^2\right]+4\beta^2C_{g1}+4\alpha^2C_g^2L_{\omega}^2+\frac{\alpha^2C_g^3D_{\omega}}{L_g}.
\end{align}
Recursively applying the inequality in \eqref{eq:iidrecur}, it follows that 
\begin{align}
   \mathbb{E}\left[\|z_t\|^2 \right]&\leq (1-q)^t\|z_0\|^2+\frac{\alpha L_{\omega}}{4}\sum^t_{i=0}(1-q)^{t-i}\mathbb{E}\left[\|\nabla J(\theta_i)\|^2\right]\nn\\
    &\quad+\frac{1}{q}\left( 4\beta^2C_{g1}+4\alpha^2C_g^2L_{\omega}^2+\frac{\alpha^2C_g^3D_{\omega}}{L_g}\right),
\end{align}
and summing up w.r.t. $t$ from $0$ to $T-1$, it follows that 
\begin{align}\label{eq:trackingerror}
    \frac{\sum^{T-1}_{t=0}\mathbb{E}\left[\|z_t\|^2 \right]}{T}
    &\leq \frac{\sum^{T-1}_{t=0}(1-q)^t}{T}\|z_0\|^2+\frac{\alpha L_{\omega}}{4T}\sum^{T-1}_{t=0}\sum^t_{i=0}(1-q)^{t-i}\mathbb{E}\left[\|\nabla J(\theta_i)\|^2\right]\nn\\
    &\quad+\frac{1}{q}\left( 4\beta^2C_{g1}+4\alpha^2C_g^2L_{\omega}^2+\frac{\alpha^2C_g^3D_{\omega}}{L_g}\right)\nn\\
    &\overset{(a)}{\leq}\frac{\|z_0\|^2}{Tq}+ \frac{\alpha L_{\omega}}{4q}\frac{\sum^{T-1}_{t=0}\mathbb{E}\left[\|\nabla J(\theta_t)\|^2 \right]}{T}\nn\\
    &\quad+\frac{1}{q}\left( 4\beta^2C_{g1}+4\alpha^2C_g^2L_{\omega}^2+\frac{\alpha^2C_g^3D_{\omega}}{L_g}\right)\nn\\
    &=\mathcal{O}\left( \frac{1}{T\beta}+\frac{\alpha}{\beta}\frac{\sum^{T-1}_{t=0}\mathbb{E}\left[\|\nabla J(\theta_t)\|^2 \right]}{T}+\beta\right),
\end{align}
where $(a)$ is due to the double-sum trick, i.e., for any $x_i\geq 0$, $\sum^{T-1}_{t=0}\sum^t_{i=0}(1-q)^{t-i}x_i\leq \sum^{T-1}_{t=0}(1-q)^t\sum^{T-1}_{t=0}x_t\leq \frac{1}{q}\sum^{T-1}_{t=0}x_t$, and the last step is because $q=\mathcal{O}(\beta)$. 

\subsubsection{Bound on Term $(a)$}\label{section:a}
In this section we provide the detailed proof of the bound on term $(a)$ in \eqref{eq:bound4}.

It can be shown that 
\begin{align}
    &\|\left(-A_{\theta_t}(s_t)z_t-A_{\theta_t}(s_t)\omega(\theta_t)+\rho_t\delta_t(\theta_t)\phi_{\theta_t}(s_t)\right) \|^2\nn\\
    &\leq 2 \|-A_{\theta_t}(s_t)z_t \|^2+ 2\| -A_{\theta_t}(s_t)\omega(\theta_t)+\rho_t\delta_t(\theta_t)\phi_{\theta_t}(s_t)\|^2\nn\\
    &\overset{(a)}{\leq}2C_{\phi}^2\|z_t\|^2+2\left( \frac{C_{\phi}^3}{\lambda_v}(r_{\max}+(1+\gamma) C_v)+\rho_{\max}C_{\phi}(r_{\max}+(1+\gamma)C_v)\right)^2,
\end{align}
where $(a)$ is from the fact that $\| A_{\theta}(s)\|=\|\phi_{\theta}(s)\phi_{\theta}(s)^\top \|\leq C_{\phi}^2$ and the bounds in \eqref{eq:deltabounds} and \eqref{eq:omegabound}. 

\subsubsection{Bound on Term $(b)$}\label{section:b}
In this section we provide the detailed proof of the bound on term $(b)$ in \eqref{eq:bound3}.

We first show that $G_{t+1}(\theta,\omega)$ is Lipschitz in $\om$ for any fixed $\theta$. Specifically, for any $\theta,\om_1,\om_2\in\mathbb{R}^N$, it follows that 
\begin{align}\label{eq:Glip}
    &\|G_{t+1}(\theta,\omega_1)-G_{t+1}(\theta,\omega_2)\|\nn\\
    &=\|\rho_t\delta_t(\theta)\phi_{\theta}(s_t)-\gamma\rho_t\phi_{\theta}(s'_t)\phi_{\theta}(s_t)^\top\omega_1 -h_t(\theta,\omega_1)-\rho_t\delta_t(\theta)\phi_{\theta}(s_t)+\gamma\rho_t\phi_{\theta}(s'_t)\phi_{\theta}(s_t)^\top\omega_2\nn\\
    &\quad+h_t(\theta,\omega_2)\|\nn\\
    &\leq \|h_t(\theta,\omega_1)-h_t(\theta,\omega_2) \|+ \|\gamma\rho_t\phi_{\theta}(s'_t)\phi_{\theta}(s_t)^\top\omega_1-\gamma\rho_t\phi_{\theta}(s'_t)\phi_{\theta}(s_t)^\top\omega_2 \|\nn\\
    &\overset{(a)}{\leq} \left(C_{\phi}D_vR_{\om}+D_v(C_{\phi}R_{\om}+\rho_{\max}(r_{\max}+C_v+\gamma C_v))+\gamma\rho_{\max}C_{\phi}^2\right)\|\om_1-\om_2\|\nn\\
    &\triangleq L_g\|\om_1-\om_2\|,
\end{align}
where $L_g= D_v(2C_{\phi}R_{\om}+\rho_{\max}(r_{\max}+C_v+\gamma C_v))+\gamma\rho_{\max}C_{\phi}^2$, and $(a)$ is from the Lipschitz continuous of $h_t(\theta,\cdot)$, i.e., 
\begin{align}
    \|h_t(\theta,\om_1)-h_t(\theta,\om_2) \|\leq \rho_{\max}(r_{\max}+(1+\gamma)C_v)D_v\|\om_1-\om_2\|+2C_{\phi}D_vR_{\om}\|\om_1-\om_2\|.
\end{align}
We note that to show \eqref{eq:Glip}, we use the bound on $\om_t$, which is guaranteed by the projection step. And this is the only step in our proof where the projection is used.

Then it follows that 
\begin{align}\label{eq:thetaupdatebound}
    \|\theta_{t+1}-\theta_t\|&=\alpha\| G_{t+1}(\theta_t,\omega_t)\|\nn\\
    &\leq\alpha \|G_{t+1}(\theta_t,\omega_t)-G_{t+1}(\theta_t,\omega(\theta_t))+G_{t+1}(\theta_t,\omega(\theta_t))\|\nn\\
    &\leq \alpha  L_g\|z_t \|+\alpha \|G_{t+1}(\theta_t,\omega(\theta_t))\|\nn\\
    &\leq \alpha L_g\|z_t \|+\alpha  C_g,
\end{align}
where $C_g= \rho_{\max}C_{\phi}(r_{\max}+(1+\gamma)C_v)+\gamma  \rho_{\max}R_{\om}C_{\phi}^2+D_vR_{\om}(R_{\om}C_{\phi}+\rho_{\max}(r_{\max}+C_v+\gamma C_v))$, and the last step in \eqref{eq:thetaupdatebound} can be shown as follows
\begin{align}
    &\| G_{t+1}(\theta_t,\omega(\theta_t))\|\nn\\
    &=\|\rho_t\delta_t(\theta)\phi_{\theta}(s_t)-\gamma\rho_t\phi_{\theta}(s'_t)\phi_{\theta}(s_t)^\top\omega(\theta)-h_t(\theta,\omega(\theta))\|\nn\\
    &\leq \rho_{\max}C_{\phi}(r_{\max}+(1+\gamma)C_v)+\gamma  \rho_{\max}R_{\om}C_{\phi}^2+D_vR_{\om}(R_{\om}C_{\phi}+\rho_{\max}(r_{\max}+C_v+\gamma C_v)).
\end{align}

Using \eqref{eq:omegalip} and \eqref{eq:thetaupdatebound}, it follows that 
\begin{align}
     \|\omega(\theta_t)-\omega(\theta_{t+1}) \|\leq L_{\omega} \|\theta_{t+1}-\theta_t \|\leq \alpha L_{\omega} L_g\|z_t \|+\alpha C_gL_{\omega},
\end{align}
and 
\begin{align}
    \|\omega(\theta_t)-\omega(\theta_{t+1}) \|^2\leq 2\alpha^2 L_{\omega}^2 L_g^2\|z_t \|^2+2\alpha^2 C^2_gL^2_{\omega}.
\end{align}
This completes the proof for term $(b)$. 

\subsubsection{Bound on Term $(c)$}\label{section:c}
In this section we provide the detailed proof of the bound on term $(c)$ in \eqref{eq:bound2}.

Consider the inner product $\langle z_t, \omega(\theta_t)-\omega(\theta_{t+1})\rangle$. By the Mean-Value Theorem, it follows that 
\begin{align}
    &\langle z_t, \omega(\theta_t)\rangle-\langle z_t, \omega(\theta_{t+1})\rangle =\langle z_t,  \omega(\theta_t)-\omega(\theta_{t+1})\rangle =\langle z_t, \nabla \omega(\hat{\theta}_t)(\theta_t-\theta_{t+1})\rangle,
\end{align}
where $\hat{\theta}_t=c\theta_t+(1-c)\theta_{t+1}$ for some $c\in[0,1]$. 
Thus, it follows that
\begin{align}
    &\langle z_t,  \omega(\theta_t)-\omega(\theta_{t+1})\rangle\nn\\
    &=\langle z_t, \nabla \omega(\hat{\theta}_t)(\theta_t-\theta_{t+1})\rangle\nn\\
    &=-\alpha\langle z_t, \nabla \omega(\hat{\theta}_t)G_{t+1}(\theta_t,\omega_t)\rangle\nn\\
    &=-\alpha\left\langle z_t, \nabla \omega(\hat{\theta}_t)\left(G_{t+1}(\theta_t,\omega_t)-G_{t+1}(\theta_t,\omega(\theta_t))+G_{t+1}(\theta_t,\omega(\theta_t))+\frac{\nabla J(\theta_t)}{2}\right)\right\rangle\nn\\
    &\quad+\alpha\left\langle z_t,\nabla \omega(\hat{\theta}_t)\frac{\nabla J(\theta_t)}{2} \right\rangle \nn\\
    &=-\alpha\left\langle z_t, \nabla \omega(\hat{\theta}_t)\left(G_{t+1}(\theta_t,\omega_t)-G_{t+1}(\theta_t,\omega(\theta_t))\right)\right\rangle+\alpha\left\langle z_t,\nabla \omega(\hat{\theta}_t)\frac{\nabla J(\theta_t)}{2} \right\rangle\nn\\
    &\quad-\alpha \left\langle z_t,\nabla \omega(\hat{\theta}_t) \left(G_{t+1}(\theta_t,\omega(\theta_t))+\frac{\nabla J(\theta_t)}{2}\right)\right\rangle \nn\\
    &\overset{(a)}{\leq}\alpha L_{\omega}L_g \|z_t\|^2+\alpha L_{\omega}\|z_t\|\left\| \frac{\nabla J(\theta_t)}{2}\right\|-\alpha \left\langle z_t,\nabla \omega({\theta_t}) \left(G_{t+1}(\theta_t,\omega(\theta_t))+\frac{\nabla J(\theta_t)}{2}\right)\right\rangle\nn\\
    &\quad+\alpha \left\langle z_t,(\nabla \omega(\theta_t)-\nabla \omega(\hat{\theta}_t)) \left(G_{t+1}(\theta_t,\omega(\theta_t))+\frac{\nabla J(\theta_t)}{2}\right)\right\rangle\nn\\
    &\leq \alpha L_{\omega}L_g \|z_t\|^2+\frac{1}{2}\alpha L_{\omega}\|z_t\|^2+\frac{\alpha L_{\omega}}{8}\| \nabla J(\theta_t)\|^2+\alpha \eta_G(\theta_t,z_t,O_t) \nn\\
    &\quad+\alpha\|z_t\|\|\nabla \omega(\theta_t)-\nabla \omega(\hat{\theta}_t) \| \left\|G_{t+1}(\theta_t,\omega(\theta_t))+\frac{\nabla J(\theta_t)}{2} \right\|\nn\\
    &\overset{(b)}{\leq}\alpha L_{\omega}L_g \|z_t\|^2+\frac{1}{2}\alpha L_{\omega}\|z_t\|^2+\frac{\alpha L_{\omega}}{8}\| \nabla J(\theta_t)\|^2+\alpha \eta_G(\theta_t,z_t,O_t) +2\alpha C_gD_{\omega} \|z_t\| \|\theta_t-\hat{\theta}_t \|\nn\\
    &\overset{(c)}{\leq} \alpha L_{\omega}L_g \|z_t\|^2+\frac{1}{2}\alpha L_{\omega}\|z_t\|^2+\frac{\alpha L_{\omega}}{8}\| \nabla J(\theta_t)\|^2+\alpha \eta_G(\theta_t,z_t,O_t)\nn\\
    &\quad+2\alpha C_gD_{\omega} \|z_t\| \|\theta_t-\theta_{t+1}\|\nn\\
    &\overset{(d)}{\leq}\alpha L_{\omega}L_g \|z_t\|^2+\frac{1}{2}\alpha L_{\omega}\|z_t\|^2+\frac{\alpha L_{\omega}}{8}\| \nabla J(\theta_t)\|^2+\alpha \eta_G(\theta_t,z_t,O_t)\nn\\
    &\quad+2\alpha C_gD_{\omega} \|z_t\|(\alpha L_g \|z_t\|+\alpha C_g)\nn\\
    &\overset{(e)}{\leq}\alpha L_{\omega}L_g \|z_t\|^2+\frac{1}{2}\alpha L_{\omega}\|z_t\|^2+\frac{\alpha L_{\omega}}{8}\| \nabla J(\theta_t)\|^2+\alpha \eta_G(\theta_t,z_t,O_t) \nn\\
    &\quad+2\alpha^2 C_gD_{\omega} \left(2L_g\|z_t\|^2+\frac{C_g^2}{4L_g}\right)\nn\\
    &\leq (\alpha L_{\omega}L_g+\frac{1}{2}\alpha L_{\omega}+4\alpha^2C_gL_gD_{\omega})\|z_t\|^2+\frac{\alpha L_{\omega}}{8}\| \nabla J(\theta_t)\|^2+\frac{\alpha^2C_g^3D_{\omega}}{2L_g}+\alpha \eta_G(\theta_t,z_t,O_t),
\end{align}
where $\eta_G(\theta_t,z_t,O_t)=-\left\langle z_t,\nabla \omega({\theta_t}) \left(G_{t+1}(\theta_t,\omega(\theta_t))+\frac{\nabla J(\theta_t)}{2}\right)\right\rangle$, $(a)$ is from the Lipschitz continuity of $G_{t+1}(\theta,\cdot)$ proved in \eqref{eq:Glip}, $(b)$ is from the Lipschitz continuity of $\nabla \omega(\theta)$, which is shown in \eqref{eq:nablaomegalip}, $(c)$ is from the fact that $\|\theta_t-\hat{\theta}_t\|=(1-c)\|\theta_t-\theta_{t+1} \|\leq \|\theta_t-\theta_{t+1} \|$,  $(d)$ is from the bound of $\|\theta_t-\theta_{t+1} \|$ in \eqref{eq:thetaupdatebound}, and $(e)$ is from the fact that $C_g\|z_t\|\leq L_g\|z_t\|^2+\frac{C_g^2}{4L_g}$.

This completes the proof.

\subsubsection{Bound on Term $(d)$}\label{section:d}
In this section we provide the detailed proof of the bound on term $(d)$ in \eqref{eq:bound1}.

It can be shown that
\begin{align}
    -2\beta\langle z_t,A_{\theta_t}(s_t)z_t \rangle&=-2\beta\langle z_t,A_{\theta_t}z_t \rangle+2\beta\langle z_t,(A_{\theta_t}-A_{\theta_t}(s_t))z_t \rangle\nn\\
    &\leq -2\beta\lambda_v\|z_t\|^2+2\beta\langle z_t,(A_{\theta_t}-A_{\theta_t}(s_t))z_t \rangle,
\end{align}
where the inequality is due to the fact that $\langle z_t,A_{\theta_t}z_t \rangle=z_t^\top A_{\theta_t}z_t\geq \lambda_L(A_{\theta_t})\|z_t\|^2\geq \lambda_v\|z_t\|^2$.

\subsection{Proof under the i.i.d.\ Setting}\label{section:iidmain}
In this section we provide the proof of Theorem \ref{thm:main} under the i.i.d.\ setting.

From Lemma \ref{lemma:Lsmooth}, we know that the objective function $J(\theta)$ is $L_J$-smooth, hence it follows that 
\begin{align}
    J(\theta_{t+1}) &\leq J(\theta_t) +\left\langle \nabla J(\theta_t), \theta_{t+1}-\theta_t\right\rangle  + \frac{L_J}{2} \| \theta_{t+1}-\theta_t\|^2\nn\\
    &=J(\theta_t) +\alpha \left\langle \nabla J(\theta_t),G_{t+1}(\theta_t,\omega_t) \right\rangle  + \frac{L_J}{2} \alpha^2\|G_{t+1}(\theta_t,\omega_t)\|^2\nn\\
    &=J(\theta_t)-\alpha\left\langle \nabla J(\theta_t),-G_{t+1}(\theta_t, \omega_t)-\frac{\nabla J(\theta_t)}{2}+G_{t+1}(\theta_t, \omega(\theta_t))-G_{t+1}(\theta_t, \omega(\theta_t)) \right\rangle \nn\\
    &\quad-\frac{\alpha}{2}\|\nabla J(\theta_t)\|^2+\frac{L_J}{2} \alpha^2\|G_{t+1}(\theta_t,\omega_t)\|^2\nn\\
    &=J(\theta_t)-\alpha\left\langle \nabla J(\theta_t),-G_{t+1}(\theta_t, \omega_t)+G_{t+1}(\theta_t, \omega(\theta_t)) \right\rangle\nn\\
    &\quad+\alpha \left\langle \nabla J(\theta_t), \frac{\nabla J(\theta_t)}{2}+G_{t+1}(\theta_t, \omega(\theta_t)) \right\rangle-\frac{\alpha}{2}\|\nabla J(\theta_t)\|^2+\frac{L_J}{2} \alpha^2\|G_{t+1}(\theta_t,\omega_t)\|^2\nn\\
    &\overset{(a)}{\leq} J(\theta_t) +\alpha L_g\|\nabla J(\theta_t) \|\|\omega(\theta_t)-\omega_t \|+\alpha \left\langle \nabla J(\theta_t), \frac{\nabla J(\theta_t)}{2}+G_{t+1}(\theta_t, \omega(\theta_t)) \right\rangle \nn\\
    &\quad-\frac{\alpha}{2}\|\nabla J(\theta_t)\|^2+\frac{L_J}{2} \alpha^2\|G_{t+1}(\theta_t,\omega_t)\|^2\nn\\
    &\overset{(b)}{\leq}J(\theta_t) +\alpha L_g\|\nabla J(\theta_t) \|\|z_t \|+\alpha \left\langle \nabla J(\theta_t), \frac{\nabla J(\theta_t)}{2}+G_{t+1}(\theta_t, \omega(\theta_t)) \right\rangle \nn\\
    &\quad-\frac{\alpha}{2}\|\nabla J(\theta_t)\|^2+\frac{L_J}{2} \alpha^2\left(2L_g^2\|z_t\|^2+2C_g^2\right),
\end{align}
where $(a)$ is from \eqref{eq:Glip} and $(b)$ is because $\|\theta_{t+1}-\theta_t\|=\alpha\|G_{t+1}(\theta_t,\om_t)\| \leq \alpha L_g\|z_t \|+\alpha  C_g$, whose detailed proof is provided in \eqref{eq:thetaupdatebound}. 
Thus by re-arranging the terms, taking expectation and summing up w.r.t. $t$ from $0$ to $T-1$, it follows that 
\begin{align}
    &\frac{\alpha}{2}\sum^{T-1}_{t=0}\mathbb{E}[\|\nabla J(\theta_t)\|^2] \nn\\
    &\leq -\mathbb{E}[J(\theta_{T})]+J(\theta_0)+\alpha L_g\sqrt{\sum^{T-1}_{t=0}\mathbb{E}[\|\nabla J(\theta_t)\|^2]}\sqrt{\sum^{T-1}_{t=0}\mathbb{E}[\|z_t\|^2]} +\alpha^2L_JL_g^2\sum^{T-1}_{t=0}\mathbb{E}[\|z_t\|^2]\nn\\
    &\quad+\alpha^2C_g^2L_JT,
\end{align}
which is due to the fact that under the i.i.d.\ setting,
\begin{align}
    &\mathbb{E}\left[\left\langle \nabla J(\theta_t), \frac{\nabla J(\theta_t)}{2}+G_{t+1}(\theta_t, \omega(\theta_t)) \right\rangle\right]\nn\\
    &=\mathbb{E}\left[\left\langle \nabla J(\theta_t), \mathbb{E}\left[\frac{\nabla J(\theta_t)}{2}+G_{t+1}(\theta_t, \omega(\theta_t))\Big|\mathcal{F}_t\right] \right\rangle\right]=0,
\end{align}
and the Cauchy's inequality
\begin{align}
    \sum^{T-1}_{t=0}\mathbb{E}[\|\nabla J(\theta_t) \|\|z_t \|]\leq \sqrt{\sum^{T-1}_{t=0}\mathbb{E}[\|\nabla J(\theta_t)\|^2]}\sqrt{\sum^{T-1}_{t=0}\mathbb{E}[\|z_t\|^2]}.
\end{align}
Thus dividing both sides by $\frac{\alpha T}{2}$, it follows that 
\begin{align}
    &\frac{\sum^{T-1}_{t=0}\mathbb{E}[\|\nabla J(\theta_t)\|^2]}{T}\nn\\
    &\leq \frac{2(J(\theta_0)-J^*)}{T\alpha}+2L_g\sqrt{\frac{\sum^{T-1}_{t=0}\mathbb{E}[\|\nabla J(\theta_t)\|^2]}{T}}\sqrt{\frac{\sum^{T-1}_{t=0}\mathbb{E}[\|z_t\|^2]}{T}}\nn\\
    &\quad+2\alpha L_JL_g^2\frac{\sum^{T-1}_{t=0}\mathbb{E}[\|z_t\|^2]}{T} +2\alpha C_g^2L_J,
\end{align}
where $J^*\triangleq\min_{\theta} J(\theta)$.

Recall the tracking error in \eqref{eq:trackingerror}:
    \begin{align}
    &\frac{\sum^{T-1}_{t=0}\mathbb{E}\left[\|z_t\|^2 \right]}{T}\nn\\
    &{\leq}\frac{\|z_0\|^2}{Tq}+ \frac{\alpha L_{\omega}}{4q}\frac{\sum^{T-1}_{t=0}\mathbb{E}\left[\|\nabla J(\theta_t)\|^2 \right]}{T}+\frac{1}{q}\left( 4\beta^2C_{g1}+4\alpha^2C_g^2L_{\omega}^2+\frac{\alpha^2C_g^3D_{\omega}}{L_g}\right).
\end{align}

We then plug in the tracking error and obtain that 
\begin{align}
    &\frac{\sum^{T-1}_{t=0}\mathbb{E}[\|\nabla J(\theta_t)\|^2]}{T}\nn\\
    &\leq \frac{2(J(\theta_0)-J^*)}{T\alpha}+2\alpha C_g^2L_J+2L_g\sqrt{\frac{\sum^{T-1}_{t=0}\mathbb{E}[\|\nabla J(\theta_t)\|^2]}{T}}\nn\\
    &\quad\times\sqrt{\frac{\|z_0\|^2}{Tq}+{\alpha L_{\omega}}\frac{1}{4q}\frac{\sum^{T-1}_{t=0}\mathbb{E}\left[\|\nabla J(\theta_t)\|^2 \right]}{T}+\frac{1}{q}\left( 4\beta^2C_{g1}+4\alpha^2C_g^2L_{\omega}^2+\frac{\alpha^2C_g^3D_{\omega}}{L_g}\right)}\nn\\
    &\quad+2\alpha L_JL_g^2\Bigg(\frac{\|z_0\|^2}{Tq}+{\alpha L_{\omega}}\frac{1}{4q}\frac{\sum^{T-1}_{t=0}\mathbb{E}\left[\|\nabla J(\theta_t)\|^2 \right]}{T}\nn\\
    &\quad+\frac{1}{q}\left( 4\beta^2C_{g1}+4\alpha^2C_g^2L_{\omega}^2+\frac{\alpha^2C_g^3D_{\omega}}{L_g}\right)\Bigg)\nn\\
    &\leq \frac{2(J(\theta_0)-J^*)}{T\alpha}+2\alpha C_g^2L_J+L_g\sqrt{\frac{\alpha L_{\omega}}{q}} \frac{\sum^{T-1}_{t=0}\mathbb{E}\left[\|\nabla J(\theta_t)\|^2 \right]}{T}\nn\\
    &\quad+2L_g\sqrt{\frac{\sum^{T-1}_{t=0}\mathbb{E}[\|\nabla J(\theta_t)\|^2]}{T}} \sqrt{\frac{\|z_0\|^2}{Tq}+\frac{1}{q}\left( 4\beta^2C_{g1}+4\alpha^2C_g^2L_{\omega}^2+\frac{\alpha^2C_g^3D_{\omega}}{L_g}\right)}\nn\\
    &\quad+2\alpha L_JL_g^2\Bigg(\frac{\|z_0\|^2}{Tq}+{\alpha L_{\omega}}\frac{1}{4q}\frac{\sum^{T-1}_{t=0}\mathbb{E}\left[\|\nabla J(\theta_t)\|^2 \right]}{T}\nn\\
    &\quad+\frac{1}{q}\left( 4\beta^2C_{g1}+4\alpha^2C_g^2L_{\omega}^2+\frac{\alpha^2C_g^3D_{\omega}}{L_g}\right)\Bigg),
\end{align}
where the last step is from the fact that $\sqrt{x+y}\leq \sqrt{x}+\sqrt{y}$ for any $x, y\geq 0$. Re-arranging the terms, it follows that 
\begin{align}\label{eq:main1}
    &\left(1-L_g\sqrt{\frac{\alpha L_{\omega}}{q}}-\frac{\alpha^2L_JL_g^2L_{\omega}}{2q}\right)\frac{\sum^{T-1}_{t=0}\mathbb{E}[\|\nabla J(\theta_t)\|^2]}{T}\nn\\
    &\leq \frac{2(J(\theta_0)-J^*)}{T\alpha}+2\alpha C_g^2L_J+2\alpha L_JL_g^2\left(\frac{\|z_0\|^2}{Tq} +\frac{1}{q}\left( 4\beta^2C_{g1}+4\alpha^2C_g^2L_{\omega}^2+\frac{\alpha^2C_g^3D_{\omega}}{L_g}\right)\right) \nn\\
    &\quad+2L_g\sqrt{\frac{\sum^{T-1}_{t=0}\mathbb{E}[\|\nabla J(\theta_t)\|^2]}{T}} \sqrt{\frac{\|z_0\|^2}{Tq}+\frac{1}{q}\left( 4\beta^2C_{g1}+4\alpha^2C_g^2L_{\omega}^2+\frac{\alpha^2C_g^3D_{\omega}}{L_g}\right)}.
\end{align}
Note that $\left(L_g\sqrt{\frac{\alpha L_{\omega}}{q}}+\frac{ \alpha^2L_JL_g^2L_{\omega}}{2q}\right)=\mathcal{O}\left(\sqrt{\frac{\alpha}{\beta}}+\frac{\alpha^2}{\beta}\right)$, hence we can choose $\alpha$ and $\beta$ such that   $\left(1- L_g\sqrt{\frac{\alpha L_{\omega}}{q}}-\frac{ \alpha^2L_JL_g^2L_{\omega}}{2q}\right)\geq \frac{1}{2}$. Thus \eqref{eq:main1} implies that 
\begin{align}\label{eq:main2}
    &\frac{\sum^{T-1}_{t=0}\mathbb{E}[\|\nabla J(\theta_t)\|^2]}{T}\nn\\
    &\leq \frac{4(J(\theta_0)-J^*)}{T\alpha}+4\alpha C_g^2L_J+4\alpha L_JL_g^2\left(\frac{\|z_0\|^2}{Tq} +\frac{1}{q}\left( 4\beta^2C_{g1}+4\alpha^2C_g^2L_{\omega}^2+\frac{\alpha^2C_g^3D_{\omega}}{L_g}\right)\right) \nn\\
    &\quad+4L_g\sqrt{\frac{\sum^{T-1}_{t=0}\mathbb{E}[\|\nabla J(\theta_t)\|^2]}{T}} \sqrt{\frac{\|z_0\|^2}{Tq}+\frac{1}{q}\left( 4\beta^2C_{g1}+4\alpha^2C_g^2L_{\omega}^2+\frac{\alpha^2C_g^3D_{\omega}}{L_g}\right)}.
\end{align}
Denote $U=\frac{4(J(\theta_0)-J^*)}{T\alpha}+4\alpha C_g^2L_J+4\alpha L_JL_g^2\left(\frac{\|z_0\|^2}{Tq} +\frac{1}{q}\left( 4\beta^2C_{g1}+4\alpha^2C_g^2L_{\omega}^2+\frac{\alpha^2C_g^3D_{\omega}}{L_g}\right)\right) $, and $V=4L_g  \sqrt{\frac{\|z_0\|^2}{Tq}+\frac{1}{q}\left( 4\beta^2C_{g1}+4\alpha^2C_g^2L_{\omega}^2+\frac{\alpha^2C_g^3D_{\omega}}{L_g}\right)}$. Then it follows that 
\begin{align}
    \frac{\sum^{T-1}_{t=0}\mathbb{E}[\|\nabla J(\theta_t)\|^2]}{T}\leq V\sqrt{\frac{\sum^{T-1}_{t=0}\mathbb{E}[\|\nabla J(\theta_t)\|^2]}{T}}+U,
\end{align}
which further implies that 
\begin{align}\label{eq:iidresult}
    &\frac{\sum^{T-1}_{t=0}\mathbb{E}[\|\nabla J(\theta_t)\|^2]}{T}\nn\\
    &\leq V^2+2U\nn\\
    &= 16L_g^2\left(\frac{\|z_0\|^2}{Tq}+\frac{1}{q}\left( 4\beta^2C_{g1}+4\alpha^2C_g^2L_{\omega}^2+\frac{\alpha^2C_g^3D_{\omega}}{L_g}\right)\right)+\frac{8(J(\theta_0)-J^*)}{T\alpha}\nn\\
    &\quad+8\alpha C_g^2L_J+8\alpha L_JL_g^2\left(\frac{\|z_0\|^2}{Tq} +\frac{1}{q}\left( 4\beta^2C_{g1}+4\alpha^2C_g^2L_{\omega}^2+\frac{\alpha^2C_g^3D_{\omega}}{L_g}\right)\right)\nn\\
    &= (16L_g^2+8\alpha L_JL_g^2)\left(\frac{\|z_0\|^2}{Tq} +\frac{1}{q}\left( 4\beta^2C_{g1}+4\alpha^2C_g^2L_{\omega}^2+\frac{\alpha^2C_g^3D_{\omega}}{L_g}\right)\right)\nn\\
    &\quad+\frac{8(J(\theta_0)-J^*)}{T\alpha}+8\alpha C_g^2L_J\nn\\
    &=\mathcal{O}\left( \frac{1}{T\beta}+\beta+\frac{1}{T\alpha}\right)\nn\\
    &=\mathcal{O}\left( \frac{1}{T^{1-a}}+\frac{1}{T^b}+\frac{1}{T^{1-b}}\right). 
\end{align}
This completes the proof. 

\subsection{Choice of Step-sizes}\label{sec:step1}
As the proof is complicated and we have made several assumptions on the step-sizes, in this section we summarize all the assumptions we made on the step-sizes. This would help the readers to have a more clear understanding of the choice of $\alpha$ and $\beta$. 

In the proof under the i.i.d.\ setting, we made two assumptions on step-sizes. In \eqref{eq:tracking1}, we assume 
\begin{align}\label{assonstep1}
    q=2\beta\lambda_v-4\beta^2C_{\phi}^2-4\alpha^2L_{\omega}^2L_g^2-2\alpha L_wL_g-\alpha L_w-8\alpha^2C_gL_gD_{\omega}>0;
\end{align}
And in \eqref{eq:main1}, we moreover assume 
\begin{align}\label{assonstep2}
    \left(1- L_g\sqrt{\frac{\alpha L_{\omega}}{q}}-\frac{ \alpha^2L_JL_g^2L_{\omega}}{2q}\right)\geq \frac{1}{2}.
\end{align}
Note that the first one can be satisfied if $\beta\leq\min\left\{1, \frac{\lambda_v}{4C_{\phi}^2}\right\}$ and $\frac{\alpha}{\beta}\leq \frac{\lambda_v}{4L_{\omega}^2L_g^2+2 L_wL_g+ L_w+8C_gL_gD_{\omega}}$. As for assumption \eqref{assonstep2}, we only need to find $\alpha$ and $\beta$ such that 
\begin{align}
    L_g\sqrt{\frac{\alpha L_{\omega}}{q}}\leq \frac{1}{4},\nn\\
    \frac{ \alpha^2L_JL_g^2L_{\omega}}{2q}\leq \frac{1}{4}.
\end{align}
Note that these two conditions are satisfied if condition \eqref{assonstep1} is satisfied. 

Hence to meet all the requirements on the step-sizes, we can set $\beta\leq\min\left\{1, \frac{\lambda_v}{4C_{\phi}^2}\right\}$ and $\frac{\alpha}{\beta}\leq \min\left\{1, \frac{\lambda_v}{4L_{\omega}^2L_g^2+2 L_wL_g+ L_w+8C_gL_gD_{\omega}}\right\}$.

\section{Non-asymptotic Analysis under the Markovian Setting}\label{section:maralg}
In this section we provide the proof of Theorem \ref{thm:main} under that Markovian setting. In \Cref{section:markoviantracking} we develop the finite-time analysis of the tracking error and in  \Cref{section:markovmain} we prove Theorem \ref{thm:main}.

\subsection{Tracking Error Analysis under the Markovian Setting}\label{section:markoviantracking}
We first define the mixing time $\tb=\inf\left\{ t:m\kappa^t\leq \beta\right\}$ (Assumption \ref{ass:mixing}). It can be shown  that for any bounded function $\|f(O_t)\|\leq C_f$, for any $t\geq \tau_{\beta}$, $\|\mE[f(O_t)]-\mE_{O\sim \mu^{\pi_b}}[f(O)] \|\leq C_f \beta$ and $\tau_{\beta}=\mathcal{O}(-\log\beta)$. We note that $\beta\tau_{\beta}\to 0$ as $\beta\to 0$, and  we assume that $\beta\tau_{\beta}C_{\phi}^2\leq\frac{1}{4}$.

From \eqref{eq:zupdate}, the update of the tracking error $z_t$ can be written as \begin{align}\label{eq:marzupdate}
    z_{t+1}=z_t+\beta(-A_{\theta_t}(s_t)z_t+b_t(\theta_t))+\omega(\theta_t)-\omega(\theta_{t+1}),
\end{align}
where $A_{\theta_t}(s_t)=\phi_{\theta_t}(s_t)\phi_{\theta_t}(s_t)^\top$ and  $b_t(\theta_t)=-A_{\theta_t}(s_t)\omega(\theta_t)+\rho_t\delta_t(\theta_t)\phi_{\theta_t}(s_t)$. Note that for any $\theta\in \mathbb{R}^N$ and any sample $O_t=(s_t,a_t,r_t,s_{t+1})\in \mcs\times\mca\times\mathbb{R}\times\mcs$, $\|b_t(\theta_t)\|\leq C_{\phi}^2R_{\om}+\rho_{\max}C_{\phi}(r_{\max}+C_v+\gamma C_v)\triangleq b_{\max}$. 

Then it can be shown that 
\begin{align}\label{eq:aaa}
    &\mE\left[\|z_{t+1}\|^2-\|z_t\|^2\right]\nn\\
    &=\mE\left[2z_t^\top(z_{t+1}-z_t)+\|z_{t+1}-z_t\|^2\right]\nn\\
    &=\mE\left[2z_t^\top(z_{t+1}-z_t+\beta A_{\theta_t}z_t)\right]+\mE\left[\|z_{t+1}-z_t \|^2\right]+\beta\mE\left[2z_t^\top(-A_{\theta_t})z_t\right]\nn\\
    &\leq \underbrace{\mE\left[\|z_{t+1}-z_t \|^2\right]}_{(a)}+\underbrace{\mE\left[2z_t^\top(z_{t+1}-z_t+\beta A_{\theta_t}z_t)\right]}_{(b)}-2\beta\lambda_v\mE\left[\|z_t\|^2\right],
\end{align}
where the last inequality is due to the fact that $\lambda_L(A_{\theta_t})\geq \lambda_v$. We first provide the bounds on terms $(a)$ and $(b)$ as follows, and their detailed proof can be found in \Cref{section:mara,section:marb}. 

\textbf{Term $(a)$ can be bounded as follows:}

For any $t\geq 0$, we have that
\begin{align}
    \|z_{t+1}-z_t\|^2\leq 2\beta^2 C_{\phi}^4 \|z_t\|^2+2\beta^2(b_{\max}+L_{\omega}C_g)^2.
\end{align}

\textbf{Term $(b)$ can be bounded as follows:}

For any $t\geq\tau_{\beta}$, we have that
\begin{align}
    &\left|\mathbb{E}\left[z_t^\top\left(-A_{\theta_t}z_{t}-\frac{1}{\beta}(z_{t+1}-z_{t})\right)\right]\right|\nn\\
    &\leq  (R_1+R_3+P_1+P_2+P_3)\mE\left[\left\|z_t\right\|^2\right]+(Q_1+Q_2+Q_3+P_1+P_2+P_3)\nn\\
    &\quad+\frac{\alpha}{8\beta}L_{\omega}\mE\left[\left\|\nabla J(\theta_t)\right\|^2\right],
\end{align}
where the definition of $P_i, Q_i$ and $R_i$, $i=1,2,3$,  can be found in \eqref{eq:term1}, \eqref{eq:term2} and \eqref{eq:term3}.

From \eqref{eq:aaa}, it can be shown that for any $t\geq \tau_{\beta}$,
 \begin{align}
    &\mE\left[\|z_{t+1}\|^2-\|z_t\|^2\right]\nn\\
    &\leq  2\beta(R_1+R_3+P_1+P_2+P_3)\mE\left[\left\|z_t\right\|^2\right]+2\beta(Q_1+Q_2+Q_3+P_1+P_2+P_3) \nn\\
    &\quad+\frac{\alpha}{4}L_{\omega}\mE\left[\left\|\nabla J(\theta_t)\right\|^2\right]+2\beta^2 C_{\phi}^4 \mE\left[\|z_t\|^2\right]+2\beta^2(b_{\max}+  L_{\omega}C_g)^2
    -2\beta\lambda_v\mE\left[\|z_t\|^2\right].
 \end{align}
 Thus by re-arranging the terms we obtain that 
 \begin{align}\label{eq:pq}
      &\mE\left[\|z_{t+1}\|^2\right]\nn\\
    &\leq (1-2\beta\lambda_v+2\beta(R_1+R_3+P_1+P_2+P_3)+2\beta^2C_{\phi}^4)\mE\left[\|z_t\|^2\right]+\frac{\alpha}{4}L_{\omega}\mE\left[\|\nabla J(\theta_{t})\|^2\right] \nn\\
    &\quad+2\beta(Q_1+Q_2+Q_3+P_1+P_2+P_3)+2\beta^2(b_{\max}+  L_{\omega}C_g)^2\nn\\
    &\triangleq (1-q)\mE\left[\|z_t\|^2\right]+\frac{\alpha}{4}L_{\omega}\mE\left[\|\nabla J(\theta_{t})\|^2\right]+p,
 \end{align}
 where $q=2\beta\lambda_v-2\beta(R_1+R_3+P_1+P_2+P_3)-2\beta^2C_{\phi}^4=\mo(\beta)$ and $p= 2\beta(Q_1+Q_2+Q_3+P_1+P_2+P_3)+2\beta^2(b_{\max}+  L_{\omega}C_g)^2=\mo(\beta^2\tb)$. 
Then by recursively using the previous inequality, it follows that for any $t\geq \tau_{\beta}$, 
\begin{align}
    \mE[\|z_t\|^2]\leq (1-q)^{t-\tau_{\beta}}\mE\left[\left\|z_{\tb}\right\|^2\right]+\frac{\alpha L_{\omega}}{4}\sum^t_{j=0}(1-q)^{t-j}\mE[|\nabla J(\theta_{j})\|^2]+\frac{p}{q},
\end{align}
and hence
\begin{align}\label{eq:markoviantrackingerror}
    &\frac{\sum^{T-1}_{t=0}\mE[\|z_t\|^2]}{T}\nn\\
    &=\frac{\sum^{T-1}_{t={\tau_{\beta}}}\mE[\|z_t\|^2]}{T}+\frac{\sum^{\tau_{\beta}-1}_{t=0}\mE[\|z_t\|^2]}{T}\nn\\
    &\leq \frac{\mE\left[\left\|z_{\tb}\right\|^2\right]}{Tq}+\frac{\tau_{\beta}\left(2\|z_0\|+2\beta\tau_{\beta}(b_{\max}+  L_{\omega}C_g) \right)^2}{T}+\frac{\alpha L_{\omega}}{4q}\frac{\sum^{T-1}_{t=0}\mE[\|\nabla J(\theta_t)\|^2]}{T}+\frac{p}{q}\nn\\
    &\leq {\left(2\|z_0\|+2\beta\tau_{\beta}(b_{\max}+  L_{\omega}C_g) \right)^2}\left(\frac{1}{Tq}+\frac{\tau_{\beta}}{T} \right)+\frac{\alpha L_{\omega}}{4q}\frac{\sum^{T-1}_{t=0}\mE[\|\nabla J(\theta_t)\|^2]}{T}+\frac{p}{q}\nn\\
    &=\mathcal{O}\left(\frac{1}{T\beta}+\frac{\alpha}{\beta}\frac{\sum^{T-1}_{t=0}\mE[\|\nabla J(\theta_t)\|^2]}{T}+\beta\tau_{\beta}\right),
\end{align}
where the last step is because $q=\mathcal{O}(\beta)$ and $p=\mathcal{O}(\beta^2\tau_{\beta})$.  

\subsubsection{Bound on Term $(a)$}\label{section:mara}
In this section we provide the detailed proof of the bound on term $(a)$ in \eqref{eq:aaa}.

We first note that from the update of $z_t$ in \eqref{eq:marzupdate}, term $\|z_{t+1}-z_t\|$ can be bounded as follows
\begin{align}\label{eq:zupdatebound2}
    \|z_{t+1}-z_t\|&\leq \|\beta(-A_{\theta_t}(s_t)z_t+b_t(\theta_t)) \|+\| \omega(\theta_t)-\omega(\theta_{t+1})\|\nn\\
    &\leq \beta C_{\phi}^2\|z_t\|+\beta b_{\max}+L_{\omega}\|\theta_t-\theta_{t+1}\|\nn\\
    &\overset{(a)}{\leq}\beta C_{\phi}^2\|z_t\|+\beta b_{\max}+\alpha L_{\omega}C_g\nn\\
    &\leq \beta C_{\phi}^2\|z_t\|+\beta (b_{\max}+ L_{\omega}C_g),
\end{align}
where $(a)$ is due to the fact $\|G_{t+1}(\theta_t,\om_t)\|\leq C_g$ for any $t\geq 0$, and where the last inequality is from the fact that $\alpha\leq \beta$. Hence term $(a)$ can be bounded as follows
\begin{align}
    \|z_{t+1}-z_t\|^2\leq 2\beta^2 C_{\phi}^4 \|z_t\|^2+2\beta^2(b_{\max}+L_{\omega}C_g)^2.
\end{align}
This completes the proof. 
\subsubsection{Bound on Term $(b)$}\label{section:marb}
In this section we provide the detailed proof of the bound on term $(b)$ in \eqref{eq:aaa}.

From  \eqref{eq:zupdatebound2}, it follows that 
\begin{align}\label{eq:zupdatebound1}
    \|z_{t+1}\|&\leq (1+\beta C_{\phi}^2)\|z_t\|+\beta b_{\max}+\alpha L_{\omega}C_g\nn\\
    &\leq(1+\beta C_{\phi}^2)\|z_t\|+\beta(b_{\max}+L_{\omega}C_g).
\end{align}

By applying \eqref{eq:zupdatebound1} recursively, it follows that 
\begin{align}
    \|z_{t}\|&\leq (1+\beta C_{\phi}^2)^t\|z_0\|+\beta(b_{\max}+L_{\omega}C_g)\frac{(1+\beta C_{\phi}^2)^t-1}{\beta C_{\phi}^2}\nn\\
    &= (1+\beta C_{\phi}^2)^t\|z_0\|+ (b_{\max}+L_{\omega}C_g)\frac{(1+\beta C_{\phi}^2)^t-1}{C_{\phi}^2}.
\end{align}
We first show the following lemma which bounds the update $\left\|z_{t}-z_{t-\tb}\right\|$ by $\|z_t\|$.
\begin{Lemma}\label{lemma:boundz} 
For any $t\geq\tau_{\beta}$ and $t\geq j\geq t-\tau_{\beta}$, we have that
\begin{align}
    \|z_j\|&\leq 2\|z_{t-\tau_{\beta}}\|+2\beta\tau_{\beta}(b_{\max}+L_{\omega}C_g); \\
     \|z_{t}-z_{t-\tau_{\beta}}\|&\leq  2\beta\tau_{\beta} C_{\phi}^2\|z_{t-\tau_{\beta}}\|+2\beta\tau_{\beta}(b_{\max}+ L_{\omega}C_g),\label{eq:t-tau}\\
          \|z_{t}-z_{t-\tau_{\beta}}\|
    &\leq 4\beta\tau_{\beta} C_{\phi}^2\|z_{t}\| +4\beta\tau_{\beta}(b_{\max}+ L_{\omega}C_g)\label{eq:t-taut}.
\end{align}
\end{Lemma}
\begin{proof}
From \eqref{eq:zupdatebound1}, it follows that 
\begin{align} \label{eq:ine1}
    \|z_{t+1}\| \leq(1+\beta C_{\phi}^2)\|z_t\|+\beta(b_{\max}+L_{\omega}C_g).
\end{align}
First note that $\beta C_{\phi}^2\tau_{\beta}\leq \frac{1}{4}$ and hence $\beta C_{\phi}^2\leq\frac{1}{4\tau_{\beta}}\leq\frac{\log 2}{\tau_{\beta}-1}$. This implies that 
\begin{align}\label{eq:ine2}
    (1+\beta C_{\phi}^2)^{\tau_{\beta}}\leq1+2{\tau_{\beta}}\beta C_{\phi}^2,
\end{align}
which is because  $(1+x)^k\leq 1+2kx$ for $x\leq\frac{\log 2}{k-1}$.

Applying inequality \eqref{eq:ine1} recursively,  it follows that 
\begin{align}\label{eq:zjnorm}
    \|z_j\|&\leq (1+\beta C_{\phi}^2)^{j-t+\tau_{\beta}}\|z_{t-\tau_{\beta}}\|+ (b_{\max}+L_{\omega}C_g)\frac{(1+\beta C_{\phi}^2)^{\tau_{\beta}}-1}{C_{\phi}^2}\nn\\
    &\leq (1+\beta C_{\phi}^2)^{\tau_{\beta}}\|z_{t-\tau_{\beta}}\|+ (b_{\max}+L_{\omega}C_g)\frac{(1+\beta C_{\phi}^2)^{\tau_{\beta}}-1}{C_{\phi}^2}\nn\\
    &\overset{(a)}{\leq} (1+2{\tau_{\beta}}\beta C_{\phi}^2)\|z_{t-\tau_{\beta}}\|+2\beta\tau_{\beta}(b_{\max}+L_{\omega}C_g)\nn\\
    &\overset{(b)}{\leq}2\|z_{t-\tau_{\beta}}\|+2\beta\tau_{\beta}(b_{\max}+L_{\omega}C_g),
\end{align}
where $(a)$ is from \eqref{eq:ine2}, and  $(b)$ is from the fact that $\beta\tau_{\beta}C_{\phi}^2\leq \frac{1}{4}$.


To prove \eqref{eq:t-tau} and \eqref{eq:t-taut}, first note that  
\begin{align}
    \|z_{t}-z_{t-\tau_{\beta}}\|&\leq \sum^{{t}-1}_{j={t-\tau_{\beta}}} \|z_{j+1}-z_j\|\nn\\
    &\overset{(a)}{\leq}\sum^{{t}-1}_{j={t-\tau_{\beta}}}  \beta C_{\phi}^2\|z_j\|+\beta\tau_{\beta} (b_{\max}+ L_{\omega}C_g)\nn\\
    &\overset{(b)}{\leq}  \sum^{{t}-1}_{j={t-\tau_{\beta}}}  \beta C_{\phi}^2(2\|z_{t-\tau_{\beta}}\|+2\beta\tau_{\beta}(b_{\max}+L_{\omega}C_g))+\beta\tau_{\beta} (b_{\max}+ L_{\omega}C_g)\nn\\
    & {\leq} \beta\tau_{\beta} C_{\phi}^2(2\|z_{t-\tau_{\beta}}\|+2\beta\tau_{\beta}(b_{\max}+L_{\omega}C_g))+\beta\tau_{\beta} (b_{\max}+ L_{\omega}C_g)\nn\\
    &=2\beta\tau_{\beta} C_{\phi}^2\|z_{t-\tau_{\beta}}\|+(2\beta^2\tau_{\beta}^2C_{\phi}^2+\beta\tau_{\beta})(b_{\max}+ L_{\omega}C_g)\nn\\
    &\overset{(c)}{\leq}  2\beta\tau_{\beta} C_{\phi}^2\|z_{t-\tau_{\beta}}\|+2\beta\tau_{\beta}(b_{\max}+ L_{\omega}C_g),
\end{align}
where $(a)$ is from \eqref{eq:zupdatebound2}, $(b)$ is from \eqref{eq:zjnorm} and $(c)$ is due to the fact that $\beta \tb C_{\phi}^2\leq\frac{1}{4}$ . Moreover, it can be further shown that 
\begin{align}
    \|z_{t}-z_{t-\tau_{\beta}}\|&\leq 2\beta\tau_{\beta} C_{\phi}^2(\|z_{t}\| +\|z_{t}-z_{t-\tau_{\beta}} \|)+2\beta\tau_{\beta}(b_{\max}+ L_{\omega}C_g)\nn\\
    &\leq 2\beta\tau_{\beta} C_{\phi}^2\|z_{t}\| +\frac{1}{2}\|z_{t}-z_{t-\tau_{\beta}}\|+2\beta\tau_{\beta}(b_{\max}+ L_{\omega}C_g),
\end{align}
where the last step is because $\beta\tau_{\beta}C_{\phi}^2\leq\frac{1}{4}$. Hence
\begin{align}
      \|z_{t}-z_{t-\tau_{\beta}}\|
    \leq 4\beta\tau_{\beta} C_{\phi}^2\|z_{t}\| +4\beta\tau_{\beta}(b_{\max}+ L_{\omega}C_g).
\end{align}
\end{proof}

The bound on term $(b)$ in \eqref{eq:aaa} is straightforward from the following lemma.
\begin{Lemma}\label{lemma5}
For any $t\geq\tau_{\beta}$, it follows that 
\begin{align}
    &\left|\mathbb{E}\left[z_t^\top\left(-A_{\theta_t}z_{t}-\frac{1}{\beta}(z_{t+1}-z_{t})\right)\right]\right|\nn\\
    &\leq  (R_1+R_3+P_1+P_2+P_3)\mE\left[\left\|z_t\right\|^2\right]+(Q_1+Q_2+Q_3+P_1+P_2+P_3)\nn\\
    &\quad+\frac{\alpha}{8\beta}L_{\omega}\mE\left[\left\|\nabla J(\theta_t)\right\|^2\right],
\end{align}
where the definition of $P_i, Q_i$ and $R_i$, $i=1,2,3,$  can be found in \eqref{eq:term1}, \eqref{eq:term2} and \eqref{eq:term3}. 
\end{Lemma}
\begin{proof}
We only prove the case $t=\tau_{\beta}$ here. The proof for the general case with $t>\tb$ is similar, and thus is omitted here. 
First note that 
\begin{align}\label{eq:terms}
    &\mathbb{E}\left[z_{\tau_{\beta}}^\top\left(-A_{\theta_{\tau_{\beta}}}z_{\tau_{\beta}}-\frac{1}{\beta}\left(z_{\tau_{\beta}+1}-z_{\tau_{\beta}}\right)\right)\right]\nn\\
    &=\mathbb{E}\left[z_{\tau_{\beta}}^\top\left(-A_{\theta_{\tau_{\beta}}}+A_{\theta_{\tau_{\beta}}}\left(s_{\tau_{\beta}}\right)\right)z_{\tau_{\beta}}\right]-\mathbb{E}\left[z_{\tau_{\beta}}^\top b_{\tau_{\beta}}\right]-\mathbb{E}\left[ z_{\tau_{\beta}}^\top\frac{\omega\left(\theta_{\tau_{\beta}}\right)-\omega\left(\theta_{\tau_{\beta}+1}\right)}{\beta}\right].
\end{align}
 We then bound the terms in \eqref{eq:terms} one by one. First, it can be shown that 
\begin{align}
    &\left|\mathbb{E}\left[z_{\tau_{\beta}}^\top\left(-A_{\theta_{\tau_{\beta}}}+A_{\theta_{\tau_{\beta}}}\left(s_{\tau_{\beta}}\right)\right)z_{\tau_{\beta}}\right ]\right|\nn\\
    &\leq\left|\mathbb{E}\left[z_0^\top\left(-A_{\theta_{\tau_{\beta}}}+A_{\theta_{\tau_{\beta}}}\left(s_{\tau_{\beta}}\right)\right)z_0\right ]\right|+\left|\mE\left[\left(z_{\tau_{\beta}}-z_0\right)^\top \left(-A_{\theta_{\tau_{\beta}}}+A_{\theta_{\tau_{\beta}}}\left(s_{\tau_{\beta}}\right)\right)\left(z_{\tau_{\beta}}-z_0\right)\right ]\right|\nn\\
    &\quad+2\left|\mE\left[\left(z_{\tau_{\beta}}-z_0\right)^\top \left(-A_{\theta_{\tau_{\beta}}}+A_{\theta_{\tau_{\beta}}}\left(s_{\tau_{\beta}}\right)\right)z_0 \right ]\right|\nn\\
    &\leq \|z_0\|^2\left\|\mE\left[-A_{\theta_{\tau_{\beta}}}+A_{\theta_{\tau_{\beta}}}\left(s_{\tau_{\beta}}\right)\right ]\right\|+2C_{\phi}^2\mE\left[\|z_{\tau_{\beta}}-z_0\|^2\right ]+4\|z_0\|C_{\phi}^2\mE\left[\|z_{\tau_{\beta}}-z_0\|\right ]\nn\\
    &\leq \|z_0\|^2\left\|\mE\left[-A_{\theta_0}+A_{\theta_0}\left(s_{\tau_{\beta}}\right)\right ]\right\|+\|z_0\|^2\left\|\mE\left[-A_{\theta_0}+A_{\theta_{\tb}}\right ]\right\|\nn\\
    &\quad+\|z_0\|^2\left\|\mE\left[-A_{\theta_{\tb}}\left(s_{\tau_{\beta}}\right)+A_{\theta_0}\left(s_{\tau_{\beta}}\right)\right ]\right\|+2C_{\phi}^2\mE\left[\|z_{\tau_{\beta}}-z_0\|^2\right ]+4\|z_0\|C_{\phi}^2\mE\left[\|z_{\tau_{\beta}}-z_0\|\right ]\nn\\
    &\overset{\left(a\right)}{\leq}\left(\beta C_{\phi}^2+4C_{\phi}D_vC_g\alpha\tb\right)\|z_0\|^2+2C_{\phi}^2\mE\left[\|z_{\tau_{\beta}}-z_0\|^2\right ]+4\|z_0\|C_{\phi}^2\mE\left[\|z_{\tau_{\beta}}-z_0\|\right ],
\end{align}
where $(a)$ is due to the facts that  $\left\| \mE\left[-A_{\theta_0}+A_{\theta_0}(s_{\tau_{\beta}})\right]\right\|\leq C_{\phi}^2\beta$ from the uniform ergodicity of the MDP, both $A_{\theta}$ and $A_{\theta}(s_{\tb})$ are Lipschitz with constant $2C_{\phi}D_v$, and $\|\theta_0-\theta_{\tb}\|\leq\sum^{\tb-1}_{j=0}\|\theta_{j+1}-\theta_j\|\leq \alpha\tb C_g$. 

We then plug in the results from Lemma \ref{lemma:boundz}, and hence we have that
\begin{align}\label{eq:term1}
    &\left|\mathbb{E}\left[z_{\tau_{\beta}}^\top\left(-A_{\theta_{\tau_{\beta}}}+A_{\theta_{\tau_{\beta}}}\left(s_{\tau_{\beta}}\right)\right)z_{\tau_{\beta}}\right ]\right|\nn\\
    &\leq\left(\beta C_{\phi}^2+4C_{\phi}D_vC_g\alpha\tb\right)\|z_0\|^2+2C_{\phi}^2\mE\left[\|z_{\tau_{\beta}}-z_0\|^2\right ]+4\|z_0\|C_{\phi}^2\mE\left[\|z_{\tau_{\beta}}-z_0\|\right ]\nn\\
    &\overset{(a)}{\leq} \left(\beta C_{\phi}^2+4C_{\phi}D_vC_g\alpha\tb\right)\left( 2(1+4\beta\tb C_{\phi}^2)^2\mE\left[\left\|z_{\tb}\right\|^2\right]+32\beta^2\tb^2(b_{\max}+L_{\om}C_g)^2\right)\nn\\
    &\quad+2C_{\phi}^2\left(32\beta^2\tb^2C_{\phi}^4\mE\left[\left\|z_{\tb}\right\|^2\right]+32\beta^2\tb^2(b_{\max}+L_{\om}C_g)^2 \right)\nn\\
    &\quad+4C_{\phi}^2\left(4\beta\tb C_{\phi}^2(1+4\beta\tb C_{\phi}^2)\mE\left[\left\|z_{\tb}\right\|^2\right]+4\beta\tb (b_{\max}+L_{\om}C_g)(1+8\beta\tb C_{\phi}^2)\mE\left[\left\|z_{\tb}\right\| \right] \right)\nn\\
    &\quad+64C_{\phi}^2 \beta^2\tb^2(b_{\max}+L_{\om}C_g)^2\nn\\
    &\triangleq R_1\mE\left[\left\|z_{\tb}\right\|^2\right]+P_1\mE\left[\left\|z_{\tb}\right\|\right]+Q_1,
    \end{align}
where $(a)$ is from \eqref{eq:t-taut} and the fact that \begin{align}
    \|z_0\|\leq\left\|z_{\tb}-z_0\right\|+\left\|z_{\tb}\right\|\leq (1+4\beta\tb C_{\phi}^2)\left\|z_{\tb}\right\|+4\beta\tb (b_{\max}+  L_{\om}C_g);
\end{align}
and $R_1=2(1+4\beta\tb C_{\phi}^2)^2\left(\beta C_{\phi}^2+4C_{\phi}D_vC_g\alpha\tb\right)+64\beta^2\tb^2C_{\phi}^6+16 \beta\tb C_{\phi}^4(1+4\beta\tb C_{\phi}^2)=\mo(\beta\tb)$, $P_1=16C_{\phi}^2\beta\tb (b_{\max}+L_{\om}C_g)(1+8\beta\tb C_{\phi}^2)=\mo(\beta\tb)$ and $Q_1= \left(\beta C_{\phi}^2+4C_{\phi}D_vC_g\alpha\tb\right)32\beta^2\tb^2(b_{\max}+L_{\om}C_g)^2+64C_{\phi}^2\beta^2\tb^2(b_{\max}+L_{\om}C_g)^2+64C_{\phi}^2 \beta^2\tb^2(b_{\max}+L_{\om}C_g)^2=\mo(\beta^2\tau^2)$.

Similarly, the second term in \eqref{eq:terms} can be bounded as follows
\begin{align} 
    \left|\mE\left[z_{\tau_{\beta}}^\top b_{\tau_{\beta}}(\theta_{\tb})\right]\right|&\leq\left|\mE\left[(z_{\tau_{\beta}}-z_0)^\top b_{\tau_{\beta}}(\theta_{\tb})\right]\right|+\left|\mE\left[z_0^\top b_{\tau_{\beta}}(\theta_0)\right]\right|\nn\\
    &\quad+\|\mE\left[z_0^\top( b_{\tau_{\beta}}(\theta_{\tb})-b_{\tau_{\beta}}(\theta_0))\right]\|\nn\\
    &\leq b_{\max}\mE\left[\|z_{\tau_{\beta}}-z_0\|\right]+\beta b_{\max}\|z_0\|+\alpha\tb C_gL_b\|z_0\|,
\end{align}
where $L_b=2C_{\phi}D_vR_{\om}+L_{\om}C_{\phi}^2+\rho_{\max}((1+\gamma)C_{\phi}^2+D_v(r_{\max}+(1+\gamma)C_v))$ is the Lipschitz constant of $b_t(\theta)$. 
Again applying Lemma \ref{lemma:boundz} implies that
\begin{align}\label{eq:term2}
    &\left|\mE\left[z_{\tau_{\beta}}^\top b_{\tau_{\beta}}(\theta_{\tb})\right]\right|\nn\\
    &\leq b_{\max}\mE\left[\|z_{\tau_{\beta}}-z_0\|\right]+\beta b_{\max}\|z_0\|+\alpha\tb C_gL_b\|z_0\|\nn\\
    &\leq b_{\max}\left( 4\beta\tb C_{\phi}^2\mE\left[\left\|z_{\tb}\right\|\right]+4\beta\tb(b_{\max}+L_{\om}C_g)\right)\nn\\
    &\quad+(\beta b_{\max}+\alpha\tb C_gL_b)\left( \left(1+4\beta\tb C_{\phi}^2\right)\mE\left[\left\|z_{\tb}\right\|\right]+4\beta\tb(b_{\max}+L_{\om}C_g)\right)\nn\\
    &\triangleq P_2\mE\left[\left\|z_{\tb}\right\|\right]+Q_2,
\end{align}
where $P_2=4\beta\tb b_{\max} C_{\phi}^2+(\beta b_{\max}+\alpha\tb C_gL_b) \left(1+4\beta\tb C_{\phi}^2\right)=\mo(\beta\tb)$ and $Q_2=4\beta\tb(b_{\max}+L_{\om}C_g)(b_{\max}+\beta b_{\max}+\alpha\tb  C_gL_b)=\mo(\beta\tb)$. 

We then bound the last term in \eqref{eq:terms} as follows 
\begin{align}\label{eq:trackingterm3}
    &\left|\mathbb{E}\left[ z_{\tau_{\beta}}^\top\frac{\omega(\theta_{\tau_{\beta}})-\omega(\theta_{\tau_{\beta}+1})}{\beta}\right]\right|\nn\\
    &\overset{(a)}{=}\left|\frac{1}{\beta}\mE[z_{\tau_{\beta}}^\top \nabla \omega(\hat{\theta}_{\tau_{\beta}})(\theta_{\tau_{\beta}+1}-\theta_{\tau_{\beta}})]\right|\nn\\
    &=\left|\frac{\alpha}{\beta}\mE[z_{\tau_{\beta}}^\top \nabla \omega(\hat{\theta}_{\tau_{\beta}})  G_{\tau_{\beta}+1}(\theta_{\tau_{\beta}},\omega_{\tau_{\beta}})]\right|\nn\\
    &=\Bigg|\frac{\alpha}{\beta}\mE\Bigg[z_{\tau_{\beta}}^\top \nabla \omega(\hat{\theta}_{\tau_{\beta}})\Bigg(  G_{\tau_{\beta}+1}(\theta_{\tau_{\beta}},\omega_{\tau_{\beta}})-  G_{\tau_{\beta}+1}(\theta_{\tau_{\beta}},\omega(\theta_{\tau_{\beta}}))+  G_{\tau_{\beta}+1}(\theta_{\tau_{\beta}},\omega(\theta_{\tau_{\beta}}))\nn\\
    &\quad+\frac{\nabla J(\theta_{\tau_{\beta}})}{2}-\frac{\nabla J(\theta_{\tau_{\beta}})}{2}\Bigg)\Bigg]\Bigg|\nn\\
    &=\Bigg|\frac{\alpha}{\beta}\mE\Bigg[z_{\tau_{\beta}}^\top \nabla \omega(\hat{\theta}_{\tau_{\beta}})(  G_{\tau_{\beta}+1}(\theta_{\tau_{\beta}},\omega_{\tau_{\beta}})-  G_{\tau_{\beta}+1}(\theta_{\tau_{\beta}},\omega(\theta_{\tau_{\beta}})))\Bigg]\Bigg|\nn\\
    &\quad+\left|\frac{\alpha}{\beta}\mE\left[z_{\tau_{\beta}}^\top \nabla \omega(\hat{\theta}_{\tau_{\beta}})\left(  G_{\tau_{\beta}+1}(\theta_{\tau_{\beta}},\omega(\theta_{\tau_{\beta}}))+\frac{\nabla J(\theta_{\tau_{\beta}})}{2}\right)\right]\right|\nn\\
    &\quad+\Bigg|\frac{\alpha}{\beta}\mE\Bigg[z_{\tau_{\beta}}^\top \nabla \omega(\hat{\theta}_{\tau_{\beta}})\left(-\frac{\nabla J(\theta_{\tau_{\beta}})}{2}\right)\Bigg]\Bigg|\nn\\
     &\overset{(b)}{\leq}\frac{\alpha}{\beta}  L_{\omega}L_g\mE\left[\left\|z_{\tb}\right\|^2\right]+\frac{\alpha}{2\beta}L_{\omega}\mE\left[\left\|z_{\tb}\right\|^2\right]+\frac{\alpha}{8\beta}L_{\omega}\mE\left[\left\|\nabla J(\theta_{\tau_{\beta}})\right\|^2\right]\nn\\
    &\quad+\frac{\alpha}{\beta}\left|\mE\left[z_{\tau_{\beta}}^\top \nabla \omega({\theta}_{\tau_{\beta}})\left(  G_{\tau_{\beta}+1}(\theta_{\tau_{\beta}},\omega(\theta_{\tau_{\beta}}))+\frac{\nabla J(\theta_{\tau_{\beta}})}{2}\right)\right]\right|\nn\\
    &\quad+\frac{\alpha}{\beta}\left|\mE\left[z_{\tau_{\beta}}^\top (\nabla \omega(\hat{\theta}_{\tau_{\beta}})-\nabla \omega(\theta_{\tau_{\beta}}))\left(  G_{\tau_{\beta}+1}(\theta_{\tau_{\beta}},\omega(\theta_{\tau_{\beta}}))+\frac{\nabla J(\theta_{\tau_{\beta}})}{2}\right)\right]\right|\nn\\
    &\leq \frac{\alpha}{\beta}  L_{\omega}L_g\mE\left[\left\|z_{\tb}\right\|^2\right]+\frac{\alpha}{2\beta}L_{\omega}\mE\left[\left\|z_{\tb}\right\|^2\right]+\frac{\alpha}{8\beta}L_{\omega}\mE\left[\left\|\nabla J(\theta_{\tau_{\beta}})\right\|^2\right]\nn\\
    &\quad+\frac{\alpha}{\beta}\left|\mE\left[z_0^\top \nabla \omega({\theta}_{\tau_{\beta}})\left(  G_{\tau_{\beta}+1}(\theta_{\tau_{\beta}},\omega(\theta_{\tau_{\beta}}))+\frac{\nabla J(\theta_{\tau_{\beta}})}{2}\right)\right]\right|\nn\\
    &\quad+\frac{\alpha}{\beta}\left|\mE\left[(z_{\tau_{\beta}}-z_0)^\top \nabla \omega({\theta}_{\tau_{\beta}})\left(  G_{\tau_{\beta}+1}(\theta_{\tau_{\beta}},\omega(\theta_{\tau_{\beta}}))+\frac{\nabla J(\theta_{\tau_{\beta}})}{2}\right)\right]\right| \nn\\
    &\quad+\frac{2\alpha}{\beta}  C_gD_{\omega}\mE\left[\left\|z_{\tb}\right\|\left\|\theta_{\tau_{\beta}}-\theta_{\tau_{\beta}+1}\right \|\right]\nn\\
    &\leq \frac{\alpha}{\beta}  L_{\omega}L_g\mE\left[\left\|z_{\tb}\right\|^2\right]+\frac{\alpha}{2\beta}L_{\omega}\mE\left[\left\|z_{\tb}\right\|^2\right]+\frac{\alpha}{8\beta}L_{\omega}\mE\left[\left\|\nabla J(\theta_{\tau_{\beta}})\right\|^2\right]\nn\\
    &\quad+\frac{\alpha}{\beta}\left|\mE\left[z_0^\top \nabla \omega({\theta}_0)\left(  G_{\tau_{\beta}+1}(\theta_0,\omega(\theta_0))+\frac{\nabla J(\theta_0)}{2}\right)\right]\right|\nn\\
    &\quad+\frac{\alpha}{\beta}\Bigg|\mE\Bigg[z_0^\top \Bigg( \nabla \omega({\theta}_{\tb})\left(  G_{\tau_{\beta}+1}(\theta_{\tb},\omega(\theta_{\tb}))+\frac{\nabla J(\theta_{\tb})}{2}\right)\nn\\
    &\quad- \nabla \omega({\theta}_0)\left(  G_{\tau_{\beta}+1}(\theta_0,\omega(\theta_0))+\frac{\nabla J(\theta_0)}{2}\right) \Bigg)\Bigg]\Bigg|\nn\\
    &\quad+\frac{\alpha}{\beta}\left|\mE\left[(z_{\tau_{\beta}}-z_0)^\top \nabla \omega({\theta}_{\tau_{\beta}})\left(  G_{\tau_{\beta}+1}(\theta_{\tau_{\beta}},\omega(\theta_{\tau_{\beta}})+\frac{\nabla J(\theta_{\tau_{\beta}})}{2}\right)\right]\right|\nn\\
    &\quad+\frac{2\alpha}{\beta}  C_gD_{\omega}\mE\left[\left\|z_{\tb}\right\|\left\|\theta_{\tau_{\beta}}-\theta_{\tau_{\beta}+1} \right\|\right]\nn\\
    &\leq \frac{\alpha}{\beta}  L_{\omega}L_g\mE\left[\left\|z_{\tb}\right\|^2\right]+\frac{\alpha}{8\beta}L_{\omega}\mE\left[\left\|\nabla J(\theta_{\tau_{\beta}})\right\|^2\right]\nn\\
    &\quad+\frac{\alpha}{\beta}\|z_0\|L_{\omega}\left\|\mE\left[   G_{\tau_{\beta}+1}(\theta_0,\omega(\theta_0))+\frac{\nabla J(\theta_{\tau_{\beta}})}{2} \right]\right\|\nn\\
    &\quad+\frac{\alpha}{2\beta}L_{\omega}\mE\left[\left\|z_{\tb}\right\|^2\right]+\frac{\alpha}{\beta}\|z_0\|L_k\mE\left[\left\|\theta_{\tau_{\beta}}-\theta_0 \right\|\right] +\frac{2\alpha}{\beta}  L_{\omega}C_g\mE\left[\left\|z_{\tau_{\beta}}-z_0\right\| \right]\nn\\
    &\quad+\frac{2\alpha}{\beta}  C_gD_{\omega}\mE\left[\left\|z_{\tb}\right\|\left\|\theta_{\tau_{\beta}}-\theta_{\tau_{\beta}+1} \right\|\right]\nn\\
    &\overset{(c)}{\leq}\frac{\alpha}{\beta}  L_{\omega}L_g\mE\left[\left\|z_{\tb}\right\|^2\right]+\frac{\alpha}{2\beta}L_{\omega}\mE\left[\left\|z_{\tb}\right\|^2\right]+\frac{\alpha}{8\beta}L_{\omega}\mE\left[\left\|\nabla J(\theta_{\tau_{\beta}})\right\|^2\right]+\frac{\alpha}{\beta}\|z_0\|L_{\omega}C_g  \beta\nn\\
    &\quad+\frac{\alpha^2}{\beta}\tb L_k\|z_0\|  C_g +\frac{2\alpha}{\beta}  L_{\omega}C_g\mE\left[\left\|z_{\tau_{\beta}}-z_0\right\| \right]+\frac{2\alpha^2}{\beta}  C_g^2D_{\omega}\mE\left[\left\|z_{\tb}\right\|\right]\nn\\
    &=\left(\frac{\alpha}{\beta}  L_{\omega}L_g+\frac{\alpha}{2\beta} L_{\omega}\right)\mE\left[\left\|z_{\tb}\right\|^2\right]+\frac{2\alpha^2}{\beta}  C_g^2D_{\omega}\mE\left[\left\|z_{\tb}\right\|\right]+\frac{\alpha}{8\beta}L_{\omega}\mE\left[\left\|\nabla J(\theta_{\tau_{\beta}})\right\|^2\right]\nn\\ 
    &\quad+\left(\alpha   L_{\omega}C_g +\frac{\alpha^2}{\beta}\tb L_k  C_g\right)\|z_0\|+\frac{2\alpha}{\beta}  L_{\omega}C_g\mE\left[\left\|z_{\tau_{\beta}}-z_0\right\| \right],
\end{align}
where $(a)$ is from the Mean-Value theorem and $\hat{\theta}_{\tau_{\beta}}=c\theta_{\tau_{\beta}}+(1-c)\theta_{\tau_{\beta}+1}$ for some $c\in[0,1]$, $(b)$ is from Lemmas  \ref{lemma:omegalip} and \ref{lemma:graomegalip}, $(c)$ is due to the fact that $\left\|\mE\left[  G_{t+1}(\theta_0,\omega(\theta_0))+\frac{\nabla J(\theta_0)}{2} \right]\right\|\leq   C_g\beta$ for any $t\geq \tau_{\beta}$ and $\|\theta_{\tb}-\theta_0 \|\leq \alpha\tb   C_g$, and $L_k=2  C_gD_{\om}+\left(L_J+ \frac{ L_g'}{2}\right)L_{\om}$ is the Lipschitz constant of $ \nabla \om(\theta)\left(  G_{t+1}(\theta,\om(\theta))+\frac{\nabla J(\theta)}{2} \right)$, and $L_g'$ is the Lipschitz constant of $G_{t+1}(\theta,\om(\theta))$. 

Our next step is to rewrite the bound in \eqref{eq:trackingterm3} using $\|z_{\tb}\|$. Note that from Lemma \ref{lemma:boundz}, we have that 
\begin{align}
    \|z_0\|\leq\left\|z_{\tb}-z_0\right\|+\left\|z_{\tb}\right\|\leq (1+4\beta\tb C_{\phi}^2)\left\|z_{\tb}\right\|+4\beta\tb (b_{\max}+  L_{\om}C_g).
\end{align}
Plugging in \eqref{eq:trackingterm3}, it follows that 
\begin{align}\label{eq:term3}
    &\left|\mathbb{E}\left[ z_{\tau_{\beta}}^\top\frac{\omega(\theta_{\tau_{\beta}})-\omega(\theta_{\tau_{\beta}+1})}{\beta}\right]\right|\nn\\
    &\leq  \left(\frac{\alpha}{\beta}  L_{\omega}L_g+\frac{\alpha}{2\beta} L_{\omega}\right)\mE\left[\left\|z_{\tb}\right\|^2\right]+\frac{2\alpha^2}{\beta}   C_g^2D_{\omega}\mE\left[\left\|z_{\tb}\right\|\right]+\frac{\alpha}{8\beta}L_{\omega}\mE\left[\left\|\nabla J(\theta_{\tau_{\beta}})\right\|^2\right]\nn\\ 
    &\quad+\left(\alpha   L_{\omega}C_g +\frac{\alpha^2}{\beta}\tb L_k  C_g\right)\|z_0\|+\frac{2\alpha}{\beta}  L_{\omega}C_g\mE\left[\left\|z_{\tau_{\beta}}-z_0\right\| \right]\nn\\
    &\leq  \left(\frac{\alpha}{\beta}  L_{\omega}L_g+\frac{\alpha}{2\beta} L_{\omega}\right)\mE\left[\left\|z_{\tb}\right\|^2\right]+\frac{2\alpha^2}{\beta}   C_g^2D_{\omega}\mE\left[\left\|z_{\tb}\right\|\right]+\frac{\alpha}{8\beta}L_{\omega}\mE\left[\left\|\nabla J(\theta_{\tau_{\beta}})\right\|^2\right]\nn\\ 
    &\quad+\left(\alpha   L_{\omega}C_g +\frac{\alpha^2}{\beta}\tb L_k  C_g\right)\left( (1+4\beta\tb C_{\phi}^2)\mE\left[\left\|z_{\tb}\right\|\right]+4\beta\tb(b_{\max}+  L_{\om}C_g)\right)\nn\\
    &\quad+\frac{2\alpha}{\beta}  L_{\omega}C_g\left(\mE\left[ 4\beta\tb C_{\phi}^2\left\|z_{\tb}\right\|\right]+4\beta\tb(b_{\max}+  L_{\om}C_g)\right)\nn\\
    &= \left(\frac{\alpha}{\beta}  L_{\omega}L_g+\frac{\alpha}{2\beta} L_{\omega}\right)\mE\left[\left\|z_{\tb}\right\|^2\right]\nn\\
    &\quad +\left(\frac{2\alpha^2}{\beta} C_g^2D_{\omega}+\left(\alpha   L_{\omega}C_g +\frac{\alpha^2}{\beta}\tb L_k  C_g\right)  (1+4\beta\tb C_{\phi}^2)+8\alpha\tb  L_{\om}C_gC_{\phi}^2\right)\mE\left[\left\|z_{\tb}\right\|\right]\nn\\
    &\quad+\frac{\alpha}{8\beta}L_{\omega}\mE\left[\left\|\nabla J(\theta_{\tau_{\beta}})\right\|^2\right] +\left(\alpha   L_{\omega}C_g +\frac{\alpha^2}{\beta}\tb L_k  C_g\right)\left(4\beta\tb(b_{\max}+  L_{\om}C_g)\right)\nn\\
    &\quad+{8\alpha\tb}  L_{\omega}C_g(b_{\max}+  L_{\om}C_g)\nn\\
    &\triangleq R_3\mE\left[\left\|z_{\tb}\right\|^2\right]+P_3\mE\left[\left\|z_{\tb}\right\|\right]+Q_3+\frac{\alpha}{8\beta}L_{\omega}\mE\left[\left\|\nabla J(\theta_{\tau_{\beta}})\right\|^2\right],
\end{align}
where $R_3=\left(\frac{\alpha}{\beta}  L_{\omega}L_g+\frac{\alpha}{2\beta} L_{\omega}\right)=\mo\left(\frac{\alpha}{\beta}\right)$, $P_3=\bigg(\frac{2\alpha^2}{\beta} C_g^2D_{\omega}+\left(\alpha   L_{\omega}C_g +\frac{\alpha^2}{\beta}\tb L_k  C_g\right)  (1+4\beta\tb C_{\phi}^2)+8\alpha\tb  L_{\om}C_gC_{\phi}^2\bigg)=\mo(\alpha\tb)$ and $Q_3=\left(\alpha   L_{\omega}C_g +\frac{\alpha^2}{\beta}\tb L_k  C_g\right)\left(4\beta\tb(b_{\max}+  L_{\om}C_g)\right) +{8\alpha\tb}  L_{\omega}C_g(b_{\max}+  L_{\om}C_g)=\mo(\alpha\tb)$.

Then we combine all three bounds in \eqref{eq:term1}, \eqref{eq:term2} and \eqref{eq:term3}, and it follows that 
\begin{align} \label{eq:con}
    &\left|\mathbb{E}\left[z_{\tau_{\beta}}^\top\left(-A_{\theta_{\tau_{\beta}}}z_{\tau_{\beta}}-\frac{1}{\beta}\left(z_{\tau_{\beta}+1}-z_{\tau_{\beta}}\right)\right)\right]\right|\nn\\
    &\leq (R_1+R_3)\mE\left[\left\|z_{\tb}\right\|^2\right]+(P_1+P_2+P_3)\mE\left[\left\|z_{\tb}\right\| \right]+(Q_1+Q_2+Q_3)\nn\\
    &\quad+\frac{\alpha}{8\beta}L_{\omega}\mE\left[\left\|\nabla J(\theta_{\tau_{\beta}})\right\|^2\right],
\end{align}

Finally due to the fact that  $x\leq x^2+1$, $\forall x \in \mathbb{R}$, it follows that 
\begin{align}\label{eq:zterm3new}
    &\left|\mathbb{E}\left[z_{\tau_{\beta}}^\top\left(-A_{\theta_{\tau_{\beta}}}z_{\tau_{\beta}}-\frac{1}{\beta}\left(z_{\tau_{\beta}+1}-z_{\tau_{\beta}}\right)\right)\right]\right|\nn\\
    &\leq (R_1+R_3+P_1+P_2+P_3)\mE\left[\left\|z_{\tb}\right\|^2\right]+(Q_1+Q_2+Q_3+P_1+P_2+P_3)\nn\\
    &\quad+\frac{\alpha}{8\beta}L_{\omega}\mE\left[\left\|\nabla J(\theta_{\tau_{\beta}})\right\|^2\right].
\end{align}
This completes the proof.
\end{proof}

\subsection{Proof under the Markovian Setting}\label{section:markovmain}
In this section, we prove Theorem \ref{thm:main} under the Markovian setting.

From the $L_J$-smoothness of $J(\theta)$, it follows that 
\begin{align}
    J(\theta_{t+1}) &\leq J(\theta_t) +\left\langle \nabla J(\theta_t), \theta_{t+1}-\theta_t\right\rangle  + \frac{L_J}{2} \| \theta_{t+1}-\theta_t\|^2\nn\\
    &=J(\theta_t) +\alpha \left\langle \nabla J(\theta_t),  G_{t+1}(\theta_t,\omega_t) \right\rangle  + \frac{L_J}{2} \alpha^2\|  G_{t+1}(\theta_t,\omega_t)\|^2\nn\\
    &=J(\theta_t)-\alpha\left\langle \nabla J(\theta_t),-  G_{t+1}(\theta_t, \omega_t)-\frac{\nabla J(\theta_t)}{2}+  G_{t+1}(\theta_t, \omega(\theta_t))-  G_{t+1}(\theta_t, \omega(\theta_t)) \right\rangle \nn\\
    &\quad-\frac{\alpha}{2}\|\nabla J(\theta_t)\|^2+\frac{L_J}{2} \alpha^2\|  G_{t+1}(\theta_t,\omega_t)\|^2\nn\\
    &=J(\theta_t)-\alpha\left\langle \nabla J(\theta_t),-  G_{t+1}(\theta_t, \omega_t)+  G_{t+1}(\theta_t, \omega(\theta_t)) \right\rangle\nn\\
    &\quad+\alpha \left\langle \nabla J(\theta_t), \frac{\nabla J(\theta_t)}{2}+  G_{t+1}(\theta_t, \omega(\theta_t)) \right\rangle-\frac{\alpha}{2}\|\nabla J(\theta_t)\|^2+\frac{L_J}{2} \alpha^2\|  G_{t+1}(\theta_t,\omega_t)\|^2\nn\\
    & {\leq} J(\theta_t) +\alpha   L_g\|\nabla J(\theta_t) \|\|\omega(\theta_t)-\omega_t \|+\alpha \left\langle \nabla J(\theta_t), \frac{\nabla J(\theta_t)}{2}+  G_{t+1}(\theta_t, \omega(\theta_t)) \right\rangle \nn\\
    &\quad-\frac{\alpha}{2}\|\nabla J(\theta_t)\|^2+\frac{L_J}{2} \alpha^2\|  G_{t+1}(\theta_t,\omega_t)\|^2\nn\\
    &\overset{(a)}{\leq}J(\theta_t) +\alpha   L_g\|\nabla J(\theta_t) \|\|z_t \|+\alpha \left\langle \nabla J(\theta_t), \frac{\nabla J(\theta_t)}{2}+  G_{t+1}(\theta_t, \omega(\theta_t)) \right\rangle \nn\\
    &\quad-\frac{\alpha}{2}\|\nabla J(\theta_t)\|^2+\frac{L_J}{2}  \alpha^2C_g^2,
\end{align}
where $(a)$ is from the fact that $\|\theta_{t+1}-\theta_t\| \leq\alpha   C_g$. 
Thus by re-arranging the terms, taking expectation and summing up w.r.t. $t$ from $0$ to $T-1$, it follows that 
\begin{align}\label{eq:markovmain1}
    &\frac{\alpha}{2}\sum^{T-1}_{t=0}\mathbb{E}[\|\nabla J(\theta_t)\|^2] \nn\\
    &\leq -\mathbb{E}[J(\theta_{T})]+J(\theta_0)+\alpha   L_g\sqrt{\sum^{T-1}_{t=0}\mathbb{E}[\|\nabla J(\theta_t)\|^2]}\sqrt{\sum^{T-1}_{t=0}\mathbb{E}[\|z_t\|^2]}  +\sum^{T-1}_{t=0} \alpha \mE[\zeta_G(\theta_t,O_t)]\nn\\
    &\quad+L_J\alpha^2TC_g^2 ,
\end{align}
where $\zeta_G(\theta_t,O_t)=\left\langle \nabla J(\theta_t), \frac{\nabla J(\theta_t)}{2}+  G_{t+1}(\theta_t, \omega(\theta_t)) \right\rangle$. We then bound $\zeta_G$ in the following lemma. 
\begin{Lemma}
For any $t\geq\tau_{\beta}$,  
\begin{align}\label{eq:stobias}
     \mE[\zeta_G(\theta_{t},O_t) ] \leq 2C_g^2\beta+2\alpha \tb L_{\zeta}C_g.
\end{align}
\end{Lemma}
\begin{proof}
We only need to consider the case $t=\tau_{\beta}$, the proof for general case of $t\geq\tau_{\beta}$ is similar, and thus is omitted here. We first have that
\begin{align}
    \zeta_G(\theta_{\tau_{\beta}},O_{\tau_{\beta}})&=\left\langle \nabla J(\theta_{\tau_{\beta}}), \frac{\nabla J(\theta_{\tau_{\beta}})}{2}+  G_{\tau_{\beta}+1}(\theta_{\tau_{\beta}}, \omega(\theta_{\tau_{\beta}})) \right\rangle\nn\\
    &= \left\langle \nabla J(\theta_0), \frac{\nabla J(\theta_0)}{2}+  G_{\tau_{\beta}+1}(\theta_0, \omega(\theta_{0})) \right\rangle\nn\\ &\quad+\left\langle \nabla J(\theta_{\tau_{\beta}}), \frac{\nabla J(\theta_{\tau_{\beta}})}{2}+  G_{\tau_{\beta}+1}(\theta_{\tau_{\beta}}, \omega(\theta_{\tau_{\beta}})) \right\rangle\nn\\
    &\quad-\left\langle \nabla J(\theta_0), \frac{\nabla J(\theta_0)}{2}+  G_{\tau_{\beta}+1}(\theta_0, \omega(\theta_{0})) \right\rangle\nn\\
    &\leq \left\langle \nabla J(\theta_0), \frac{\nabla J(\theta_0)}{2}+  G_{\tau_{\beta}+1}(\theta_0, \omega(\theta_{0})) \right\rangle+2L_{\zeta}\|\theta_{\tau_{\beta}}-\theta_0 \|\nn\\
    &\leq \left\langle \nabla J(\theta_0), \frac{\nabla J(\theta_0)}{2}+  G_{\tau_{\beta}+1}(\theta_0, \omega(\theta_{0})) \right\rangle+2\alpha\tb  L_{\zeta}C_g   ,
\end{align}
where $L_{\zeta}=2C_g(L_g'+\frac{3L_J}{2})$ is the Lipschitz constant of $\zeta_G(\theta,O_t)$. 

Then it follows that 
\begin{align}
    &\mE[\zeta_G(\theta_{\tau_{\beta}},O_{\tau_{\beta}})]\nn\\
 &=\mE\left[\left\langle \nabla J(\theta_0), \frac{\nabla J(\theta_0)}{2}+  G_{\tau_{\beta}+1}(\theta_0, \omega(\theta_{0})) \right\rangle\right]+2\alpha    L_{\zeta}C_g\tb\nn\\
    &\leq 2 C_g^2\beta+2\alpha \tb L_{\zeta}C_g ,
\end{align}
where the last step follows from the uniform ergodicity of the MDP (Assumption \ref{ass:mixing}).
\end{proof}
Plugging the bound  in \eqref{eq:markovmain1}, it follows that 
\begin{align} 
    &\frac{\alpha}{2}\sum^{T-1}_{t=0}\mathbb{E}[\|\nabla J(\theta_t)\|^2] \nn\\
    &\leq J(\theta_0)-J^*+\alpha    L_g\sqrt{\sum^{T-1}_{t=0}\mathbb{E}[\|\nabla J(\theta_t)\|^2]}\sqrt{\sum^{T-1}_{t=0}\mathbb{E}[\|z_t\|^2]}\nn\\
    &\quad+\alpha^2    C_g^2L_JT+\alpha \left( T(2C_g^2\beta    +2\alpha \tb    L_{\zeta}C_g)+4\tb     C_g^2 \right),
\end{align}
and thus
\begin{align}
    &\sum^{T-1}_{t=0}\mathbb{E}[\|\nabla J(\theta_t)\|^2] \nn\\
    &\leq \frac{2(J(\theta_0)-J^*)}{\alpha}+2    L_g\sqrt{\sum^{T-1}_{t=0}\mathbb{E}[\|\nabla J(\theta_t)\|^2]}\sqrt{\sum^{T-1}_{t=0}\mathbb{E}[\|z_t\|^2]} +2\alpha     C_g^2L_JT \nn\\
    &\quad+2 \left( T(2    C_g^2\beta+2\alpha \tb    L_{\zeta}C_g)+4\tb C_g^2     \right).
\end{align}
This further implies that 
\begin{align}
    &\frac{\sum^{T-1}_{t=0}\mathbb{E}[\|\nabla J(\theta_t)\|^2]}{T} \nn\\
    &\leq \frac{2(J(\theta_0)-J^*)}{\alpha T}+2   L_g\sqrt{\frac{\sum^{T-1}_{t=0}\mathbb{E}[\|\nabla J(\theta_t)\|^2]}{T}}\sqrt{\frac{\sum^{T-1}_{t=0}\mathbb{E}[\|z_t\|^2]}{T}}+2\alpha     C_g^2L_J \nn\\
    &\quad+2     \left( (2C_g^2\beta+2\alpha \tb L_{\zeta}C_g)+4C_g^2\frac{\tb}{T} \right).
\end{align}
We plug in the tracking error \eqref{eq:markoviantrackingerror}, and it follows that 
\begin{align}
    &\frac{\sum^{T-1}_{t=0}\mathbb{E}[\|\nabla J(\theta_t)\|^2]}{T} \nn\\
    &\leq \frac{2(J(\theta_0)-J^*)}{\alpha T}+2\alpha     C_g^2L_J +2     \left( (2C_g^2\beta+2\alpha \tb L_{\zeta}C_g)+4C_g^2\frac{\tb}{T} \right)\nn\\
    &\quad+2   L_g\sqrt{\frac{\sum^{T-1}_{t=0}\mathbb{E}[\|\nabla J(\theta_t)\|^2]}{T}}\nn\\
    &\quad \cdot \sqrt{ {\left(2\|z_0\|+2\beta\tau_{\beta}   (b_{\max}+   L_{\omega}C_g) \right)^2}\left(\frac{1}{Tq}+\frac{\tau_{\beta}}{T} \right)+\frac{\alpha L_{\omega}}{4q}\frac{\sum^{T-1}_{t=0}\mE[\|\nabla J(\theta_t)\|^2]}{T}+\frac{p}{q}} \nn\\
    &\leq  \frac{2(J(\theta_0)-J^*)}{\alpha T}+2\alpha     C_g^2L_J +2     \left( (2C_g^2\beta+2\alpha \tb L_{\zeta}C_g)+4C_g^2\frac{\tb}{T} \right)\nn\\
    &\quad+2   L_g\sqrt{\frac{\alpha L_{\om}}{4q}}\frac{\sum^{T-1}_{t=0}\mathbb{E}[\|\nabla J(\theta_t)\|^2]}{T} \nn\\
    &\quad+2L_g   \sqrt{\frac{\sum^{T-1}_{t=0}\mathbb{E}[\|\nabla J(\theta_t)\|^2]}{T}}\sqrt{ {\left(2\|z_0\|+2\beta\tau_{\beta}(b_{\max}+   L_{\omega}C_g) \right)^2}\left(\frac{1}{Tq}+\frac{\tau_{\beta}}{T} \right)+ \frac{p}{q}}.
\end{align}
Note that $2L_g   \sqrt{\frac{\alpha L_{\om}}{4q}}=\mathcal{O}\left(\sqrt{\frac{\alpha}{\beta}}\right)$, hence we can choose $\alpha$ and $\beta$ such that $2   L_g\sqrt{\frac{\alpha L_{\om}}{4q}} \leq\frac{1}{2}$. Hence it follows that 
\begin{align}\label{eq:UV}
     \frac{\sum^{T-1}_{t=0}\mathbb{E}[\|\nabla J(\theta_t)\|^2]}{T} 
    &\leq  \frac{4(J(\theta_0)-J^*)}{\alpha T}+4\alpha     C_g^2L_J +4    \left( (2C_g^2\beta+2\alpha \tb L_{\zeta}C_g)+4C_g^2\frac{\tb}{T} \right) \nn\\
    &\quad+4   L_g\sqrt{\frac{\sum^{T-1}_{t=0}\mathbb{E}[\|\nabla J(\theta_t)\|^2]}{T}}\nn\\
    &\cdot\sqrt{ {\left(2\|z_0\|+2\beta\tau_{\beta}(b_{\max}+   L_{\omega}C_g) \right)^2}\left(\frac{1}{Tq}+\frac{\tau_{\beta}}{T} \right)+ \frac{p}{q}}\nn\\
    &\triangleq U\sqrt{\frac{\sum^{T-1}_{t=0}\mathbb{E}[\|\nabla J(\theta_t)\|^2]}{T}} +V,
\end{align}
where $U=4   L_g\sqrt{ {\left(2\|z_0\|+2\beta\tau_{\beta}(b_{\max}+   L_{\omega}C_g) \right)^2}\left(\frac{1}{Tq}+\frac{\tau_{\beta}}{T} \right)+ \frac{p}{q}}=\mo\left(\sqrt{\beta\tb+\frac{1}{T\beta}} \right)$ and $V=\frac{4(J(\theta_0)-J^*)}{\alpha T}+4\alpha     C_g^2L_J +4    \left( (2C_g^2\beta+2\alpha \tb L_{\zeta}C_g)+4C_g^2\frac{\tb}{T} \right)=\mo\left(\frac{1}{T\alpha}+\alpha\tb +\beta\right)$.
Thus it can be shown that
\begin{align}\label{eq:markovresult}
    \frac{\sum^{T-1}_{t=0}\mathbb{E}[\|\nabla J(\theta_t)\|^2]}{T} &\leq \left(\frac{U+\sqrt{U^2+4V}}{2} \right)^2\nn\\
    &\leq U^2+2V \nn\\
    &=16    L_g^2\left( {\left(2\|z_0\|+2\beta\tau_{\beta}(b_{\max}+   L_{\omega}C_g) \right)^2}\left(\frac{1}{Tq}+\frac{\tau_{\beta}}{T} \right)+ \frac{p}{q}\right)\nn\\
    &\quad+\frac{8(J(\theta_0)-J^*)}{\alpha T}+8\alpha     C_g^2L_J +8    \left( (2C_g^2\beta+2\alpha \tb L_{\zeta}C_g)+4C_g^2\frac{\tb}{T} \right)\nn\\
    &=\mo\left(\beta\tb+\frac{1}{T\beta}+ \alpha\tb+\frac{1}{T\alpha}\right).
\end{align}
This  completes the proof. 

\subsection{Choice of Step-sizes}\label{sec:step2}
In the proof under the Markovian setting, we first assume $\beta\tb C_{\phi}^2\leq \frac{1}{4}$. The last assumption on the step-sizes is $\frac{\alpha}{q}\leq \frac{1}{4L_g^2L_{\omega}}$, where $q=2\beta\lambda_v-2\beta(R_1+R_3+P_1+P_2+P_3)-2\beta^2C_{\phi}^4=\mo(\beta)$. Note that this assumption can be satisfied by controlling $\frac{\alpha}{\beta}$ similar to Section \ref{sec:step1}, which  we omit here. 
Hence we  set $\beta<\min\left\{1, \frac{1}{4\tb C_{\phi}^2} \right\}$, and $\frac{\alpha}{q}\leq \left\{ 1, \frac{1}{4L_g^2L_{\omega}}\right\}$.

\section{Experiments}\label{sec:experiments}
In this section, we provide some numerical experiments on two RL examples: the Garnet problem \citep{archibald1995generation} and the “spiral” counter example in \citep{tsitsiklis1997analysis}.

\subsection{Garnet Problem}
The first experiment is on the Garnet problem \citep{archibald1995generation}, which can be characterized by $\mathcal{G}(|\mcs|,|\mca|,b,N)$. Here $b$ is a branching parameter specifying how many next states are possible for each state-action pair, and  these $b$ states are chosen uniformly at random. The transition probabilities are generated by sampling uniformly and randomly between 0 and 1. The parameter $N$ is the dimension of $\theta$ to be updated. In our experiments, we generate a reward matrix uniformly and randomly between 0 and 1. For every state $s$ we randomly generate one feature function $k(s) \in [0,1]$ using as the input. In both experiments, we use a five-layer neural network with (1,2,2,3,1) neurons in each layer as the function approximator. And for the activation function, we use the Sigmoid function, i.e., $f(x)=\frac{1}{1+e^{-x}}$. We set all the weights and bias of the neurons as the parameter $\theta\in\mathbb{R}^{23}$.

We consider two sets of parameters: $\mathcal{G}(5,2,5,23)$ and $\mathcal{G}(3,2,3,23)$. We set the step-size $\alpha=0.01$ and $\beta=0.05$, and also the discount factor $\gamma=0.95$. In Figures \ref{Fig.ga1} and \ref{Fig.ga2}, we plot the squared gradient norm v.s. the number of samples using 40 Garnet MDP trajectories, i.e., at each time $t$, we plot $ \|\nabla J(\theta_t)\|^2$. The upper and lower envelopes of the curves correspond to the 95 and 5 percentiles of the 40 curves, respectively. We also plot the estimated variance of the stochastic update along the iterations in Figures \ref{Fig.g12} and \ref{Fig.g22}. Specifically, we first run the algorithm to get a  sequence of $\theta_t$ and $\omega_t$. Then we generate 
500 different trajectories $O^i=(O^i_1,O^i_2,...,O^i_t,...)$ where $i=1,...,
500$, and use them to estimate the variance $\|G^i_{t+1}(\theta_t,\omega_t)-\nabla J(\theta_t)\|^2$ and plot $\frac{\sum^{500}_{i=1}\|G^i_{t+1}(\theta_t,\omega_t)-\nabla J(\theta_t)\|^2}{500}$ at each time $t$. 

It can be seen from the figures that both gradient norm $\|\nabla J(\theta_t)\|$ and the estimated variance converge to zero.

\begin{figure}[htbp]
\centering
\subfigure[$\|\nabla J(\theta_t)\|^2$.]{
\label{Fig.g11}
\includegraphics[width=0.45\linewidth]{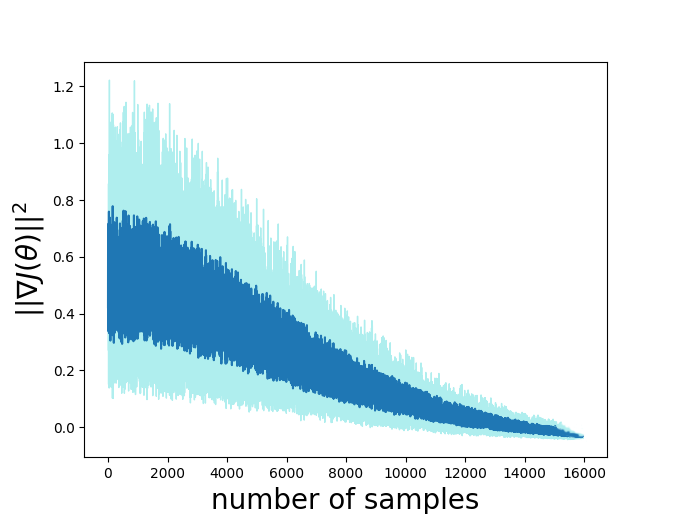}}
\subfigure[Estimated variance.]{
\label{Fig.g12}
\includegraphics[width=0.45\linewidth]{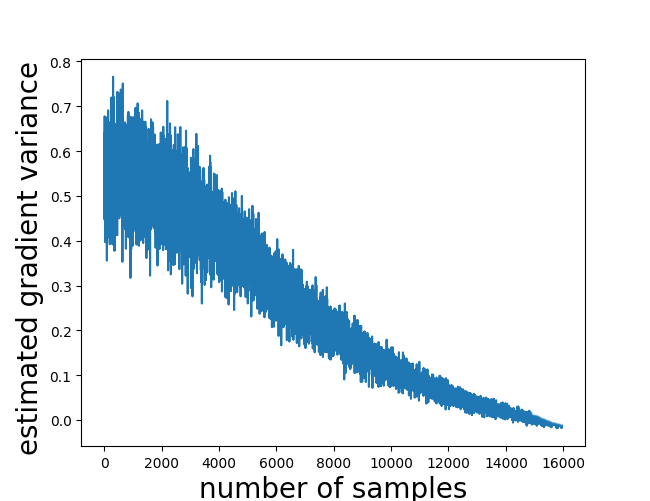}}
\captionsetup{font={normalsize}}
\caption{Garnet problem 1: $\mathcal{G}(5,2,5,23)$.}
\label{Fig.ga1}
\end{figure}
\begin{figure} 
\centering
\subfigure[$\|\nabla J(\theta_t)\|^2$.]{
\label{Fig.g21}
\includegraphics[width=0.45\linewidth]{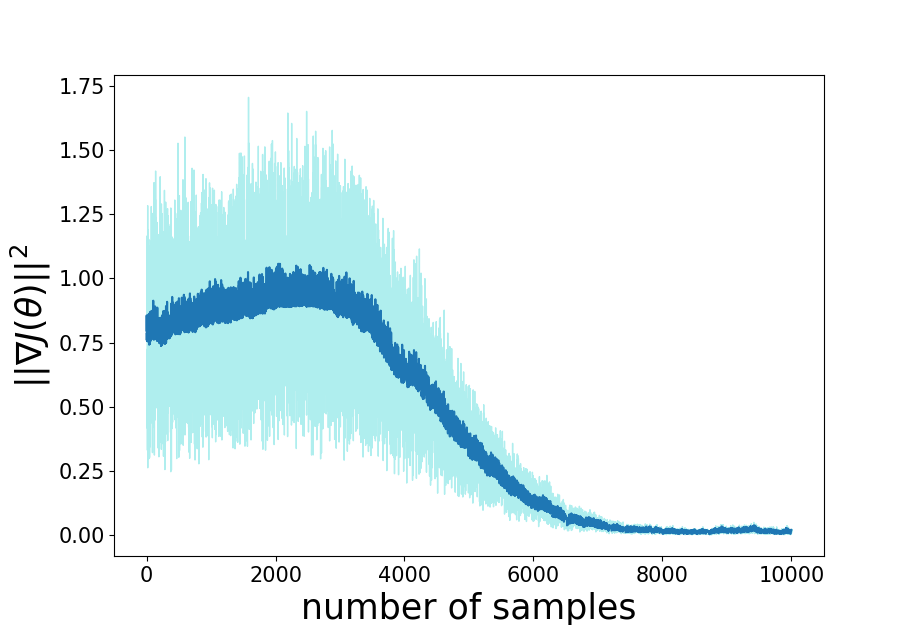}}
\subfigure[Estimated variance.]{
\label{Fig.g22}
\includegraphics[width=0.45\linewidth]{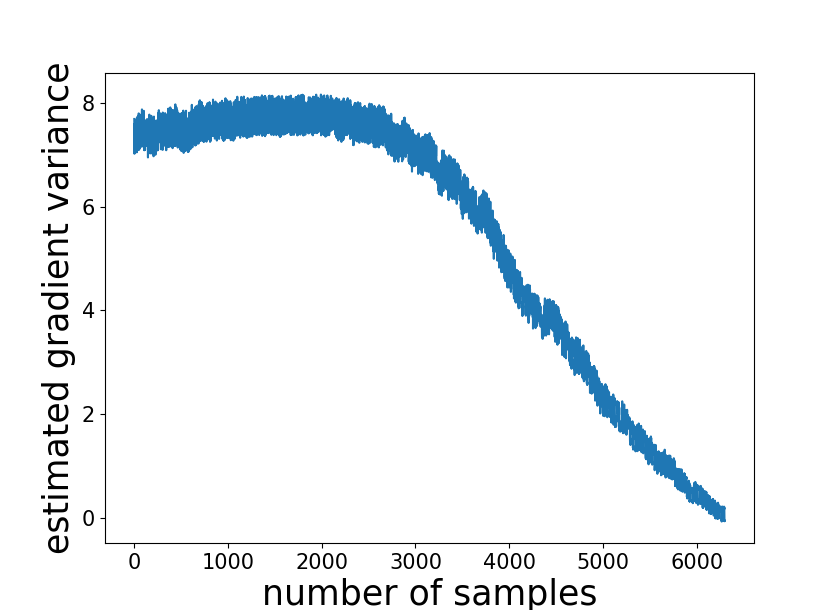}}
\captionsetup{font={normalsize}}
\caption{Garnet problem 2: $\mathcal{G}(3,2,3,23)$.}
\label{Fig.ga2}
\end{figure}

\subsection{Spiral Counter Example}
In our second experiment, we consider the spiral counter example proposed in \citep{tsitsiklis1997analysis}, which is often used to show the TD algorithm may diverge with nonlinear function approximation. The problem setting is given in Figure \ref{Fig.toy}. There are three states and each state can transit to the next one with probability $\frac{1}{2}$ or stay at the current state with probability $\frac{1}{2}$. The reward is always zero with the discount factor $\gamma=0.9$. Similar to \citep{bhatnagar2009convergent}, we consider the value function approximation: \begin{align}
    V_{\theta}(s)=(a(s)\cos(k\theta)+b(s)\sin(k\theta))e^{\epsilon \theta},
\end{align}
where in Figure \ref{Fig.toy1}, $a=[0.94,-0.43,0.18]$ and $b=[0.21,-0.52,0.76]$; and in Figure \ref{Fig.toy2}, $a=[0.21,-0.33,0.29]$ and $b=[0.68,0.41,0.82]$.
We  let $k=0.866$ and $\epsilon=0.1$. The step-size are chosen as $\alpha=0.01$ and $\beta=0.05$. In Figures \ref{Fig.t11} and \ref{Fig.t21}, we plot the squared gradient norm v.s. the number of samples using 40 MDP trajectories. The upper and lower envelopes of the curves correspond to the 95 and 5 percentiles of the 40 curves. Similarly, we also plot the estimated variance $\|G_{t+1}(\theta_t,\omega_t)-\nabla J(\theta_t)\|^2$ of the stochastic update  along the iterations using 50 samples at each time step. More specifically, we  first run the algorithm to get a  sequence of $\theta_t$ and $\omega_t$. Then we generate 50 different trajectories $O^i=(O^i_1,O^i_2,...,O^i_t,...)$ where $i=1,...,50$, and use  them to estimate the variance $\|G^i_{t+1}(\theta_t,\omega_t)-\nabla J(\theta_t)\|^2$ and plot $\frac{\sum^{50}_{i=1}\|G^i_{t+1}(\theta_t,\omega_t)-\nabla J(\theta_t)\|^2}{50}$ at each time $t$.

It can be seen that in both experiments, the gradient norm $\|\nabla J(\theta_t)\|$ converges to 0, i.e., the algorithm converges to a stationary point. The estimated variance also decreases to zero. 

\begin{figure}[htb]
\centering 
\subfigure {
\label{Fig.t}
\includegraphics[width=0.7\linewidth]{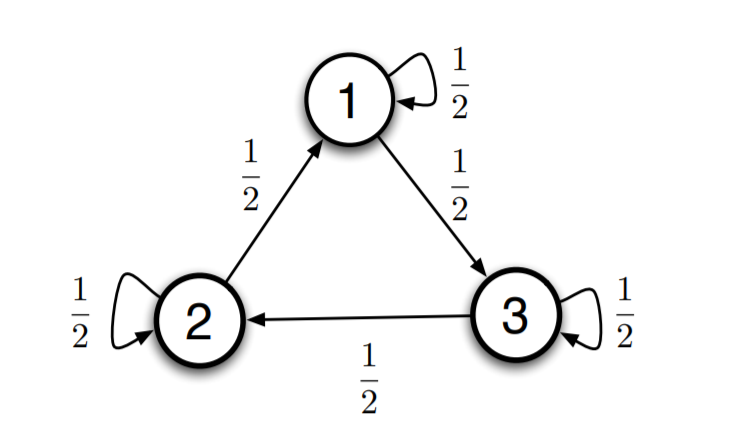}}
\captionsetup{font={normalsize}}
\caption{Spiral counter example.}
\label{Fig.toy}
\end{figure}

\begin{figure}[htbp]
\centering
\subfigure[$\|\nabla J(\theta_t)\|^2$.]{
\label{Fig.t11}
\includegraphics[width=0.48\linewidth]{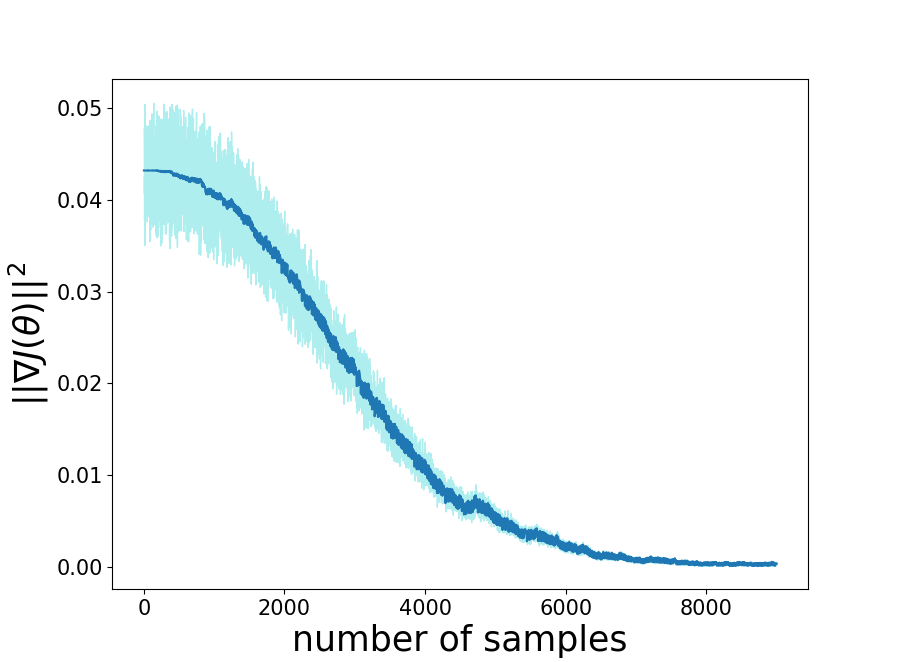}}
\subfigure[Estimated variance.]{
\label{Fig.t12}
\includegraphics[width=0.48\linewidth]{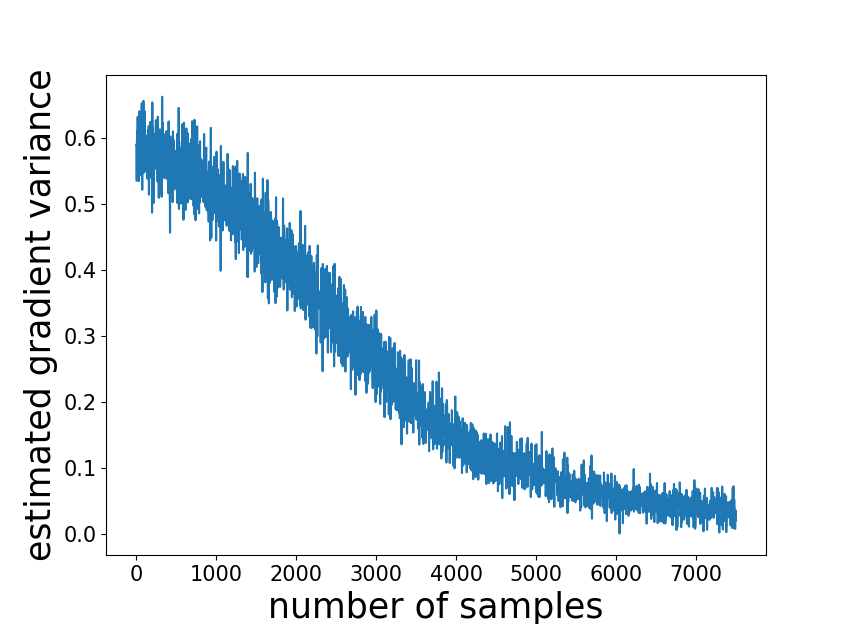}}
\captionsetup{font={normalsize}}
\caption{Spiral counter example 1:  \protect\\$a=[0.94,-0.43,0.18]$,$b=[0.21,-0.52,0.76]$.}
\label{Fig.toy1}
\end{figure}

\begin{figure}
\centering
\subfigure[$\|\nabla J(\theta_t)\|^2$.]{
\label{Fig.t21}
\includegraphics[width=0.48\linewidth]{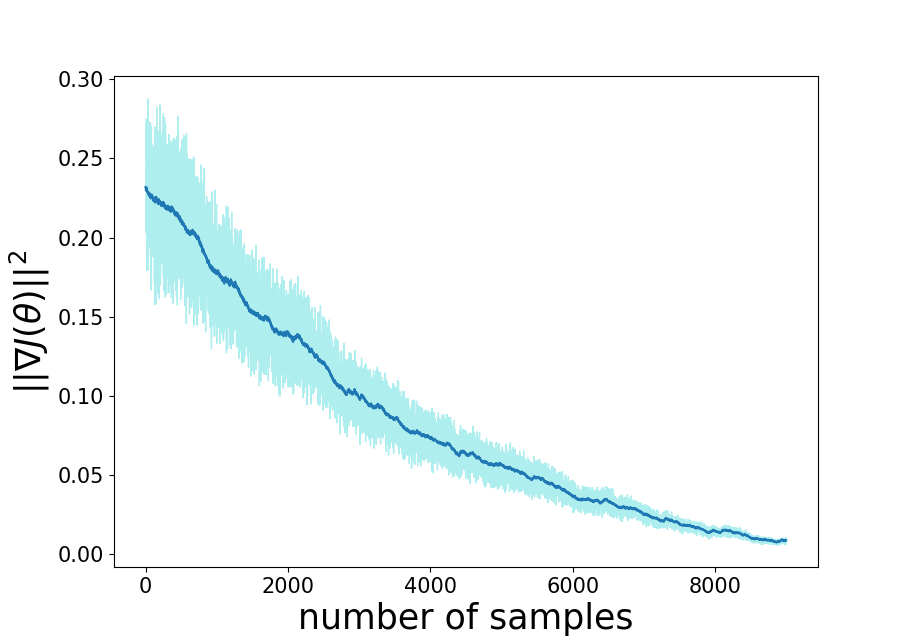}}
\subfigure[Estimated variance.]{
\label{Fig.t22}
\includegraphics[width=0.48\linewidth]{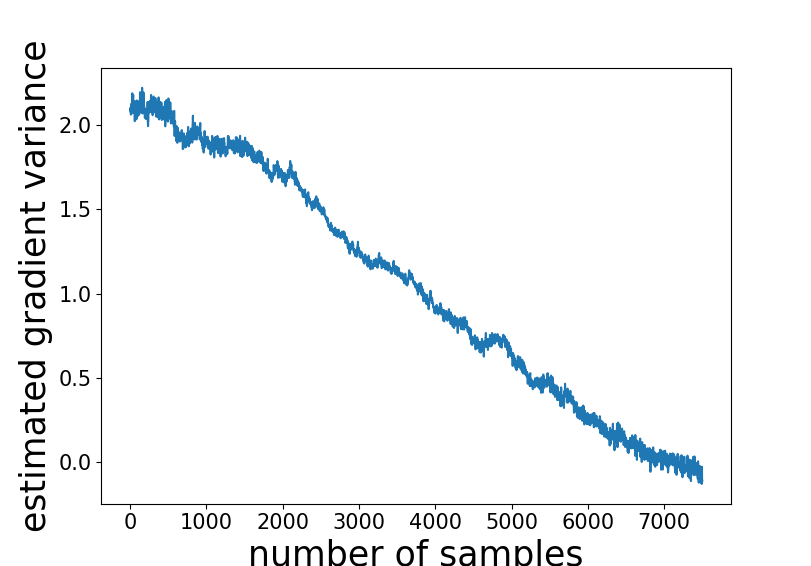}}
\captionsetup{font={normalsize}}
\caption{Spiral counter example 2:  \protect\\$a=[0.21,-0.33,0.29]$, $b=[0.68,0.41,0.82]$.}
\label{Fig.toy2}
\end{figure}

\end{document}